\documentclass[twoside,11pt]{article}

\usepackage[utf8]{inputenc}
\usepackage{amsmath}
\usepackage{amsfonts}
\usepackage{amssymb}
\usepackage{graphicx}
\usepackage{algorithm} 
\usepackage{natbib}
\usepackage{bm}
\usepackage{color}
\usepackage{amsthm}
\usepackage{url}
\usepackage{csquotes}
\usepackage[english]{babel}
\usepackage{graphicx}
\usepackage{enumerate}
\usepackage{enumitem}
\usepackage{mathtools}
\usepackage{placeins}
\usepackage{multirow}
\usepackage{titletoc}

\usepackage[letterpaper, margin=1.1in]{geometry}

\newtheorem{theorem}{Theorem}
\newtheorem{lemma}{Lemma}

\newtheorem{proposition}{Proposition}
\newtheorem{fact}{Fact}

\makeatletter
\renewcommand{\maketitle}{\bgroup\setlength{\parindent}{0pt}
\begin{center}
\LARGE 
  \textbf{\@title} \vspace{10pt}
  \end{center}
\begin{flushleft}
  \@author
\end{flushleft}
\begin{flushright}
\@date
\end{flushright}
\egroup
}

\makeatother

\begin{document}

% "Title of the paper"

\title{Probabilistic Integration: A Role in Statistical Computation?}

\author{\textbf{Fran\c{c}ois-Xavier Briol$^{1,2}$, Chris. J. Oates$^{3,4}$, Mark Girolami$^{2,4}$, \\
Michael A. Osborne$^{5}$ and Dino Sejdinovic$^{4,6}$} \vspace{10pt} \\ 
$^1$Department of Statistics, University of Warwick\\
$^2$Department of Mathematics, Imperial College London \\
$^3$School of Mathematics, Statistics and Physics, Newcastle University \\
$^4$The Alan Turing Institute for Data Science \\
$^5$Department of Engineering Science, University of Oxford  \\
$^6$Department of Statistics. University of Oxford }

%
%% indicate corresponding author with \corref{}
%\begin{aug}
%\author{\fnms{Fran\c{c}ois-Xavier} \snm{Briol$^{1,2}$}\corref{}\ead[label=e1]{f-x.briol@warwick.ac.uk}},
%\author{\fnms{Chris J.} \snm{Oates$^{3,4}$}\ead[label=e2]{chris.oates@ncl.ac.uk}},
%\author{\fnms{Mark} \snm{Girolami$^{2,4}$}\ead[label=e3]{m.girolami@imperial.ac.uk}},
%\author{\fnms{Michael A.} \snm{Osborne$^5$}\ead[label=e4]{mosb@robots.ox.ac.uk}}
%\and
%\author{\fnms{Dino} \snm{Sejdinovic$^{4,5}$}\ead[label=e5]{dino.sejdinovic@stats.ox.ac.uk}}
%\affiliation{$^1$University of Warwick, $^2$ Imperial College London, $^3$Newcastle University, $^4$Alan Turing Institute,  \and $^5$University of Oxford}
%
%\address{Department of Statistics,\\ University of Warwick,\\ Coventry,\\ CV4 7AL,\\ UK \\ \printead{e1}.}
%\address{School of Mathematics, Statistics and Physics,\\ Newcastle University,\\ Newcastle-upon-Tyne,\\ NE1 7RU,\\ UK \\ \printead{e2}.}
%\address{Department of Mathematics,\\ Imperial College London,\\ London,\\ SW7 2AZ,\\ UK \\ \printead{e3}.}
%\address{Department of Engineering Science,\\ University of Oxford,\\ Oxford,\\ OX1 3PJ,\\ UK \\ \printead{e4}.}
%\address{Department of Statistics,\\ University of Oxford,\\ Oxford,\\ OX1 3LB,\\ UK \\ \printead{e5}.}
%\runauthor{Briol, Oates, Girolami, Osborne \and Sejdinovic}
%\end{aug}

\maketitle

\begin{abstract}
\textcolor{black}{
A research frontier has emerged in scientific computation, wherein discretisation error is regarded as a source of epistemic uncertainty that can be modelled. 
This raises several statistical challenges, including the design of statistical methods that enable the coherent propagation of probabilities through a (possibly deterministic) computational work-flow, in order to assess the impact of discretisation error on the computer output. 
This paper examines the case for probabilistic numerical methods in routine statistical computation. Our focus is on numerical integration, where a \emph{probabilistic integrator} is equipped with a full distribution over its output that reflects the fact that the integrand has been discretised. 
Our main technical contribution is to establish, for the first time, rates of posterior contraction for one such method. 
Several substantial applications are provided for illustration and critical evaluation, including examples from statistical modelling, computer graphics and a computer model for an oil reservoir. 
}
\end{abstract}

%%%%%%%%%%%%%%%%%%%%%%%%%%%%%%%%%%%%%%%%%%%%%%%%%%%%%%%%%%%%%%%%%%%%%%%%%%%%%%%%%%%%%%%%%%%%%%%%%%%%%%%%%%%%%%%%%%%%%%%%%%%%%%%%%%%%%%%%%%%%%%%%%%%%%%%%%%%%%%%%%%%%%%%%%%%%%%%%%%%%%%%%%%%%%%%%%%%%%%%%%%%%%%%%%%

\section{Introduction}

This paper presents a statistical perspective on the theoretical and methodological issues pertinent to probabilistic numerical methods.
Our aim is to stimulate what we feel is an important discussion about these methods for use in contemporary and emerging scientific and statistical applications.

\subsection{Background}
Numerical methods, for tasks such as approximate solution of a linear system, integration, global optimisation and discretisation schemes to approximate the solution of differential equations, are core building blocks in modern scientific and statistical computation. 
These are typically considered as computational black-boxes that return a point estimate for a deterministic quantity of interest whose numerical error is then neglected.
Numerical methods are thus one part of statistical analysis for which uncertainty is not routinely accounted (although  analysis of errors and bounds on these are often available and highly developed). 
In many situations numerical error will be negligible and no further action is required.
However, if numerical errors are propagated through a computational pipeline and allowed to accumulate, then failure to properly account for such errors could potentially have drastic consequences on subsequent statistical inferences \citep{Mosbach2009,Oates2017hydrocyclones}.

The study of numerical algorithms from a statistical point of view, where uncertainty is formally due to discretisation, is known as {\it probabilistic numerics}. 
The philosophical foundations of probabilistic numerics were, to the best of our knowledge, first clearly exposed in the work of \cite{Larkin1972,Kadane1985b,Diaconis1988} and \cite{OHagan1992}. 
Theoretical support comes from the field of information-based complexity \citep{Traub1988}, where continuous mathematical operations are approximated by discrete and finite operations to achieve a prescribed accuracy level.
Proponents claim that this approach provides three important benefits:
Firstly, it provides a principled approach to quantify and propagate numerical uncertainty through computation, allowing for the possibility of errors with complex statistical structure.
Secondly, it enables the user to uncover key contributors to numerical error, using established statistical techniques such as analysis of variance, in order to better target computational resources. 
Thirdly, this dual perspective on numerical analysis as an inference task enables new insights, as well as the potential to critique and refine existing numerical methods.
On this final point, recent interest has led to several new and effective numerical algorithms in many areas, including differential equations, linear algebra and optimisation.
For an extensive bibliography, the reader is referred to the recent expositions of \cite{Hennig2015} and \cite{Cockayne2017}.

\subsection{Contributions}

\textcolor{black}{Our aim is to stimulate a discussion on the suitability of probabilistic numerical methods in statistical computation.}
A decision was made to focus on numerical integration due to its central role in computational statistics, including frequentist approaches such as bootstrap estimators \citep{Efron1994} and Bayesian approaches, such as computing marginal distributions \citep{Robert2013}.
In particular we focus on numerical integrals where the cost of evaluating the integrand forms a computational bottleneck.
To this end, let $\pi$ be a distribution on a state space $\mathcal{X}$.
The task is to compute (or, rather, to \emph{estimate}) integrals of the form
\begin{equation*}
\Pi[f] := \int f \; \mathrm{d}\pi,
\end{equation*}
where the integrand $f: \mathcal{X} \rightarrow \mathbb{R}$ is a function of interest.
Our motivation comes from settings where $f$ does not possess a convenient closed form so that, until the function is actually evaluated at an input $\bm{x}$, there is epistemic uncertainty over the actual value attained by $f$ at $\bm{x}$.
The use of a probabilistic model for this epistemic uncertainty has been advocated as far back as \cite{Larkin1972}. 
The probabilistic integration method that we focus on is known as \emph{Bayesian \textcolor{black}{cubature}} (BC). 
The method operates by evaluating the integrand at a set of states $\{\bm{x}_i\}_{i=1}^n \subset \mathcal{X}$, so-called {\it discretisation}, and returns a distribution over $\mathbb{R}$ that expresses belief about the true value of $\Pi[f]$.
\textcolor{black}{The computational cost associated with BQ is in general $O(n^3)$.}
As the name suggests, this distribution will be based on a prior that captures certain properties of $f$, and that is updated, via Bayes' rule, on the basis of evaluations of the integrand. 
The \emph{maximum a posteriori} (MAP) value acts as a point estimate of the integral, while the rest of the distribution captures uncertainty due to the fact that we can only evaluate the integrand at a finite number of inputs. 
\textcolor{black}{However, a theoretical investigation of this posterior\footnote{\textcolor{black}{in contrast to the MAP estimator, which has been well-studied.}} is, to the best of our knowledge, non-existent.}

\textcolor{black}{Our first contribution is therefore to investigate the claim that the BC posterior provides a coherent and honest assessment of the uncertainty due to discretisation of the integrand. 
This claim is shown to be substantiated by rigorous mathematical analysis of BC, building on analogous results from reproducing kernel Hilbert spaces, \emph{if} the prior is well-specified.
In particular, rates of posterior contraction to a point mass centred on the true value $\Pi[f]$ are established.
However, to check that a prior is well-specified for a given integration problem can be non-trivial.}

\textcolor{black}{Our second contribution is to explore the potential for the use of probabilistic integrators in the contemporary statistical context.}
In doing so, we have developed strategies for (i) model evidence evaluation via thermodynamic integration, where a large number of candidate models are to be compared, (ii) inverse problems arising in partial differential equation models for oil reservoirs, (iii) logistic regression models involving high-dimensional latent random effects, and (iv) spherical integration, as used in the rendering of virtual objects in prescribed visual environments.
\textcolor{black}{In each case results are presented ``as they are'' and the relative advantages and disadvantages of the probabilistic approach to integration are presented for critical assessment.}

\subsection{Outline}
The paper is structured as follows. Sec. \ref{sec:prob_integrators} provides background on BC and outlines an analytic framework in which the method can be studied.
Sec. \ref{sec:theoretical_results} describes our novel theoretical results.
Sec. \ref{sec:practical_issues} is devoted to a discussion of practical issues, including the important issue of prior elicitation.
Sec. \ref{sec:numerical_results} presents several novel applications of probabilistic integration for critical assessment\footnote{computer code to reproduce experiments reported in this paper can be downloaded from \url{http://www.warwick.ac.uk/fxbriol/probabilistic_integration}.}. 
Sec. \ref{sec:conclusion} concludes with an appraisal of the suitability of probabilistic numerical methods in the applied statistical context.

%%%%%%%%%%%%%%%%%%%%%%%%%%%%%%%%%%%%%%%%%%%%%%%%%%%%%%%%%%%%%%%%%%%%%%%%%%%%%%%%%%%%%%%%%%%%%%%%%%%%%%%%%%%%%%%%%%%%%%%%%%%%%%%%%%%%%%%%%%%%%%%%%%%%%%%%%%%%%%%%%%%%%%%%%%%%%%%%%%%%%%%%%%%%%%%%%%%%%%%%%%%%%%%%%%%%%%%%%%%%%%%%%%%%%%%%%%%%%%%%%%%%%%%%%%%%%%%%%%%%%%%%%%%%%%%%%%%%%%%%%%%%%%%%%%%%%%%%%%%%%%

\section{Background}\label{sec:prob_integrators}

First we provide the reader with the relevant background.
Sec. \ref{BQ introduce sec} provides a formal description of BC.
Secs. \ref{duality sec} and \ref{duality sec 2} explain how the analysis of BC is dual to minimax analysis in nonparametric regression, and Sec. \ref{how to sample} relates these ideas to established sampling methods.

%%%%%%%%%%%%%%%%%%%%%%%%%%%%%%%%%%%%%%%%%%%%%%%%%%%%%%%%%%%%%%%%%%%%%%%%%%%%%%%%%%%%%%%%%%%%%%%%%%%%%%%%%%%%%%%%%%%%%%%%%%%%%%%%%%%%%%%%%%%%%%

\paragraph{Set-Up:}

\textcolor{black}{Let $(\mathcal{X},\mathcal{B})$ be a measurable space, where $\mathcal{X}$ will either be a subspace of $\mathbb{R}^d$ or a more general manifold (e.g. the sphere $\mathbb{S}^d$),} in each case equipped with the Borel $\sigma$-algebra $\mathcal{B} = \mathcal{B}(\mathcal{X})$.
Let $\pi$ be a distribution on $(\mathcal{X},\mathcal{B})$.
Our integrand is assumed to be an integrable function $f:\mathcal{X} \rightarrow \mathbb{R}$ whose integral, $\Pi[f] = \int f \mathrm{d}\pi$, is the object of interest.

\paragraph{Notation:}

For functional arguments write $\langle f , g \rangle_2 = \int f g \; \mathrm{d}\pi$, $\|f\|_2 = \langle f , f \rangle_2^{1/2}$ and for vector arguments denote $\|\bm{u}\|_2 = (u_1^2 + \dots + u_d^2)^{1/2}$.
\textcolor{black}{For vector-valued functions $\bm{v} : \mathcal{X} \rightarrow \mathbb{R}^m$ we write $\Pi[\bm{v}]$ for the $m \times 1$ vector whose $i$th element is $\Pi[v_i]$.}
The notation $[u]_+ = \max\{0,u\}$ will be used. 
The relation $a_l \asymp b_l$ is taken to mean that there exist $0 < C_1,C_2 < \infty$ such that $C_1 a_l \leq b_l \leq C_2 a_l$.

A \textit{\textcolor{black}{cubature} rule} describes any functional $\hat{\Pi}$ of the form
\begin{equation}\label{eq:quadrature}
\hat{\Pi}[f] = \sum_{i=1}^n w_i f(\bm{x}_i),
\end{equation}
for some states $\{\bm{x}_i\}_{i=1}^n \subset \mathcal{X}$ and weights $\{w_i\}_{i=1}^n \subset \mathbb{R}$. 
The term \textit{quadrature rule} is sometimes preferred when the domain of integration is one-dimensional (i.e. $d = 1$). 
\textcolor{black}{The notation $\hat{\Pi}[f]$ is motivated by the fact that this expression can be re-written as the integral of $f$ with respect to an empirical measure $\hat{\pi} = \sum_{i=1}^n w_i \delta_{\bm{x}_i}$, where $\delta_{\bm{x}_i}$ is an atomic measure (i.e. for all $A \in \mathcal{B}$, $\delta_{\bm{x}_i}(A) = 1$ if $\bm{x}_i \in A$, $\delta_{\bm{x}_i}(A) = 0$ if $\bm{x}_i \notin A$). }
The weights $w_i$ can be negative and need not satisfy $\sum_{i=1}^n w_i = 1$. 

\subsection{Bayesian Cubature} \label{BQ introduce sec}

\textcolor{black}{Probabilistic integration begins by defining a probability space $(\Omega,\mathcal{F},\mathbb{P})$ and an associated stochastic process $g : \mathcal{X} \times \Omega \rightarrow \mathbb{R}$, such that for each $\omega \in \Omega$, $g(\cdot,\omega)$ belongs to a linear topological space $\mathcal{L}$.
For BC, \cite{Larkin1972} considered a Gaussian process (GP); this is a stochastic process such that the random variables $\omega \mapsto Lg(\cdot,\omega)$ are Gaussian for all $L \in \mathcal{L}^*$, where $\mathcal{L}^*$ is the topological dual of $\mathcal{L}$ \citep{Bogachev1998}.
In this paper, to avoid technical obfuscation, it is assumed that $\mathcal{L}$ contains only continuous functions. }
Let $\mathbb{E}_\omega$ denote expectation taken over $\omega \sim \mathbb{P}$.
A GP can be characterised by its mean function $m(\bm{x}) = \mathbb{E}_\omega[g(\bm{x},\omega)]$, and its covariance function $c(\bm{x},\bm{x}') = \mathbb{E}_\omega[(g(\bm{x},\omega) - m(\bm{x})) (g(\bm{x}',\omega) - m(\bm{x}'))]$ and we write $g \sim \mathcal{N}(m,c)$. 
\textcolor{black}{In this paper we assume without loss of generality that $m \equiv 0$.}
Note that other priors could also be used (e.g. a Student-t process affords heavier tails for values assumed by the integrand). 

\textcolor{black}{The next step is to consider the restriction of $\mathbb{P}$ to the set $\{\omega \in \Omega : g(\bm{x}_i,\omega) = f(\bm{x}_i), 1 \leq i \leq n\}$ to induce a posterior measure $\mathbb{P}_n$ over $\mathcal{L}$.
The fact that $\mathcal{L}$ contains only continuous functions ensures that $g(\bm{x}_i,\omega)$ is well-defined\footnote{\textcolor{black}{this would not have been the case if instead $\mathcal{L} = L^2(\pi)$.}}.
Moreover the restriction to a $\mathbb{P}$-null set is also well-defined\footnote{\textcolor{black}{since the canonical space of continuous processes is a Polish space and all Polish spaces are Borel spaces and thus admit regular conditional laws \citep[c.f. Theorem A1.2 and Theorem 6.3 of][]{Kallenberg2002}.}}. }
Then, for BC, $\mathbb{P}_n$ can be shown to be a GP, denoted $\mathcal{N}(m_n,c_n)$ \citep[see Chap. 2 of][]{Rasmussen2006}.

The final step is to produce a distribution on $\mathbb{R}$ by projecting the posterior $\mathbb{P}_n$ defined on $\mathcal{L}$ through the integration operator.
A sketch of the procedure is presented in Figure \ref{fig:sketch_BQ} and the relevant formulae are now provided.
Denote by $\mathbb{E}_n$, $\mathbb{V}_n$ the expectation and variance taken with respect to $\mathbb{P}_n$.
Write $\bm{f} \in \mathbb{R}^n$ for the vector of $f_i = f(\bm{x}_i)$ values, $X = \{\bm{x}_i\}_{i=1}^n$ and $\bm{c}(\bm{x},X) = \bm{c}(X,\bm{x})^\top $ for the $1 \times n$ vector whose $i$th entry is $c(\bm{x},\bm{x}_i)$ and $\bm{C}$ for the matrix with entries $C_{i,j} = c(\bm{x}_i,\bm{x}_j)$.
\begin{proposition} \label{prop: derive mean and var}
The induced distribution of $\Pi[g]$ is Gaussian with mean and variance
\begin{eqnarray}
\mathbb{E}_n [\Pi[g]] & = & \Pi[\bm{c}(\cdot,X)] \bm{C}^{-1} \bm{f} \label{eq:posterior_mean} \\
\mathbb{V}_n [\Pi[g]] & = & \Pi\Pi[c(\cdot,\cdot)] - \Pi[\bm{c}(\cdot,X)] \bm{C}^{-1} \Pi[\bm{c}(X,\cdot)]. \label{eq:posterior_var}
\end{eqnarray}
\end{proposition}
\noindent Here, $\Pi\Pi[c(\cdot,\cdot)]$ denotes the integral of $c$ with respect to each argument. All proofs in this paper are reserved for Supplement \ref{appendix:general_RKHS}.
\textcolor{black}{It can be seen that the computational cost of obtaining this full posterior is much higher than that of obtaining a point estimate for the integral, at $O(n^3)$. However, certain combinations of point sets and covariance functions can reduce this cost by several orders of magnitude (see e.g. \cite{Karvonen2017}).}

\begin{figure}[t]
\center
\includegraphics[width = 0.75\textwidth]{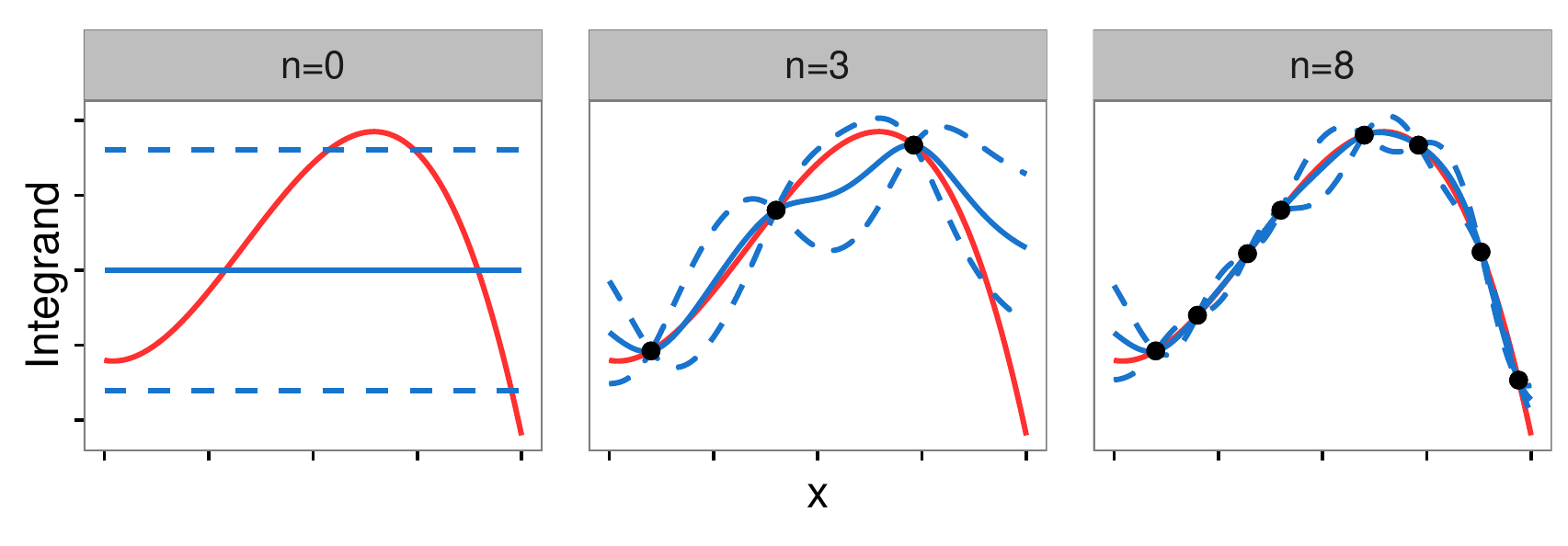}
\vspace{-9mm}
\center
\includegraphics[width = 0.75\textwidth]{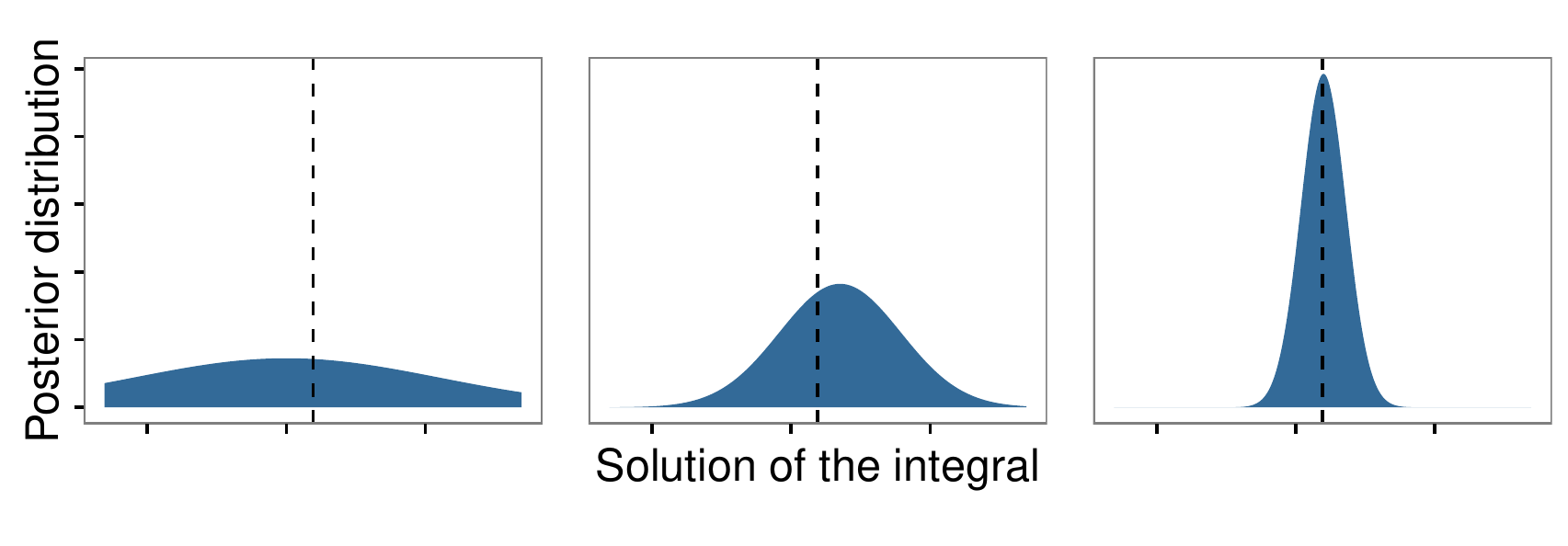}
\vspace{-3mm}
\caption{Sketch of Bayesian cubature. The top row shows the approximation of the integrand $f$ (red) by the posterior mean $m_n$ (blue) as the number $n$ of function evaluations is increased. The dashed lines represent point-wise $95\%$ posterior credible intervals. 
The bottom row shows the Gaussian distribution with mean $\mathbb{E}_n[\Pi[g]]$ and variance $\mathbb{V}_n[\Pi[g]]$ and the dashed black line gives the true value of the integral $\Pi[f]$. 
As the number of states $n$ increased, this posterior distribution contracts onto the true value of the integral $\Pi[f]$.}
\label{fig:sketch_BQ}
\end{figure}

BC formally associates the stochastic process $g$ with a prior model for the integrand $f$.
This in turn provides a probabilistic model for epistemic uncertainty over the value of the integral $\Pi[f]$.
Without loss of generality we assume $m \equiv 0$ for the remainder of the paper.
Then Eqn. \ref{eq:posterior_mean} takes the form of a \textcolor{black}{cubature} rule
\begin{equation}
\mathbb{E}_n[\Pi[g]]
= \hat{\Pi}_{\text{BC}}[f]
:= \sum_{i=1}^n w^{\text{BC}}_i f(\bm{x}_i) \label{post mean first}
\end{equation}
where $\bm{w}^{\text{BC}} := \bm{C}^{-1} \Pi[\bm{c}(X,\cdot)]$. 
Furthermore, Eqn. \ref{eq:posterior_var} does not depend on function values $\{f_i\}_{i=1}^n$, but only on the location of the states $\{\bm{x}_i\}_{i=1}^n$ and the choice of covariance function $c$. 
This is useful as it allows state locations and weights to be pre-computed and re-used. 
\textcolor{black}{However, it also means that the variance is endogeneous, being driven by the choice of prior. }
A valid quantification of uncertainty thus relies on a well-specified prior; we consider this issue further in Sec. \ref{sec:calibration}.

The BC mean (Eqn. \ref{post mean first}) coincides with classical \textcolor{black}{cubature} rules for specific choices of covariance function $c$.
For example, in one dimension a Brownian covariance function $c(x,x') = \min(x,x')$ leads to a posterior mean $m_n$ that is a piecewise linear interpolant of $f$ between the states $\{x_i\}_{i=1}^n$, i.e. the trapezium rule \citep{Suldin1959}.
\textcolor{black}{Similarly, \cite{Sarkka2015} constructed a covariance function $c$ for which Gauss-Hermite \textcolor{black}{cubature} is recovered, and \cite{Karvonen2017classical} showed how other polynomial-based \textcolor{black}{cubature} rules can be recovered.}
Clearly the point estimator in Eqn. \ref{post mean first} is a natural object; it has also received attention in both the \emph{kernel quadrature} literature \citep{Sommariva2006} and \emph{empirical interpolation} literature \citep{Kristoffersen2013}.
Recent work with a computational focus includes \cite{Kennedy1998,Minka2000,Rasmussen2003,Huszar2012,Gunter2014,Briol2015,Karvonen2017,Oettershagen2017}. 
\textcolor{black}{The present paper focuses on the full posterior, as opposed to just the point estimator that these papers studied.}

%%%%%%%%%%%%%%%%%%%%%%%%%%%%%%%%%%%%%%%%%%%%%%%%%%%%%%%%%%%%%%%%%%%%%%%%%%%%%%%%%%%%%%%%%%%%%%%%%%%%%%%%%%%%%%%%%%%%%%%%%%%%%%%%%%%%%%%%%%%%%%%%%%%%%%%%%%%%%%%%%%%%%%%%%%%%%%%%%%%%%%%

\subsection{\textcolor{black}{Cubature} Rules in Hilbert Spaces} \label{duality sec}

Next we review how analysis of the approximation properties of the \textcolor{black}{cubature} rule $\hat{\Pi}_{\text{BC}}[f]$ can be carried out in terms of reproducing kernel Hilbert spaces \citep[RKHS;][]{Berlinet2004}. 

Consider a Hilbert space $\mathcal{H}$ with inner product $\langle\cdot,\cdot\rangle_{\mathcal{H}}$ and associated norm $\|\cdot\|_{\mathcal{H}}$.
$\mathcal{H}$ is said to be an RKHS if there exists a symmetric, positive definite function $k:\mathcal{X}\times \mathcal{X}\rightarrow \mathbb{R}$, called a \emph{kernel}, that satisfies two properties: (i) $k(\cdot,\bm{x}) \in \mathcal{H}$ for all $\bm{x} \in \mathcal{X}$ and; (ii) $f(\bm{x}) = \langle f, k(\cdot,\bm{x}) \rangle_\mathcal{H}$ for all $\bm{x} \in \mathcal{X}$ and $f \in \mathcal{H}$ (the \emph{reproducing} property). 
It can be shown that every kernel defines a RKHS and every RKHS admits a unique reproducing kernel \citep[][Sec. 1.3]{Berlinet2004}. 
In this paper all kernels $k$ are assumed to satisfy $R := \int k(\bm{x},\bm{x}) \; \mathrm{d}\pi(\bm{x}) < \infty$.
In particular this guarantees $\int f^2 \; \mathrm{d}\pi < \infty$ for all $f \in \mathcal{H}$.
Define the \textit{kernel mean} $\mu(\pi):\mathcal{X} \rightarrow \mathbb{R}$ as $\mu(\pi)(\bm{x}) := \Pi[k(\cdot,\bm{x})]$.
This exists in $\mathcal{H}$ as a consequence of $R < \infty$ \citep{Smola2007}. 
The name is justified by the fact\footnote{the integral and inner product commute due to the existence of $\mu(\pi)$ as a Bochner integral \citep[p510]{Steinwart2008}.} that
\begin{equation*}
\begin{split}
\forall f \in \mathcal{H} : \qquad \Pi[f] = \int f \;\mathrm{d}\pi 
& =  \int \big\langle f, k(\cdot,\bm{x}) \big\rangle_{\mathcal{H}} \; \mathrm{d}\pi(\bm{x})  \\
& =  \Big\langle f , \int k(\cdot,\bm{x}) \; \mathrm{d}\pi(\bm{x}) \Big\rangle_{\mathcal{H}}
= \langle f, \mu(\pi) \rangle_{\mathcal{H}}.
\end{split}
\end{equation*}
The reproducing property permits an elegant theoretical analysis, with many quantities of interest tractable in $\mathcal{H}$.
In the language of kernel means, \textcolor{black}{cubature} rules of the form in Eqn. \ref{eq:quadrature} can be written as $\hat{\Pi}[f] = \langle f , \mu(\hat{\pi}) \rangle_{\mathcal{H}}$ where $\mu(\hat{\pi})$ is the approximation to the kernel mean given by $\mu(\hat{\pi})(\bm{x}) = \hat{\Pi}[k(\cdot,\bm{x})]$.
For fixed $f \in \mathcal{H}$, the integration error associated with $\hat{\Pi}$ can then be expressed as
\begin{equation*}
\hat{\Pi}[f] - \Pi[f]
= \langle f, \mu(\hat{\pi}) \rangle_{\mathcal{H}}  - \langle f, \mu(\pi) \rangle_{\mathcal{H}}  
= \langle f, \mu(\hat{\pi}) - \mu(\pi) \rangle_{\mathcal{H}}.
\end{equation*}
A tight upper bound for the error is obtained by Cauchy-Schwarz:
\begin{equation} 
|\hat{\Pi}[f] - \Pi[f] |  \leq \|f \|_{\mathcal{H}} \|  \mu(\hat{\pi}) - \mu(\pi) \|_{\mathcal{H}}.
\label{CSMT}
\end{equation}
\textcolor{black}{The expression above\footnote{\textcolor{black}{sometimes called the Koksma-Hlawka inequality \citep{Hickernell1998}.}} decouples the magnitude (in $\mathcal{H}$) of the integrand $f$ from the kernel mean approximation error.}
The following sections discuss how \textcolor{black}{cubature} rules can be tailored to target the second term in this upper bound. 

\subsection{\textcolor{black}{Optimality of \textcolor{black}{Cubature} Weights}} \label{duality sec 2}

Denote the dual space of $\mathcal{H}$ as $\mathcal{H}^*$ and denote its corresponding norm $\|\cdot\|_{\mathcal{H}^*}$.
The performance of a \textcolor{black}{cubature} rule can be quantified by its \textit{worst-case error} (WCE) in the RKHS:
$$
\|\hat{\Pi} - \Pi\|_{\mathcal{H}^*} = \sup_{\|f\|_{\mathcal{H}} \leq 1} |\hat{\Pi}[f] - \Pi[f]|
$$
The WCE is characterised as the error in estimating the kernel mean:
\begin{fact} \label{mmd and mean element}
$\|\hat{\Pi} - \Pi\|_{\mathcal{H}^*} = \|  \mu(\hat{\pi}) - \mu(\pi) \|_{\mathcal{H}}$.
\end{fact}
Minimisation of the WCE is natural and corresponds to solving a least-squares problem in the feature space induced by the kernel:
Let $\bm{w} \in \mathbb{R}^n$ denote the vector of weights $\{w_i\}_{i=1}^n$, $\bm{z} \in \mathbb{R}^n$ be a vector such that $z_i = \mu(\pi)(\bm{x}_i)$, and $\bm{K} \in \mathbb{R}^{n \times n}$ be the matrix with entries $K_{i,j} = k(\bm{x}_i,\bm{x}_j)$.
Then we obtain the following:
\begin{fact} \label{prop:WCE_kernelmean_equivalence}
$\|\hat{\Pi} - \Pi\|_{\mathcal{H}^*}^2 =  \bm{w}^\top  \bm{K} \bm{w} - 2 \bm{w}^\top  \bm{z}  +  \Pi[\mu(\pi)]$.
\end{fact}

Several optimality properties for integration in RKHS were collated in Sec. 4.2 of \cite{Novak2008}. 
Relevant to this work is that an optimal estimate $\hat{\Pi}$ can, without loss of generality, take the form of a \textcolor{black}{cubature} rule (i.e. of the form $\hat{\Pi}$ in Eqn. \ref{eq:quadrature}). 
To be more precise, any non-linear and/or adaptive estimator can be matched\footnote{of course, adaptive \textcolor{black}{cubature} may provide superior performance for a \emph{single} fixed function $f$, and the minimax result is not true in general outside the RKHS framework.} in terms of asymptotic WCE by a \textcolor{black}{cubature} rule as we have defined. 

To relate these ideas to BC, consider the challenge of deriving an optimal \textcolor{black}{cubature} rule, conditional on fixed states $\{\bm{x}_i\}_{i=1}^n$, that minimises the WCE (in the RKHS $\mathcal{H}_k$) over weights $\bm{w} \in \mathbb{R}^n$.
From Fact \ref{prop:WCE_kernelmean_equivalence}, the solution to this convex problem is $\bm{w} = \bm{K}^{-1}\bm{z}$.
\textcolor{black}{This shows that \emph{if} the reproducing kernel $k$ is equal to the covariance function $c$ of the GP, then the MAP from BC is identical to the {\it optimal} \textcolor{black}{cubature} rule in the RKHS \citep{Kadane1985}. 
Furthermore, with $k = c$, the expression for the WCE in Fact \ref{prop:WCE_kernelmean_equivalence} shows that $\mathbb{V}_n[\Pi[g]] = \|\hat{\Pi}_{\text{BC}} - \Pi\|_{\mathcal{H}^*}^2 \leq \|\hat{\Pi} - \Pi\|_{\mathcal{H}^*}^2$ where $\hat{\Pi}$ is any other \textcolor{black}{cubature} rule $\hat{\Pi}$ based on the same states $\{\bm{x}_i\}_{i=1}^n$. }
Regarding optimality, the problem is thus reduced to selection of states $\{\bm{x}_i\}_{i=1}^n$.

%%%%%%%%%%%%%%%%%%%%%%%%%%%%%%%%%%%%%%%%%%%%%%%%%%%%%%%%%%%%%%%%%%%%%%%%%%%%%%%%%%%%%%%%%%%%%%%%%%%%%%%%%%%%%%%%%%%%%%%%%%%%%%%%%%%%%%%%%%%%
%%%%%%%%%%%%%%%%%%%%%%%%%%%%%%%%%%%%%%%%%%%%%%%%%%%%%%%%%%%%%%%%%%%%%%%%%%%%%%%%%%%%%%%%%%%%%%%

\subsection{Selection of States} \label{how to sample}

In earlier work, \cite{OHagan1991} considered states $\{\bm{x}_i\}_{i=1}^n$ that are employed in Gaussian \textcolor{black}{cubature} methods.
\cite{Rasmussen2003} generated states using Monte Carlo (MC), calling the approach \emph{Bayesian} MC (BMC).
Recent work by \cite{Gunter2014,Briol2015} selected states using experimental design to target the variance $\mathbb{V}_n[\Pi[g]]$.
These approaches are now briefly recalled.

\subsubsection{Monte Carlo Methods}

An MC method is a \textcolor{black}{cubature} rule based on uniform weights $w_i^{\text{MC}} := 1/n$ and random states $\{\bm{x}_i\}_{i=1}^n$. 
The simplest of those methods consists of sampling states $\{\bm{x}_i^{\text{MC}}\}_{i=1}^n$ independently from $\pi$.
For un-normalised densities, Markov chain Monte Carlo (MCMC) methods proceed similarly but induce a dependence structure among the $\{\bm{x}_i^{\text{MCMC}}\}_{i=1}^n$.
We denote these (random) estimators by $\hat{\Pi}_{\text{MC}}$ (when $\bm{x}_i = \bm{x}_i^{\text{MC}}$) and $\hat{\Pi}_{\text{MCMC}}$ (when $\bm{x}_i = \bm{x}_i^{\text{MCMC}}$).
Uniformly weighted estimators are well-suited to many challenging integration problems since they provide a dimension-independent convergence rate for the WCE of $O_P(n^{-1/2})$.
They are widely applicable and straight-forward to analyse; for instance the central limit theorem (CLT) gives that
$\sqrt{n} ( \hat{\Pi}_{\text{MC}}[f]  - \Pi[f] ) \rightarrow \mathcal{N} (0,\tau_f^{-1} )$
where $\tau_f^{-1} = \Pi[f^2] - \Pi[f]^2$ and the convergence is in distribution. 
However, the CLT \textcolor{black}{may not be} well-suited as a measure of \emph{epistemic} uncertainty (i.e. as an explicit model for numerical error) since (i) it is only valid asymptotically, and (ii) $\tau_f$ is unknown, depending on the integral $\Pi[f]$ being estimated.

Quasi Monte Carlo (QMC) methods exploit knowledge of the RKHS $\mathcal{H}$ to spread the states in an efficient, deterministic way over the domain $\mathcal{X}$ \citep{Hickernell1998}. 
QMC also approximates integrals using a \textcolor{black}{cubature} rule $\hat{\Pi}_{\text{QMC}}[f]$ that has uniform weights $w_i^{\text{QMC}} := 1/n$.
The (in some cases) optimal convergence rates, as well as sound statistical properties, of QMC have recently led to interest within statistics \citep[e.g.][]{Gerber2015,Buchholz2017}. 
\textcolor{black}{A related method with non-uniform weights was explored in \cite{Stein1995,Stein1995b}.}

\subsubsection{Experimental Design Methods}

An \emph{Optimal} BC (OBC) rule selects states $\{\bm{x}_i\}_{i=1}^n$ to globally minimise the variance $\mathbb{V}_n[\Pi[f]]$. 
OBC corresponds to classical \textcolor{black}{cubature} rules (e.g. Gauss-Hermite) for specific choices of kernels \citep{Karvonen2017}. 
However OBC cannot in general be implemented; the problem of optimising states is in general NP-hard \citep[][Sec. 10.2.3]{Scholkopf2002}. 

A more pragmatic approach to select states is to use experimental design methods, such as the greedy algorithm that sequentially minimises $\mathbb{V}_n[\Pi[g]]$. 
This method, called \textit{sequential} BC (SBC), is straightforward to implement, e.g. using general-purpose numerical optimisation, and is a probabilistic integration method that is often used \citep{Osborne2012active,Gunter2014}.
More sophisticated optimisation algorithms have also been used:
For example, in the empirical interpolation literature, \cite{Eftang2012} proposed adaptive procedures to iteratively divide the domain of integration into sub-domains. 
In the BC literature, \cite{Briol2015} used conditional gradient algorithms for this task. 
A similar approach was recently considered in \cite{Oettershagen2017}.

At present, experimental design schemes do not possess the computational efficiency that we have come to expect from MCMC and QMC.
Moreover, they do not scale well to high-dimensional settings due to the need to repeatedly solve high-dimensional optimisation problems and have few established theoretical guarantees.
For these reasons we will focus next on MC, MCMC and QMC.

%%%%%%%%%%%%%%%%%%%%%%%%%%%%%%%%%%%%%%%%%%%%%%%%%%%%%%%%%%%%%%%%%%%%%%%%%%%%%%%%%%%%%%%%%%%%%%%%%%%%%%%%%%%%%%%%%%%%%%%%%%%%%%%%%%%%%%%%%%%%%%%%%%%%%%%%%%%%%%%%%%%%%%%%%%%%%%%%%%%%%%%%%%%%%%%%%%%%%%%%%%%%%%%%%%%%%%%%%%%%%%%%%%%%%%%%%%%%%%%%%%%%%%%%%%%%%%%%%%%%%%%%%%%%%%%%%%%%%%%%%%%%

\section{Methods}\label{sec:theoretical_results}

This section presents novel theoretical results on probabilistic integration methods in which the states $\{\bm{x}_i\}_{i=1}^n$ are generated with MCMC and QMC.
Sec. \ref{sec:fromMC_toBQ} provides formal definitions, while Sec. \ref{sec:theory_Sobolev} establishes theoretical results.

\subsection{Probabilistic Integration}\label{sec:fromMC_toBQ}

The sampling methods of MCMC and, to a lesser extent, QMC are widely used in statistical computation.
Here we pursue the idea of using MCMC and QMC to generate states for BC, with the aim to exploit BC to account for the possible impact of numerical integration error on inferences made in statistical applications. 
In MCMC it is possible that two states $\bm{x}_i = \bm{x}_j$ are identical.
To prevent the kernel matrix $\bm{K}$ from becoming singular, duplicate states should be discarded\footnote{this is justified since the information contained in function evaluations $f_i = f_j$ is not lost.
This does {\it not} introduce additional bias into BC methods, in contrast to MC methods.}.
Then we define $\hat{\Pi}_{\text{BMCMC}}[f] := \sum_{i=1}^n w_i^{\text{BC}} f (\bm{x}_i^{\text{MCMC}})$ and $\hat{\Pi}_{\text{BQMC}}[f] := \sum_{i=1}^n w_i^{\text{BC}} f (\bm{x}_i^{\text{QMC}})$.
This two-step procedure requires no modification to existing MCMC or QMC sampling methods.
Each estimator is associated with a full posterior distribution, described in Sec. \ref{BQ introduce sec}.

\textcolor{black}{A moment is taken to emphasise that the apparently simple act of re-weighting MCMC samples can have a dramatic improvement on convergence rates for integration of a sufficiently smooth integrand. }
Whilst our main interest is in the suitability of BC as a statistical model for discretisation of an integral, we highlight the efficient point estimation which comes out as a by-product. 

To date we are not aware of any previous use of BMCMC, presumably due to analytic intractability of the kernel mean when $\pi$ is un-normalised. 
BQMC has been described by \cite{Hickernell2005,marques2013,Sarkka2015}.
To the best of our knowledge there has been no theoretical analysis of the posterior distributions associated with either method.
The goal of the next section is to establish these fundamental results.

%%%%%%%%%%%%%%%%%%%%%%%%%%%%%%%%%%%%%%%%%%%%%%%%%%%%%%%%%%%%%%%%%%%%%%%%%%%%%%%%%%%%%%%%%%%%%%%%%%%%%%%%%%%%%%%%%%%%%%%%%%%%%%%%%%%%%%%%%%%%%%%%%%%%%%%%%

\subsection{Theoretical Properties} \label{sec:theory_Sobolev}

In this section we present novel theoretical results for BMC, BMCMC and BQMC.
\textcolor{black}{The setting we consider assumes that the true integrand $f$ belongs to a RKHS $\mathcal{H}$ and that the GP prior is based on a covariance function $c$ which is identical to the kernel $k$ of $\mathcal{H}$.
That the GP is not supported on $\mathcal{H}$, but rather on a Hilbert scale of $\mathcal{H}$, is viewed as a technical detail:
Indeed, a GP can be constructed on $\mathcal{H}$ via $c(\bm{x},\bm{x}') = \int k(\bm{x},\bm{y}) k(\bm{y},\bm{x}') \mathrm{d}\pi(\bm{y})$ and a theoretical analysis similar to ours could be carried out \citep[Lemma 2.2 of][]{Cialenco2012}. }

%%%%%%%%%%%%%%%%%%%%%%%%%%%%%%%%%%%%%%%%%%%%%%%%%%%%%%%%%%%%%%%%%%%%%%%%%%%%%%%%%%%%%%%%%%%%%%%%%%%%%%%%%%%%%%%%%%%%%%%%%%%%%%%%%%%%%%%%%%%%%%%%%%%%%%%%%%%%%%%%%%%%%%%%%%%%%%%%

\subsubsection{Bayesian Markov Chain Monte Carlo} \label{sec:theory_BMC}

As a baseline, we begin by noting a general result for MC estimation.
This requires a slight strengthening of the assumption on the kernel: $k_{\max} := \sup_{\bm{x} \in \mathcal{X}} k(\bm{x},\bm{x}) < \infty$.
This implies that all $f \in \mathcal{H}$ are bounded on $\mathcal{X}$. 
For MC estimators, Lemma 33 of \cite{Song2008} show that, when $k_{\max} < \infty$, the WCE converges in probability at the classical rate $\|\hat{\Pi}_{\text{MC}} - \Pi\|_{\mathcal{H}^*}  = O_P(n^{-1/2})$.

Turning now to BMCMC (and BMC as a special case), we consider the compact manifold $\mathcal{X} = [0,1]^d$.
Below the distribution $\pi$ will be assumed to admit a density with respect to Lebesgue measure, denoted by $\pi(\cdot)$.
Define the \emph{Sobolev} space $\mathcal{H}_\alpha$ to consist of all measurable functions such that $\|f\|_{\mathcal{H},\alpha}^2 := \sum_{i_1 + \dots + i_d \leq \alpha} \| \partial_{x_1}^{i_1} \dots \partial_{x_d}^{i_d} f \|_2^2 < \infty$.
Here $\alpha$ is the \emph{order} of $\mathcal{H}_\alpha$ and $(\mathcal{H}_\alpha,\|\cdot\|_{\mathcal{H},\alpha})$ is a RKHS.
Derivative counting can hence be a principled approach for practitioners to choose a suitable RKHS. 
All results below apply to RKHS $\mathcal{H}$ that are norm-equivalent\footnote{two norms $\|\cdot\|$, $\|\cdot\|'$ on a vector space $\mathcal{H}$ are \emph{equivalent} when there exists constants $0 < C_1,C_2 < \infty$ such that for all $h \in \mathcal{H}$ we have $C_1 \|h\| \leq \|h\|' \leq C_2 \|h\|$.} to $\mathcal{H}_\alpha$, permitting flexibility in the choice of kernel.
Specific examples of kernels are provided in Sec. \ref{sec:intractable_mean_element}.

Our analysis below is based on the scattered data approximation literature \citep{Wendland2005}.
A minor technical assumption, that enables us to simplify the presentation of results below, is that the set $X = \{\bm{x}_i\}_{i=1}^n$ may be augmented with a finite, pre-determined set $Y = \{\bm{y}_i\}_{i=1}^m$ where $m$ does not increase with $n$.
Clearly this has no bearing on asymptotics. 
For measurable $A$ we write $\mathbb{P}_n[A] = \mathbb{E}_n[1_A]$ where $1_A$ is the indicator function of the event $A$.
\begin{theorem}[BMCMC in $\mathcal{H}_\alpha$] \label{prop:consistency_BMCMC}
Suppose $\pi$ is bounded away from zero on $\mathcal{X} = [0,1]^d$.
Let $\mathcal{H}$ be norm-equivalent to $\mathcal{H}_\alpha$ where $\alpha > d/2$, $\alpha \in \mathbb{N}$.
Suppose states are generated by a reversible, uniformly ergodic Markov chain that targets $\pi$.
Then $\|\hat{\Pi}_{\textnormal{BMCMC}} - \Pi\|_{\mathcal{H}^*} = O_P ( n^{ -\alpha / d + \epsilon} )$ and moreover, if $f \in \mathcal{H}$ and $\delta > 0$,
\textcolor{black}{$$\mathbb{P}_n\{\Pi[f] - \delta < \Pi[g] < \Pi[f] + \delta\} = 1 - O_P(\exp(-C_\delta n^{\frac{2\alpha}{d} - \epsilon})), $$ }
where $C_\delta >0$ depends on $\delta$ and $\epsilon > 0$ can be arbitrarily small.
\end{theorem}
\noindent This result shows the posterior distribution is well-behaved; the posterior distribution of $\Pi[g]$ concentrates in any open neighbourhood of the true integral $\Pi[f]$. 
This result does not address the frequentist coverage of the posterior, which is assessed empirically in Sec. \ref{sec:numerical_results}.

Although we do not focus on point estimation, a brief comment is warranted:
A lower bound on the WCE that can be attained by randomised algorithms in this setting is $O_P(n^{-\alpha/d - 1/2})$ \citep{Novak2010}. 
Thus our result shows that the point estimate is at most one MC rate away from being optimal\footnote{the control variate trick of \cite{Bakhvalov1959} can be used to achieve the optimal randomised WCE, but this steps outside of the Bayesian framework.}.
\cite{Bach2015} obtained a similar result for fixed $n$ and a specific importance sampling distribution; his analysis does not directly imply our asymptotic results and \emph{vice versa}. 
After completion of this work, similar results on point estimation appeared in \cite{Oettershagen2017,Bauer2017}.

Thm. \ref{prop:consistency_BMCMC} can be generalised in several directions. Firstly, we can consider more general domains $\mathcal{X}$. Specifically, the scattered data approximation bounds that are used in our proof apply to any compact domain $\mathcal{X} \subset \mathbb{R}^d$ that satisfies an \emph{interior cone condition} \citep[][p.28]{Wendland2005}.
Technical results in this direction were established in \cite{Oates2016CF2}.
Second, we can consider other spaces $\mathcal{H}$. 
For example, a slight extension of Thm. \ref{prop:consistency_BMCMC} shows that \textcolor{black}{certain infinitely differentiable kernels lead} to exponential rates for the WCE and super-exponential rates for posterior contraction. 
For brevity, details are omitted.

\subsubsection{Bayesian Quasi Monte Carlo}\label{sec:theory_BQMC}

The previous section focused on BMCMC in the Sobolev space $\mathcal{H}_\alpha$. 
To avoid repetition, here we consider more interesting spaces of functions whose \emph{mixed} partial derivatives exist, for which even faster convergence rates can be obtained using BQMC.
To formulate BQMC we must posit an RKHS \emph{a priori} and consider collections of states $\{\bm{x}_i^{\text{QMC}}\}_{i=1}^n$ that constitute a QMC point set tailored to the RKHS.

Consider $\mathcal{X} = [0,1]^d$ with $\pi$ uniform on $\mathcal{X}$. 
Define the Sobolev space of \emph{dominating mixed smoothness} $\mathcal{S}_\alpha$ to consist of functions for which $\|f\|_{\mathcal{S},\alpha}^2 :=  \sum_{\forall j : i_j \leq \alpha} \| \partial_{x_1}^{i_1} \dots \partial_{x_d}^{i_d} f \|_2^2 < \infty$.
Here $\alpha$ is the \emph{order} of the space and $(\mathcal{S}_\alpha,\|\cdot\|_{\mathcal{S},\alpha})$ is a RKHS.
To build intuition, note that $\mathcal{S}_\alpha$ is norm-equivalent to the RKHS generated by a tensor product of Mat\'{e}rn kernels \citep{Sickel2009}, or indeed a tensor product of any other univariate Sobolev space -generating kernel.

For a specific space such as $\mathcal{S}_\alpha$, we seek an appropriate QMC point set.
The \emph{higher-order digital} $(t,\alpha,1,\alpha m \times m, d)-$\emph{net} construction is an example of a QMC point set for $\mathcal{S}_\alpha$; for details we refer the reader to \cite{Dick2010} for details.

\begin{theorem}[BQMC in $\mathcal{S}_\alpha$]\label{theo:BQMC_sobolev}
Let $\mathcal{H}$ be norm-equivalent to $\mathcal{S}_\alpha$, where $\alpha \geq 2$, $\alpha \in \mathbb{N}$.
Suppose states are chosen according to a higher-order digital $(t,\alpha,1,\alpha m\times m,d)$ net over $\mathbb{Z}_b$ for some prime $b$ where $n = b^m$.
Then $\|\hat{\Pi}_{\textnormal{BQMC}} - \Pi\|_{\mathcal{H}^*} = O(n^{-\alpha + \epsilon})$ and , if $f \in \mathcal{S}_\alpha$ and $\delta > 0$,
\textcolor{black}{$$\mathbb{P}_n\{\Pi[f] - \delta < \Pi[g] < \Pi[f] + \delta\} = 1 - O(\exp(-C_\delta n^{2\alpha - \epsilon})), $$ }
where $C_\delta > 0$ depends on $\delta$ and $\epsilon > 0$ can be arbitrarily small.
\end{theorem}
\noindent This result shows that the posterior is again well-behaved.
Indeed, the rate of contraction is much faster in $\mathcal{S}_\alpha$ compared to $\mathcal{H}_\alpha$.
In terms of point estimation, this is the optimal rate for any \emph{deterministic} algorithm for integration of functions in $\mathcal{S}_\alpha$ \citep{Novak2010}.
These results should be understood to hold on the sub-sequence $n=b^m$, as QMC methods do not in general give guarantees for all $n \in \mathbb{N}$.
It is not clear how far this result can be generalised, in terms of $\pi$ and $\mathcal{X}$, compared to the result for BMCMC, since this would require the use of different QMC point sets.
%The case of QMC for infinitely differentiable kernels was recently studied in \cite{Fasshauer2012}; the results therein for Smolyak point sets imply (exponential) convergence and contraction rates for BQMC via the same arguments that we have made explicit for the space $\mathcal{S}_\alpha$.

\subsubsection{Summary}

In this section we established rates of posterior contraction for BMC, BMCMC and BQMC in a general Sobolev space context. 
These results are essential since they establish the sound properties of the posterior, which is shown to contract to the truth as more evaluations are made of the integrand.
Of course, the higher computational cost of up to $O(n^3)$ may restrict the applicability of the method in large-$n$ regimes. 
However, we emphasise that the motivation is to quantify the uncertainty induced from numerical integration, an important task which often justifies the higher computational cost.

%%%%%%%%%%%%%%%%%%%%%%%%%%%%%%%%%%%%%%%%%%%%%%%%%%%%%%%%%%%%%%%%%%%%%%%%%%%%%%%%%%%%%%%%

\section{Implementation}\label{sec:practical_issues}

So far we have established sound theoretical properties for BMCMC and BQMC under the assumption that the prior is well-specified.
Unfortunately, prior specification complicates the situation in practice since, given a test function $f$, there are an infinitude of RKHS to which $f$ belongs and the specific choice of this space will impact upon the performance of the method.
\textcolor{black}{In particular, the scale of the posterior is driven by the scale of the prior, so that the uncertainty quantification being provided is endogenous and, if the prior is not well-specified, this could mitigate the advantages of the probabilistic numerical framework. }
This important point is now discussed.

It is important to highlight a distinction between B(MC)MC and BQMC; for the former the choice of states does not depend on the RKHS.
For B(MC)MC this allows for the possibility of off-line specification of the kernel after evaluations of the integrand have been obtained, whereas for alternative methods the kernel must be stated up-front.
Our discussion below therefore centres on prior specification in relation to B(MC)MC, where several statistical techniques can be applied.

\subsection{Prior Specification}\label{sec:calibration}
The above theoretical results do not address the important issue of whether the scale of the posterior uncertainty provides an accurate reflection of the actual numerical error.
\textcolor{black}{This is closely related to the well-studied problem of prior specification in the kriging literature \citep{Stein1990,Xu2017}. }

Consider a parametric kernel $k(\bm{x},\bm{x}';\theta_l,\theta_s)$, with a distinction drawn here between \emph{scale} parameters $\theta_l$ and \emph{smoothness} parameters $\theta_s$.
The former are defined as parametrising the norm on $\mathcal{H}$, whereas the latter affect the set $\mathcal{H}$ itself. 
Selection of $\theta_l,\theta_s$ based on data can only be successful in the absence of acute sensitivity to these parameters.
\textcolor{black}{For scale parameters, a wide body of evidence demonstrates that this is usually not a concern \citep{Stein1990}. }
However, selection of smoothness parameters is an active area of theoretical research \citep[e.g.][]{Szabo2015}.
In some cases it is possible to elicit a smoothness parameter from physical or mathematical considerations, such as a known number of derivatives of the integrand. 
Our attention below is instead restricted to scale parameters, where several approaches are discussed in relation to their suitability for BC:

\subsubsection{Marginalisation}

A natural approach, from a Bayesian perspective, is to set a prior $p(\theta_l)$ on parameters $\theta_l$ and then to marginalise over $\theta_l$ to obtain a posterior over $\Pi[f]$. 
Recent results for a certain infinitely differentiable kernel establish minimax optimal rates for this approach, including in the practically relevant setting where $\pi$ is supported on a low-dimensional sub-mainfold of the ambient space $\mathcal{X}$ \citep{Yang2013}.
However, the act of marginalisation itself involves an intractable integral. While the computational cost of evaluating this integral will often be dwarfed by that of the integral $\Pi[f]$ of interest, marginalisation nevertheless introduces an additional undesirable computational challenge that might require several approximations \citep[e.g.][]{Osborne2010}. 
\textcolor{black}{It is however possible to analytically marginalise certain types of scale parameters, such as amplitude parameters:
\begin{proposition} \label{Stdt prop}
Suppose our covariance function takes the form $c(\bm{x},\bm{y};\lambda) = \lambda c_0(\bm{x},\bm{y})$ where $c_0:\mathcal{X}\times \mathcal{X} \rightarrow \mathbb{R}$ is itself a reproducing kernel and $\lambda>0$ is an amplitude parameter. 
Consider the improper prior $p(\lambda) \propto \frac{1}{\lambda}$.
Then the posterior marginal for $\Pi[g]$ is a Student-t distribution with mean and variance
\begin{eqnarray*}
\Pi\left[c_0(\cdot,X)\right]\bm{C}_0^{-1}\bm{f}, \quad \frac{\bm{f}^\top \bm{C}_0^{-1} \bm{f}}{n} \{ \Pi\Pi[c_0(\cdot,\cdot)]-\Pi[\bm{c}_0(\cdot,X)]\bm{C}_0^{-1}\Pi[\bm{c}_0(X,\cdot)] \}
\end{eqnarray*}
and $n$ degrees of freedom.
Here $[\bm{C}_0]_{i,j} = c_0(\bm{x}_i,\bm{x}_j)$, $[\bm{c}_0(\cdot,X)]_i = c_0(\cdot,\bm{x}_i)$, $\bm{c}_0(\cdot,X) = \bm{c}_0(X,\cdot)^\top$.
\end{proposition}
}

\subsubsection{Cross-Validation}

Another approach to kernel choice is cross-validation. 
However, this can perform poorly when the number $n$ of data is small, since the data needs to be further reduced into training and test sets. The performance estimates are also known to have large variance in those cases \citep[Chap. 5 of][]{Rasmussen2006}. 
Since the small $n$ scenario is one of our primary settings of interest for BC, we felt that cross-validation was unsuitable for use in applications below.

\subsubsection{Empirical Bayes}

An alternative to the above approaches is \emph{empirical Bayes} (EB) selection of scale parameters, choosing $\theta_l$ to maximise the log-marginal likelihood of the data $f(\bm{x}_i)$, $i = 1,\dots,n$ \citep[Sec. 5.4.1 of][]{Rasmussen2006}.
EB has the advantage of providing an objective function that is easier to optimise relative to cross-validation.
However, we also note that EB can lead to over-confidence when $n$ is very small, since the full irregularity of the integrand has yet to be uncovered \citep{Szabo2015}.
\textcolor{black}{In addition, it can be shown that EB estimates need not converge as $n \rightarrow \infty$ when the GP is supported on infinitely differentiable functions \citep{Xu2017}. }

{\color{black} For the remainder, we chose to focus on a combination of the marginalisation approach for amplitude parameters and the EB approach for remaining scale parameters.
Empirical results support the use of this approach, though we do not claim that this strategy is optimal.}

\subsection{Tractable and Intractable Kernel Means}\label{sec:intractable_mean_element}

BC requires that the kernel mean $\mu(\pi)(\bm{x}) = \Pi[k(\cdot,\bm{x})]$ is available in closed-form.
This is the case for several kernel-distribution pairs $(k,\pi)$ and a subset of these pairs are recorded in Table \ref{kernel mean pairs}. 
In the event that the kernel-distribution pair $(k,\pi)$ of interest does not lead to a closed-form kernel mean, it is sometimes possible to determine another kernel-density pair $(k',\pi')$ for which $\Pi'[k'(\cdot,\bm{x})]$ is available and such that (i) $\pi$ is absolutely continuous with respect to $\pi'$, so that the Radon-Nikodym derivative $\mathrm{d} \pi / \mathrm{d} \pi'$ exists, and (ii) $f \; \mathrm{d} \pi / \mathrm{d} \pi' \in \mathcal{H}(k')$.
Then one can construct an importance sampling estimator 
\begin{equation}
\Pi[f] = \int f \; \mathrm{d}\pi = \int f \; \frac{\mathrm{d} \pi}{\mathrm{d} \pi'} \; \mathrm{d}\pi' = \Pi' \left[ f \frac{\mathrm{d} \pi}{\mathrm{d} \pi'} \right]
\label{importance eq}
\end{equation}
and proceed as above \citep{OHagan1991}.

\begin{table}[t!]
\centering
\resizebox{\textwidth}{!}{
\begin{tabular}{|cccp{4.5cm}|} \hline
$\mathcal{X}$ & $\pi$ & $k$ & Reference \\ \hline \hline
$[0,1]^d$ & Unif$(\mathcal{X})$ & Wendland TP & \cite{Oates2016} \\
$[0,1]^d$ & Unif$(\mathcal{X})$ & Mat\'{e}rn Weighted TP & Sec. \ref{weighted Sob app} \\
$[0,1]^d$ & Unif$(\mathcal{X})$ & Exponentiated Quadratic & Use of error function \\
$\mathbb{R}^d$ & Mixt. of Gaussians & Exponentiated Quadratic & \cite{Kennedy1998} \\
$\mathbb{S}^d$ & Unif$(\mathcal{X})$ & Gegenbauer & Sec. \ref{sec:sphere_application} \\
Arbitrary & Unif$(\mathcal{X})$ / Mixt. of Gauss. & Trigonometric & Integration by parts\\
Arbitrary & Unif$(\mathcal{X})$ & Splines & \cite{Wahba1990} \\
Arbitrary & Known moments & Polynomial TP & \cite{Briol2015} \\
Arbitrary & Known $\partial \log \pi(\bm{x})$ & Gradient-based Kernel & \cite{Oates2016CF2,Oates2015a} \\ \hline
\end{tabular}
}
\caption{A non-exhaustive list of distribution $\pi$ and kernel $k$ pairs that provide a closed-form expression for both the kernel mean $\mu(\pi)(\bm{x}) = \Pi[k(\cdot,\bm{x})]$ and the initial error $\Pi[\mu(\pi)]$.
Here TP refers to the tensor product of one-dimensional kernels.}
\label{kernel mean pairs}
\vspace{-2mm}
\end{table}

One side contribution of this research was a novel and generic approach to accommodate intractability of the kernel mean in BC.
This is described in detail in Supplement \ref{scalability appendix}  and used in case studies \#1 and \#2 presented in Sec. \ref{sec:numerical_results}.

%%%%%%%%%%%%%%%%%%%%%%%%%%%%%%%%%%%%%%%%%%%%%%%%%%%%%%%%%%%%%%%%%%%%%%%%%%%%%%%
%%%%%%%%%%%%%%%%%%%%%%%%%%%%%%%%%%%%%%%%%%%%%%%%%%%%%%%%%%%%%%

\section{Results}\label{sec:numerical_results}

The aims of the following section are two-fold; (i) to validate the preceding theoretical analysis and (ii) to explore the use of probabilistic integrators in a range of problems arising in contemporary statistical applications.

%%%%%%%%%%%%%%%%%%%%%%%%%%%%%%%%%%%%%%%%%%%%%%%%%%%%%%%%%%%%%%%%%%

\subsection{Assessment of Uncertainty Quantification}\label{sec:test_functions}

Our focus below is on the uncertainty quantification provided by BC and, in particular, the performance of the hybrid marginalisation/EB approach to kernel parameters. 
To be clear, we are not concerned with accurate point estimation at low computational cost.
This is a well-studied problem that reaches far beyond the methods of this paper.
Rather, we are aiming to assess the suitability of the probabilistic description for integration error that is provided by BC.
Our motivation is expensive integrands, but to perform assessment in a controlled environment we considered inexpensive test functions of varying degrees of irregularity, whose integrals can be accurately approximated.
\textcolor{black}{These included a non-isotropic test function $f_j(\bm{x}) = \exp(\sin(C_j x_1)^2 - \|\bm{x}\|_2^2)$ with an ``easy'' setting $C_1 = 5$ and a ``hard'' setting $C_2=20$}. 
The hard test function is more variable and will hence be more difficult to approximate (see Fig. \ref{fig:testfunc_EB_BMC}). One realisation of states $\{\bm{x}_i\}_{i=1}^n$, generated independently and uniformly over $\mathcal{X} = [-5,5]^d$ (initially $d = 1$), was used to estimate the $\Pi[f_i]$. We work in an RKHS characterised by tensor products of Mat\'{e}rn kernels
\begin{equation*}
k_\alpha(\bm{x},\bm{x}') = \lambda \prod_{i=1}^d \frac{2^{1-\alpha}}{\Gamma(\alpha)} \left(\frac{\sqrt{2 \alpha}|x_i - x_i'|}{\textcolor{black}{\sigma_i}^2}\right)^\alpha K_\alpha \left(\frac{\sqrt{2\alpha} |x_i - x_i'|}{\textcolor{black}{\sigma_i}^2}\right),
\end{equation*}
where $K_\alpha$ is the modified Bessel function of the second kind.
Closed-form kernel means exist in this case for $\alpha = p + 1/2$ whenever $p \in \mathbb{N}$.
In this set-up, EB was used to select the length-scale \textcolor{black}{parameters $\bm{\sigma} = (\sigma_1,\ldots,\sigma_d) \in (0,\infty)^d$ of the kernel, while the amplitude parameter $\lambda$ was marginalised as in Prop. \ref{Stdt prop}}.
The smoothness parameter was fixed at $\alpha=7/2$.
Note that all test functions will be in the space $\mathcal{H}_{\alpha}$ for any $\alpha>0$ and there is a degree of arbitrariness in this choice of prior. 
\begin{figure}[t!]
\centering
\begin{minipage}{0.9\textwidth}
\begin{minipage}{0.35\textwidth}
\vspace{-8mm}
\includegraphics[trim={6cm 11cm 7cm 10cm},clip,width = \textwidth]{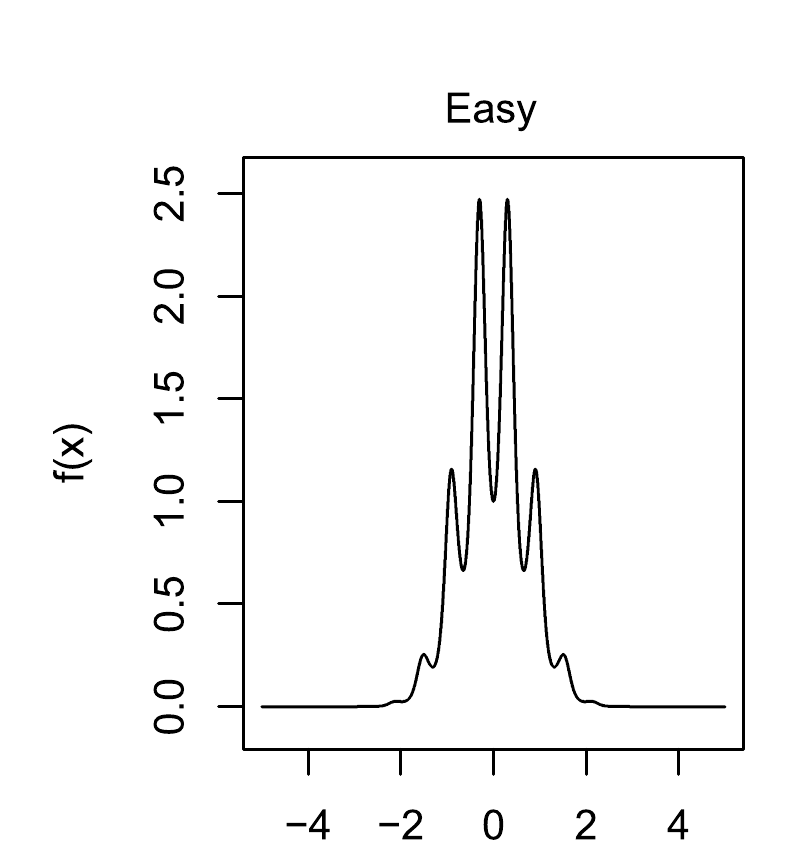} 
\end{minipage}
\begin{minipage}{0.65\textwidth}
\includegraphics[trim={0.5cm 1cm 1.25cm 0},clip,width = 0.92\textwidth]{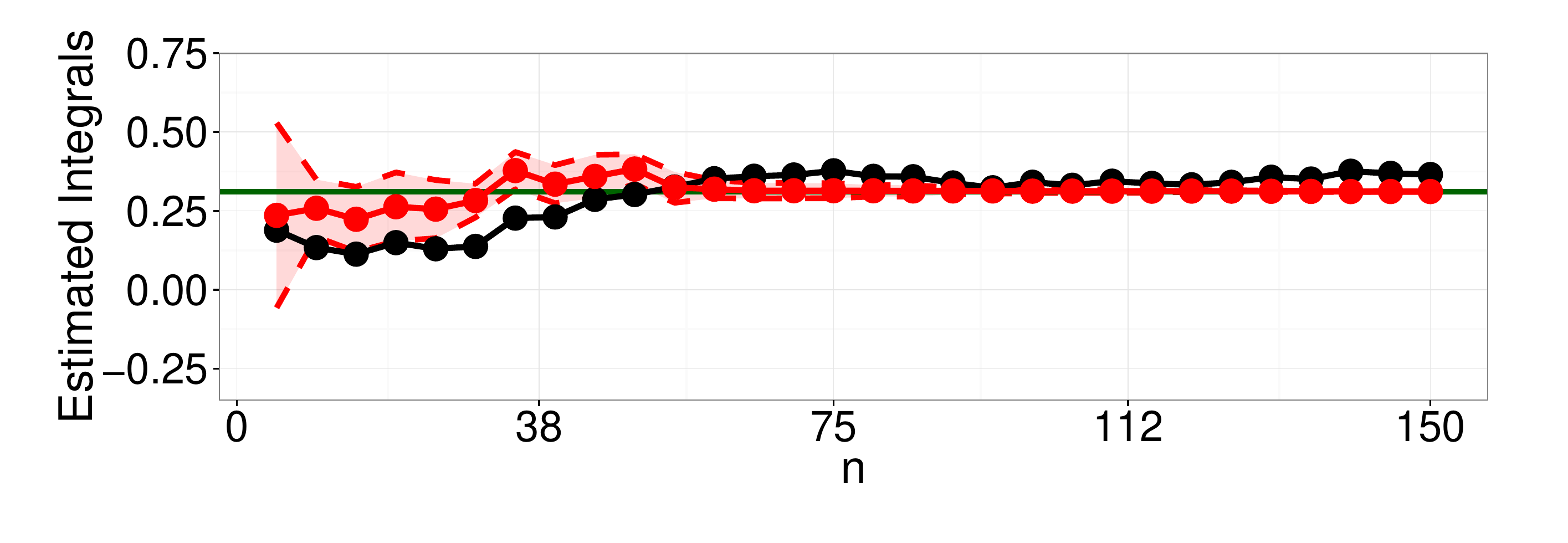}

\hspace{-5pt} \includegraphics[trim={-0.3cm 0 0 0},clip,width = 0.97\textwidth]{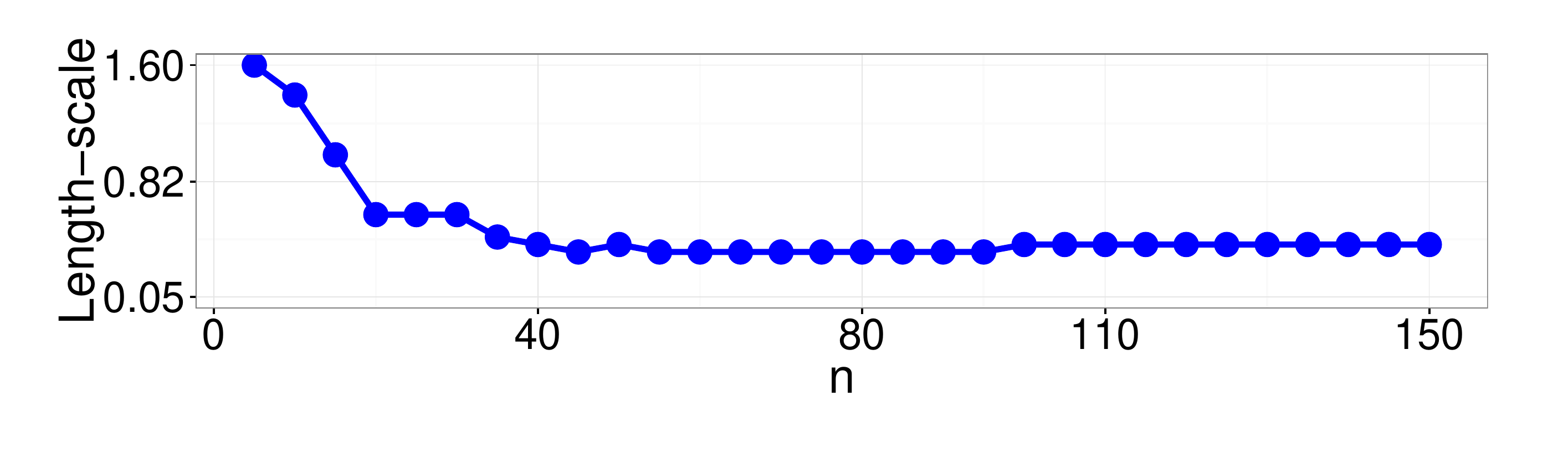}
\end{minipage}
\end{minipage}
%%%%%%%%%%%%%%%%%%%%%%%%
\begin{minipage}{0.9\textwidth}
\vspace{-3mm}
\begin{minipage}{0.35\textwidth}
\vspace{-8mm}
\includegraphics[trim={6cm 11cm 7cm 10cm},clip,width = \textwidth]{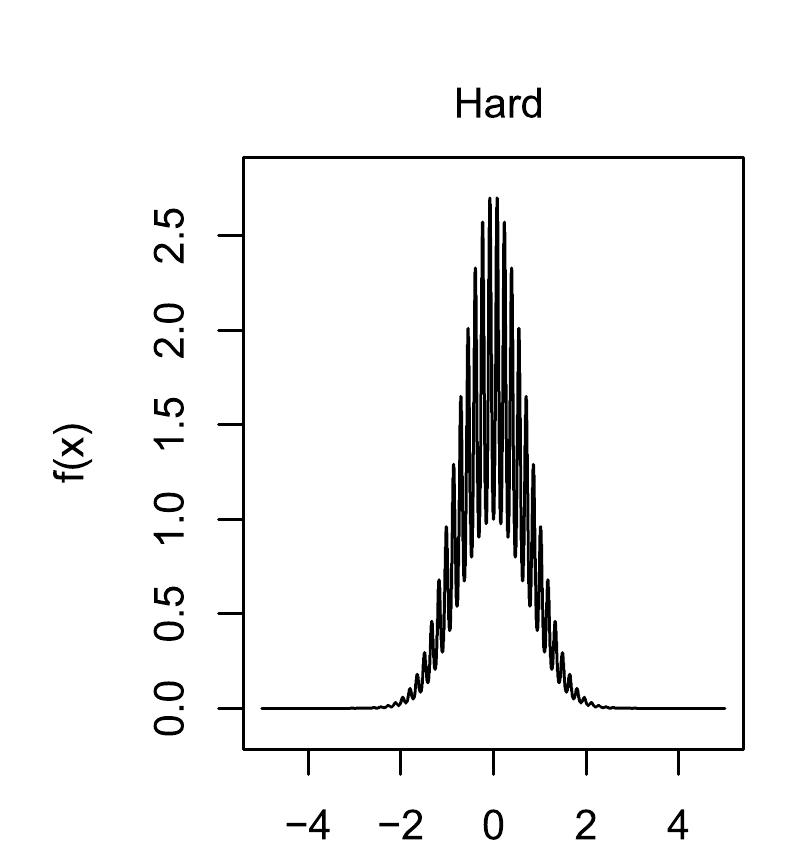} 
\end{minipage}
\begin{minipage}{0.65\textwidth}
\includegraphics[trim={0.5cm 1cm 1.25cm 0},clip,width = 0.92\textwidth]{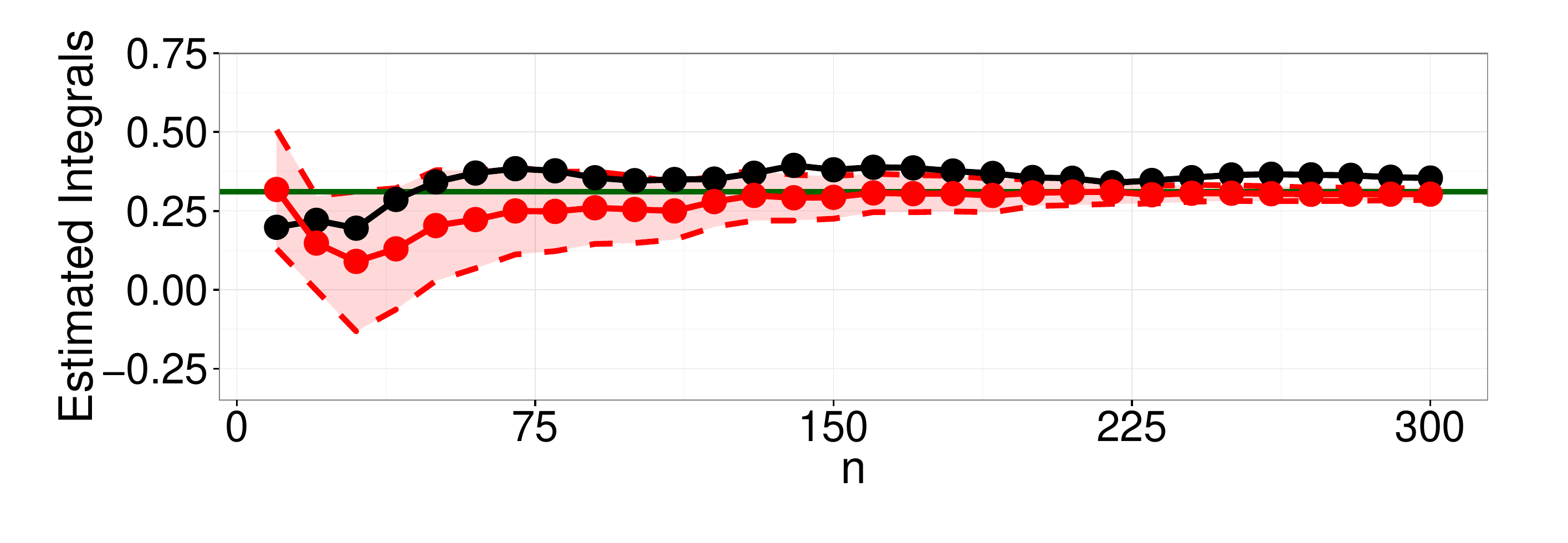}

\hspace{-5pt} \includegraphics[trim={-0.3cm 0 0 0},clip,width = 0.97\textwidth]{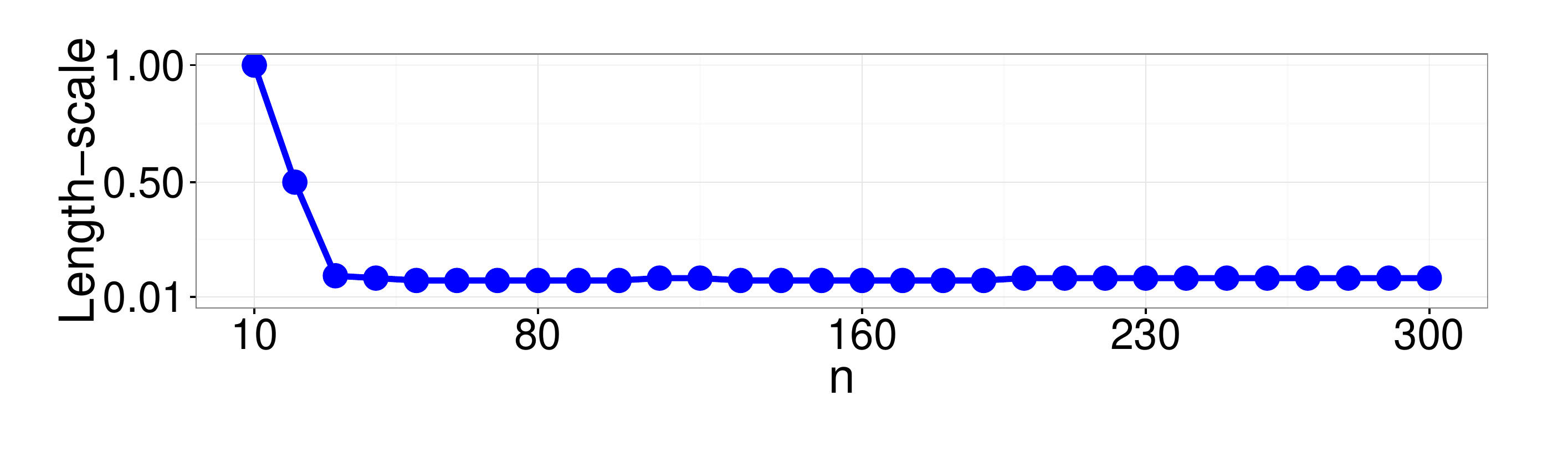}
\end{minipage}
\end{minipage}
\caption{\textcolor{black}{Evaluation of uncertainty quantification provided by BC.
Here we used empirical Bayes (EB) for $\bm{\sigma}$ with $\lambda$ marginalised. 
\textit{Left:} The test functions $f_1$ (top), $f_2$ (bottom) in $d = 1$ dimension. 
\textit{Right:} Solutions provided by Monte Carlo (MC; black) and Bayesian MC (BMC; red), for one typical realisation. 
$95\%$ credible regions are shown for BMC and the green horizontal line gives the true value of the integral. 
The blue curve gives the corresponding lengthscale parameter selected by EB.
}}
\label{fig:testfunc_EB_BMC}
\end{figure}

\begin{figure}[t!]
\centering
\includegraphics[width = 0.4\textwidth,trim={2cm 1cm 3cm 1cm},clip]{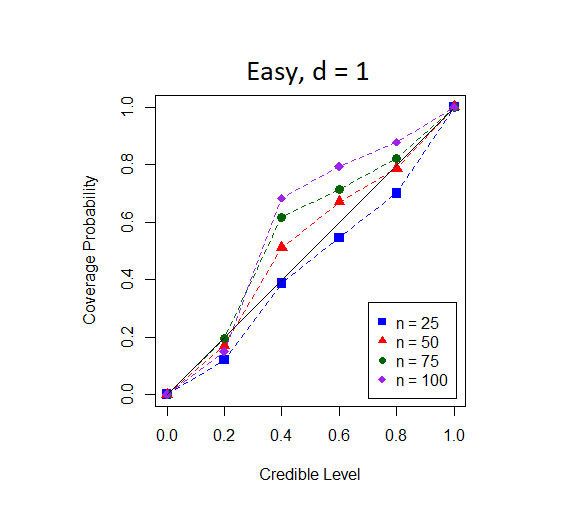}
\includegraphics[width = 0.4\textwidth,trim={2cm 1cm 3cm 1cm},clip]{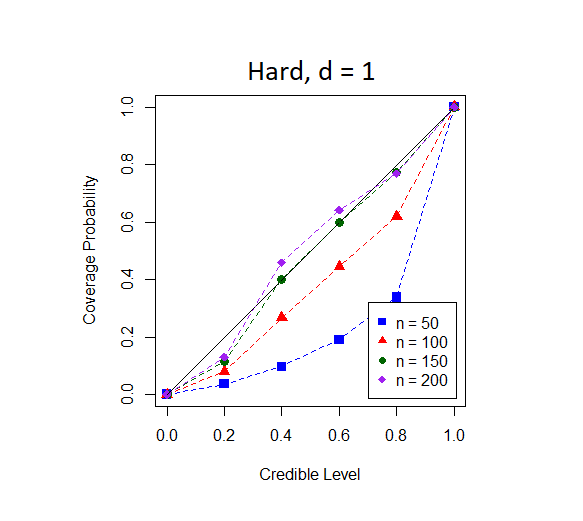}
\includegraphics[width = 0.4\textwidth,trim={2cm 1cm 3cm 1cm},clip]{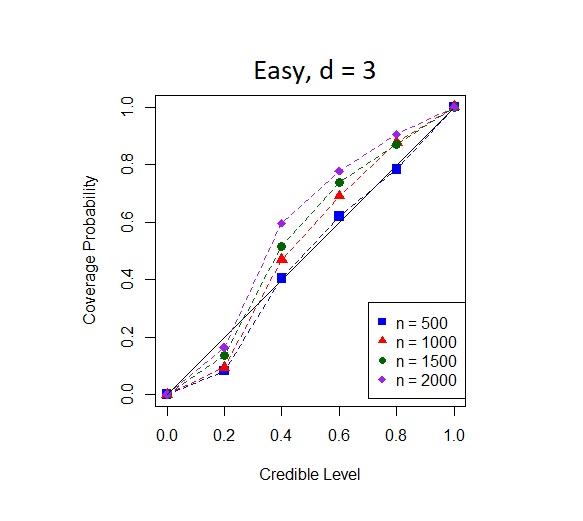}
\includegraphics[width = 0.4\textwidth,trim={2cm 1cm 3cm 1cm},clip]{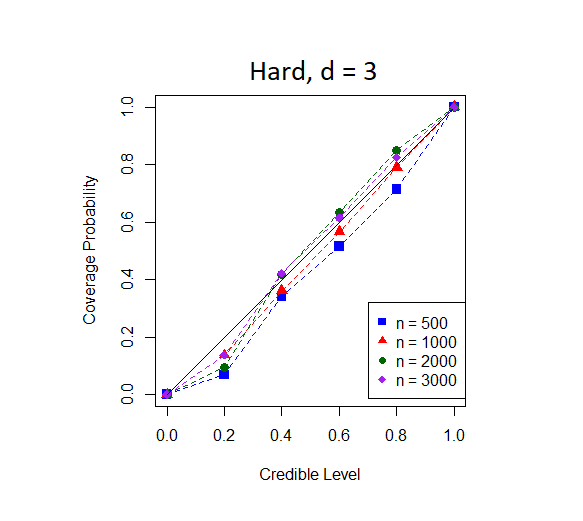}
\caption{\color{black} Evaluation of uncertainty quantification provided by BC.
Here we used empirical Bayes for $\bm{\sigma}$ with $\lambda$ marginalised in dimensions $d=1$ (top) and $d=3$ (bottom). 
Coverage frequencies (computed from $500$ (top) or $150$ (bottom) realisations) were compared against notional $100(1-\gamma)\%$ Bayesian credible regions for varying level $\gamma$ and number of observations $n$. 
The upper-left quadrant represents conservative credible intervals whilst the lower-right quadrant represents over-confident intervals.
\textit{Left:} ``Easy'' test function $f_1$.
\textit{Right:} ``Hard'' test function $f_2$.
}
\label{fig:Frequentist_coverage_BMC}
\end{figure}

Results are shown in Fig. \ref{fig:testfunc_EB_BMC}. Error-bars are used to denote the 95\% posterior credible regions for the value of the integral and we also display the values $\hat{\sigma}_i$ of the length scale $\sigma_i$ selected by EB\footnote{\textcolor{black}{the term ``credible'' is used loosely since the $\hat{\sigma}_i$ are estimated rather than marginalised.}}.
The $\hat{\sigma}_i$ appear to converge rapidly as $n \rightarrow \infty$; this is encouraging but we emphasise that we do not provide theoretical guarantees for EB in this work. 
On the negative side, over-confidence is possible at small values of $n$. 
Indeed, the BC posterior is liable to be over-confident under EB, since in the absence of evidence to the contrary, EB selects large values for $\sigma$ that correspond to more regular functions; this is most evident in the ``hard'' case.

Next we computed coverage frequencies for $100(1-\gamma)\%$ credible regions. For each sample size $n$, the process was repeated over many realisations of the states $\{\bm{x}_i\}_{i=1}^n$, shown in Fig. \ref{fig:Frequentist_coverage_BMC}.
It may be seen that (for $n$ large enough) the uncertainty quantification provided by EB is over-cautious for the easier function $f_1$, whilst being well-calibrated for the more complicated functions such as $f_2$. As expected, we observed that the coverage was over-confident for small values of $n$.
\textcolor{black}{Performance was subsequently investigated with $\lambda$ selected by EB.
In general this performed worse than when $\lambda$ was marginalised; results are contained in Supplement \ref{extra numerics appendix}. }

Finally, to understand whether theoretical results on asymptotic behaviour are realised in practice, we note (in the absence of EB) that the variance $\mathbb{V}_n[\Pi[g]]$ is independent of the integrand and may be plotted as a function of $n$.
Results in Supplement \ref{extra numerics appendix} demonstrate that theoretical rates are observed in practice for $d=1$ for BQMC; however, at large values of $d$, more data are required to achieve accurate estimation and increased numerical instability was observed.
 
The results on test functions provided in this section illustrate the extent to which uncertainty quantification in possible using BC. 
In particular, for our examples, we observed reasonable frequentist coverage if the number $n$ of samples was not too small.

For the remainder we explore possible roles for BMCMC and BQMC in statistical applications.
Four case studies, carefully chosen to highlight both the strengths and the weaknesses of BC are presented. 
Brief critiques of each study are contained below, the full details of which can be found in Supplement \ref{appendix:case_studies}.

%%%%%%%%%%%%%%%%%%%%%%%%%%%%%%%%%%%%%%%%%%%%%%%%%%%%%%%%%%%%%%%%%%%%%%%%%%%%%%%%%%%%%%%%%%%

\FloatBarrier

\subsection{Case Study \#1: Model Selection via Thermodynamic Integration}

Consider the problem of selecting a single best model among a set $\{\mathcal{M}_1 , \dots , \mathcal{M}_M\}$, based on data $\bm{y}$ assumed to arise from a true model in this set.
The Bayesian solution, assuming a uniform prior over models, is to select the MAP model. 
We focus on the case with uniform prior on models $p(\mathcal{M}_i)=1/M$, and this problem hence reduces to finding the largest marginal likelihood $p_i = p(\bm{y}|\mathcal{M}_i)$. 
The $p_i$ are usually intractable integrals over the parameters $\bm{\theta}_i$ associated with model $\mathcal{M}_i$.
One widely-used approach to model selection is to estimate each $p_i$ in turn, say by $\hat{p}_i$, then to take the maximum of the $\hat{p}_i$ over $i \in \{1,\dots,M\}$.
In particular, thermodynamic integration is one approach to approximation of marginal likelihoods $p_i$ for individual models \citep{Gelman1998,Friel2008}.

In many contemporary applications the MAP model is not well-identified, for example in variable selection where there are very many candidate models.
Then, the MAP becomes sensitive to numerical error in the $\hat{p}_i$, since an incorrect model $\mathcal{M}_i$, $i \neq k$ can be assigned an overly large value of $\hat{p}_i$ due to numerical error, in which case it could be selected in place of the true MAP model.
Below we explore the potential to exploit probabilistic integration to surmount this problem.

\subsubsection{Thermodynamic Integration}

To simplify notation below we consider computation of a single $p_i$ and suppress dependence on the index $i$ corresponding to model $\mathcal{M}_i$.
Denote the parameter space by $\Theta$.
For $t \in [0,1]$ (an \emph{inverse temperature}) define the \emph{power posterior} $\pi_t$, a distribution over $\Theta$ with density $\pi_t(\bm{\theta}) \propto p(\bm{y} | \bm{\theta})^t p(\bm{\theta})$.
The thermodynamic identity is formulated as a double integral:
\begin{equation*}
\log p(\bm{y}) = \int_0^1 \mathrm{d}t \int_{\Theta} \log p(\bm{y}|\bm{\theta}) \mathrm{d}\pi_t(\bm{\theta}).
\end{equation*}
The thermodynamic integral can be re-expressed as $\log p(\bm{y}) = \int_0^1 g(t) \mathrm{d}t$, $g(t) = \int_\Theta f(\bm{\theta}) \mathrm{d}\pi_t(\bm{\theta})$, where $f(\bm{\theta}) = \log p(\bm{y} | \bm{\theta})$.
Standard practice is to discretise the outer integral and estimate the inner integral using MCMC:
Letting $0 = t_1 < \dots < t_m = 1$ denote a fixed \emph{temperature schedule}, we thus have (e.g. using the trapezium rule)
\begin{eqnarray}
\log p(\bm{y}) \; \approx \; \sum_{i=2}^m (t_i - t_{i-1}) \frac{\hat{g}_i + \hat{g}_{i-1}}{2}, \; \; \; \hat{g}_i \; = \; \frac{1}{n} \sum_{j=1}^n \log p(\bm{y}|\bm{\theta}_{i,j}), \label{classic TI}
\end{eqnarray}
where $\{\bm{\theta}_{i,j}\}_{j=1}^n$ are MCMC samples from $\pi_{t_i}$.
Several improvements have been proposed, including the use of higher-order numerical quadrature for the outer integral \citep{Friel2014,Hug2015} and the use of control variates for the inner integral \citep{Oates2015a,Oates2016}.
To date, probabilistic integration has not been explored in this context.

\subsubsection{Probabilistic Thermodynamic Integration}

Our proposal is to apply BC to both the inner and outer integrals.
This is instructive, since nested integrals are prone to propagation and accumulation of numerical error.
Several features of the method are highlighted:

\vspace{5pt}
\noindent {\it Transfer Learning:}
In the probabilistic approach, the two integrands $f$ and $g$ are each assigned prior probability models.
For the inner integral we assign a prior $f \sim \mathcal{N}(0,k_f)$.
Our data here are the $nm \times 1$ vector $\bm{f}$ where $f_{(i-1)n+j} = f(\bm{\theta}_{i,j})$.
For estimating $g_i$ with BC we have $m$ times as much data as for the MC estimator $\hat{g}_i$, in Eqn. \ref{classic TI}, which makes use of only $n$ function evaluations.
Here, information transfer across temperatures is made possible by the explicit model for $f$ underpinning BC. 

In the posterior, $\bm{g} = [g(t_1),\dots,g(t_T)]$ is a Gaussian random vector with $\bm{g}|\bm{f} \sim \mathcal{N}(\bm{\mu},\bm{\Sigma})$ where the mean and covariance are defined, in the obvious notation, by
\begin{eqnarray*}
\mu_a  & = & \Pi_{t_a}[\bm{k}_f(\cdot,X)] \bm{K}_f^{-1}  \bm{f},  \\
\Sigma_{a,b} & = & \Pi_{t_a}\Pi_{t_b} [k_f(\cdot,\cdot)]] - \Pi_{t_a}[\bm{k}_f(\cdot,X)] \bm{K}_f^{-1} \Pi_{t_b}[\bm{k}_f(X,\cdot)],
\end{eqnarray*}
where $X = \{\bm{\theta}_{i,j}\}_{j=1}^n$ and $\bm{K}_f$ is an $nm \times nm$ kernel matrix defined by $k_f$.

\vspace{5pt}
\noindent {\it Inclusion of Prior Information:}
For the outer integral, it is known that discretisation error can be substantial; \cite{Friel2014} proposed a second-order correction to the trapezium rule to mitigate this bias, while \cite{Hug2015} pursued the use of Simpson's rule.
Attacking this problem from the probabilistic perspective, we do not want to place a stationary prior on $g(t)$, since it is known from extensive empirical work that $g(t)$ will vary more at smaller values of $t$.
Indeed the rule-of-thumb $t_i = (i/m)^5$ is commonly used \citep{Calderhead2009}.
We would like to encode this information into our prior.
To do this, we proceed with an importance sampling step
$\log p(\bm{y}) = \int_0^1 g(t) \mathrm{d}t = \int_0^1 h(t) \pi(t) \mathrm{d}t$.
The rule-of-thumb implies an importance distribution $\pi(t) \propto 1 / (\epsilon + 5 t^{4/5})$ for some small $\epsilon > 0$, which renders the function $h = g / \pi$ approximately stationary (made precise in Supplement \ref{TI importance}).
A stationary GP prior $h \sim \mathcal{N}(0,k_h)$ on the transformed integrand $h$ provides the encoding of this prior knowledge that was used.

\vspace{5pt}
\noindent {\it Propagation of Uncertainty:}
Under this construction, in the posterior $\log p(\bm{y})$ is Gaussian with mean and covariance defined as
\begin{eqnarray*}
\mathbb{E}_n[\log p(\bm{y})] & = & \Pi[\bm{k}_h(\cdot,T)] \bm{K}_h^{-1}  \bm{\mu} \\
\mathbb{V}_n[\log p(\bm{y})] & = & \underbrace{\Pi\Pi [k_h(\cdot,\cdot)]] - \Pi[\bm{k}_h(\cdot,T)] \bm{K}_h^{-1} \Pi[\bm{k}_h(T,\cdot)]}_{(*)} \\
& & \hspace{50pt} + \underbrace{\Pi[\bm{k}_h(\cdot,T)] \bm{K}_h^{-1} \bm{\Sigma} \bm{K}_h^{-1}  \Pi[\bm{k}_h(T,\cdot)]}_{(**)},
\end{eqnarray*}
where $T = \{t_i\}_{i=1}^m$ and $\bm{K}_h$ is an $m \times m$ kernel matrix defined by $k_h$.
The term $(*)$ arises from BC on the outer integral, while the term $(**)$ arises from propagating numerical uncertainty from the inner integral through to the outer integral.

\begin{figure}[t!]
\centering
\includegraphics[clip,trim = 0cm 1cm 0cm 0.4cm,width = \textwidth]{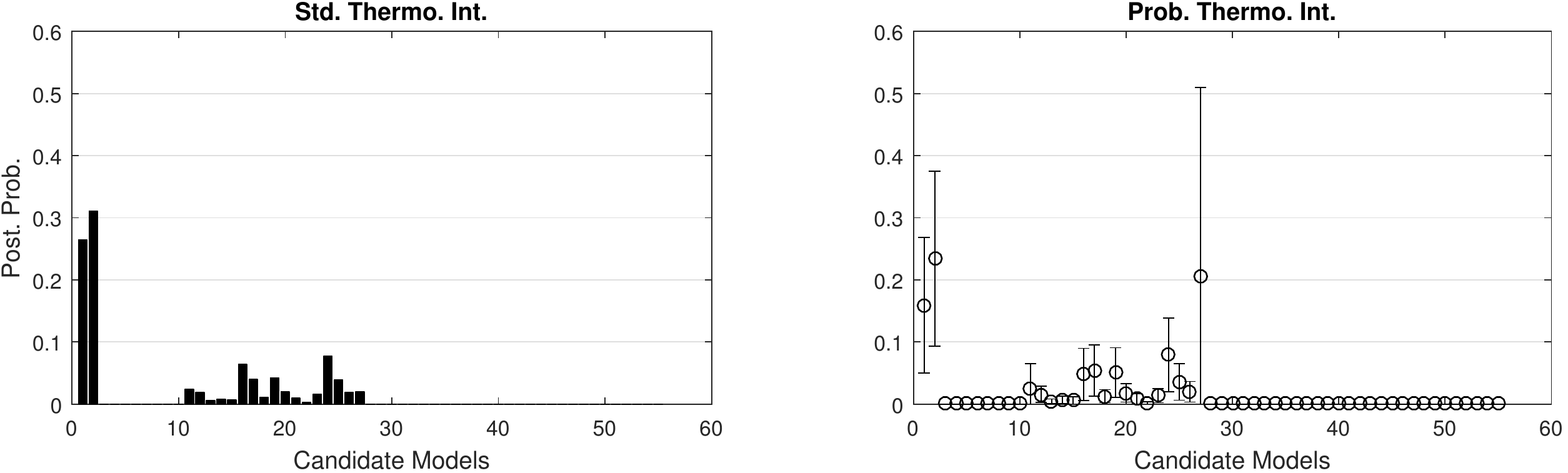} \vspace{1pt} \\
\includegraphics[clip,trim = 0cm 0cm 0cm 0.4cm,width = \textwidth]{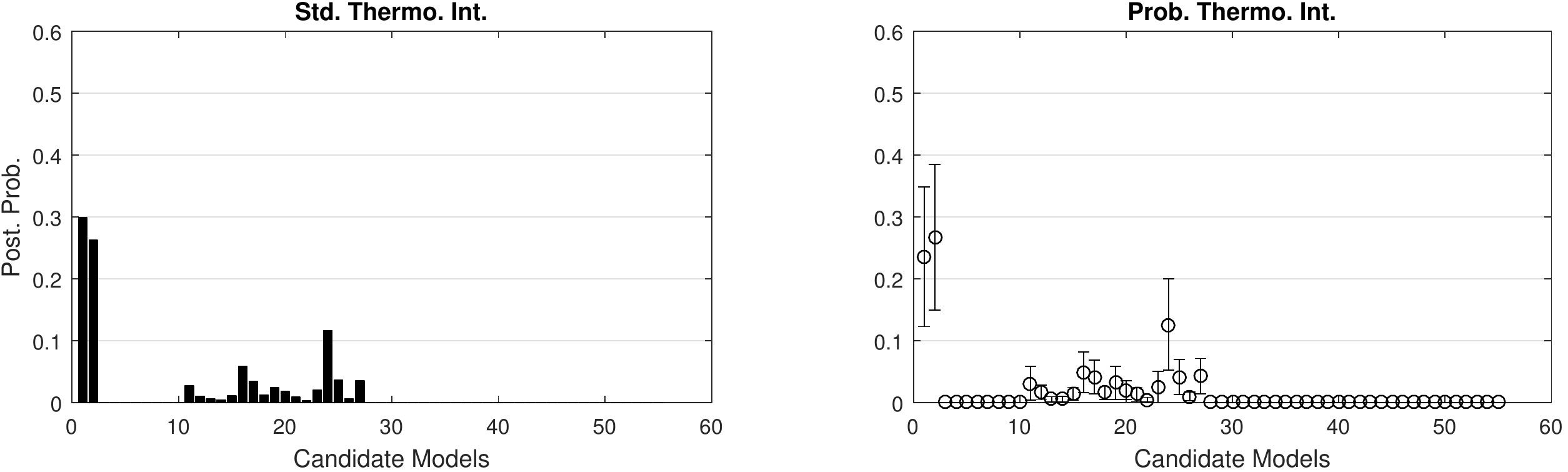}
\caption{Probabilistic thermodynamic integration; illustration on variable selection for logistic regression (the true model was $\mathcal{M}_1$).
Standard and probabilistic thermodynamic integration were used to approximate marginal likelihoods and, hence, the posterior over models.
Each row represents an independent realisation of MCMC, while the data $\bm{y}$ were fixed.
\textit{Left:} Standard Monte Carlo, where point estimates for marginal likelihood were assumed to have no associated numerical error.
\textit{Right:} Probabilistic integration, where a model for numerical error on each integral was propagated through into the posterior over models. 
The probabilistic approach produces a ``probability distribution over a probability distribution'', where the numerical uncertainty is modelled on top of the usual uncertainty associated with model selection.
}
\label{thermo results}
\end{figure}

\subsubsection{Simulation Study}

An experiment was conducted to elicit the MAP model from a collection of 56 candidate logistic regression models in a variable selection setting.
This could be achieved in many ways; our aim was not to compare accuracy of point estimates, but rather to explore the probability model that, unlike in standard methods, is provided by BC.
Full details are in Supplement \ref{TI importance}.

Results are shown in Fig. \ref{thermo results}.
Here we compared approximations to the model posterior obtained using the standard method versus the probabilistic method, over two realisations of the MCMC (the data $\bm{y}$ were fixed).
We make some observations:
(i) The probabilistic approach models numerical uncertainty on top of the usual statistical uncertainty.
(ii) The computation associated with BC required less time, in total, than the time taken afforded to MCMC.
(iii) The same model was not always selected as the MAP when numerical error was ignored and depended on the MCMC random seed. 
In contrast, under the probabilistic approach, either $\mathcal{M}_1$ or $\mathcal{M}_2$ could feasibly be the MAP under any of the MCMC realisations, up to numerical uncertainty.
(iv) The top row of Fig. \ref{thermo results} shows a large posterior uncertainty over the marginal likelihood for $\mathcal{M}_{27}$. This could be used as an indicator that more computational effort should be expended on this particular integral.
(v) The posterior variance was dominated by uncertainty due to discretisation error in the outer integral, rather than the inner integral.
This suggests that numerical uncertainty could be reduced by allocating more computational resources to the outer integral rather than the inner integral.

\FloatBarrier

%%%%%%%%%%%%%%%%%%%%%%%%%%%%%%%%%%%%%%%%%%%%%%%%%%%%%%%%%%%%%%%%%%%%%%%%%%%%%%%%%%%%%%%%%%%%%%%%%%%%%%%%%%%%%%%%%%%%%%%%%%%%%%%%%%%%

\subsection{Case Study \#2: Uncertainty Quantification for Computer Experiments} \label{computer section}

Here we consider an industrial scale computer model for the Teal South oil field, New Orleans \citep{Hajizadeh2011}. 
Conditional on field data, posterior inference was facilitated using state-of-the-art MCMC \citep{Lan2016}.
Oil reservoir models are generally challenging for MCMC: 
First, simulating from those models can be time-consuming, making the cost of individual MCMC samples a few minutes to several hours. 
Second, the posterior distribution will often exhibit strongly non-linear concentration of measure.
Here we computed statistics of interest using BMCMC, where the uncertainty quantification afforded by BC aims to enable valid inferences in the presence of relatively few MCMC samples.
Full details are provided in Supplement \ref{teal appendix}.

Quantification of the uncertainty associated with predictions is a major topic of ongoing research in this field \citep{Mohamed2010,Hajizadeh2011,Park2013} due to the economic consequences associated with inaccurate predictions of quantities such as future oil production rate. 
A probabilistic model for numerical error in integrals associated with prediction could provide a more complete uncertainty assessment.

\begin{figure}[t!]
\centering
\begin{minipage}{\textwidth}
\includegraphics[width = 0.32\textwidth]{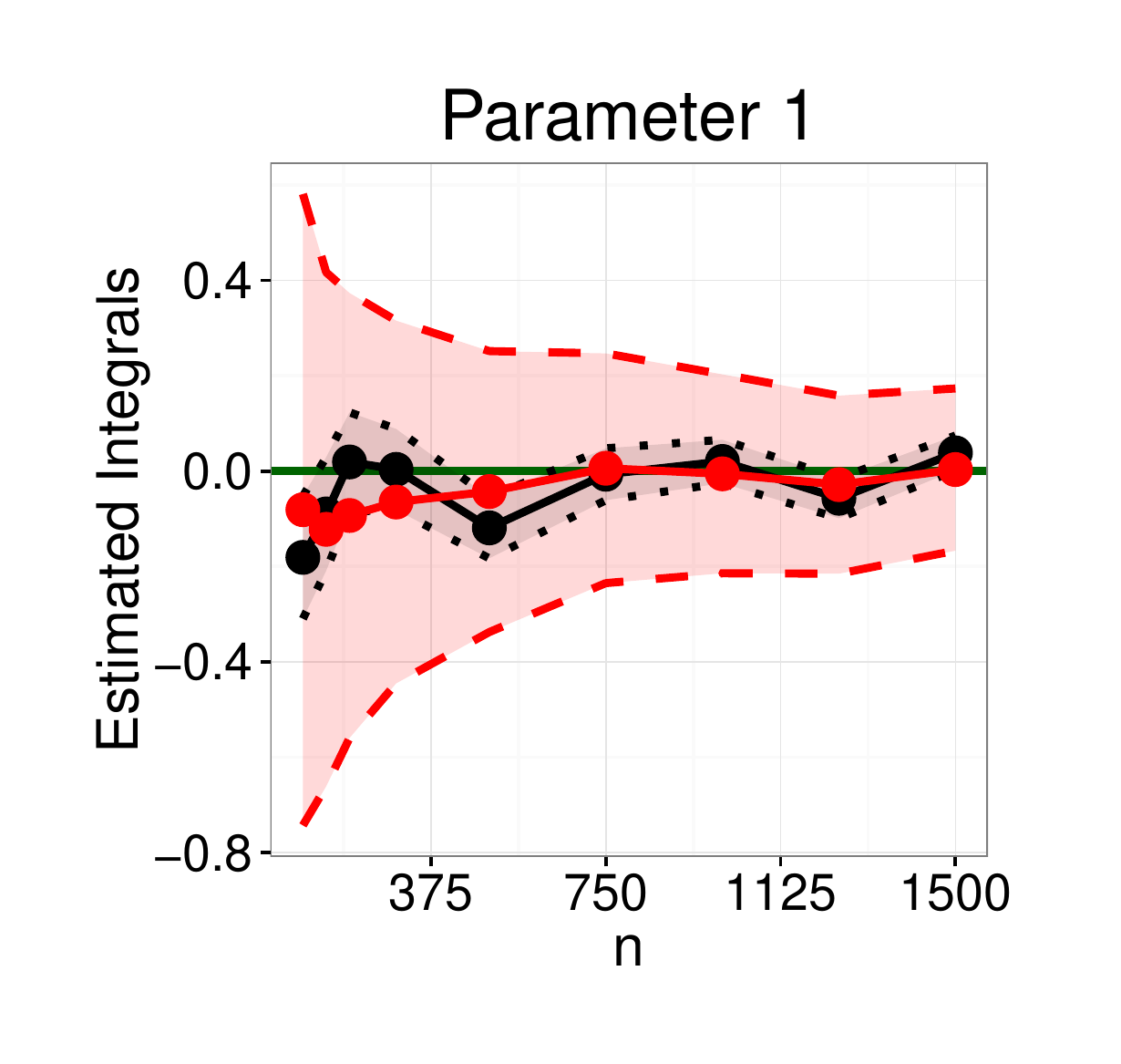}
\includegraphics[width = 0.32\textwidth]{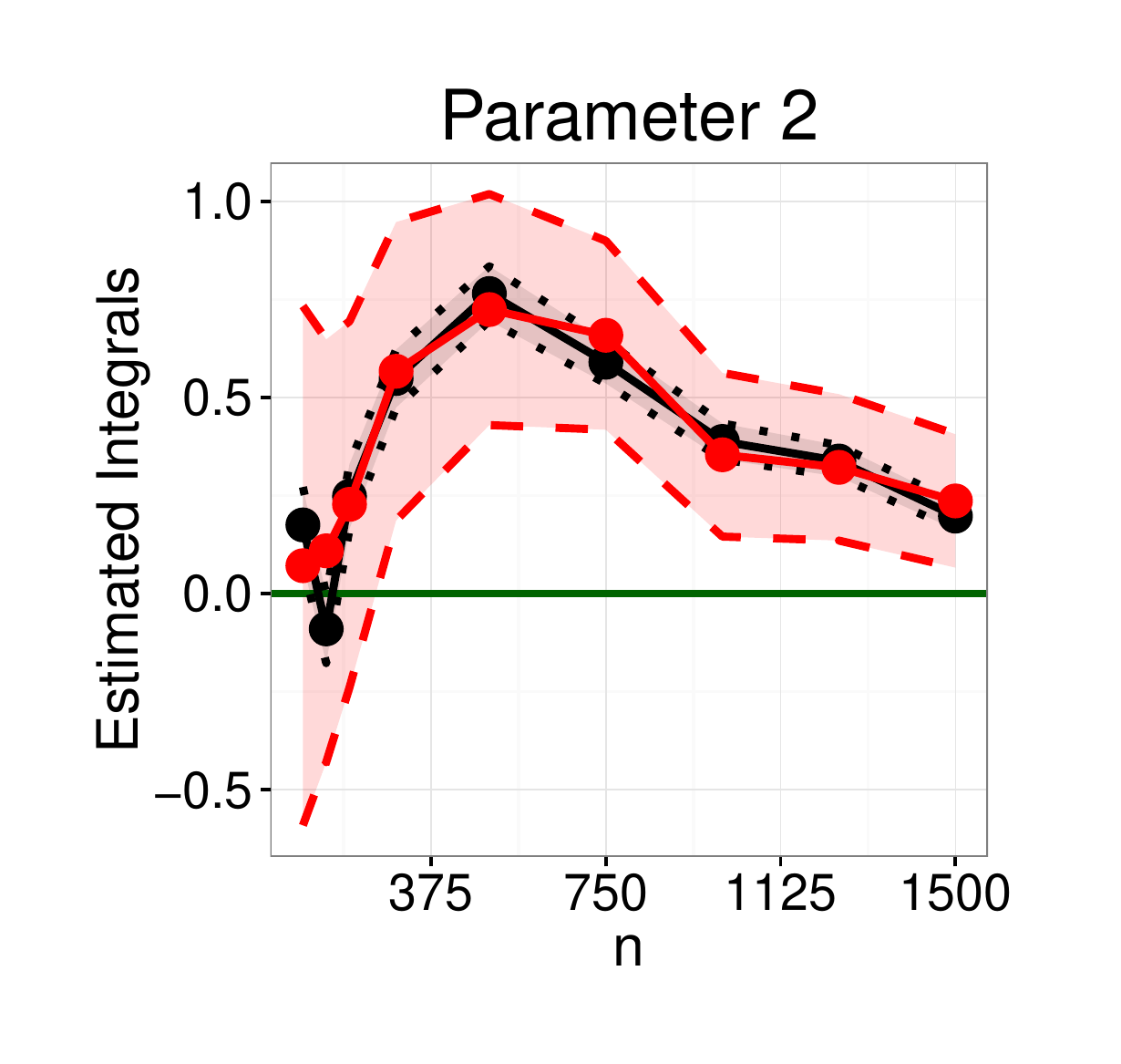}
\includegraphics[width = 0.32\textwidth]{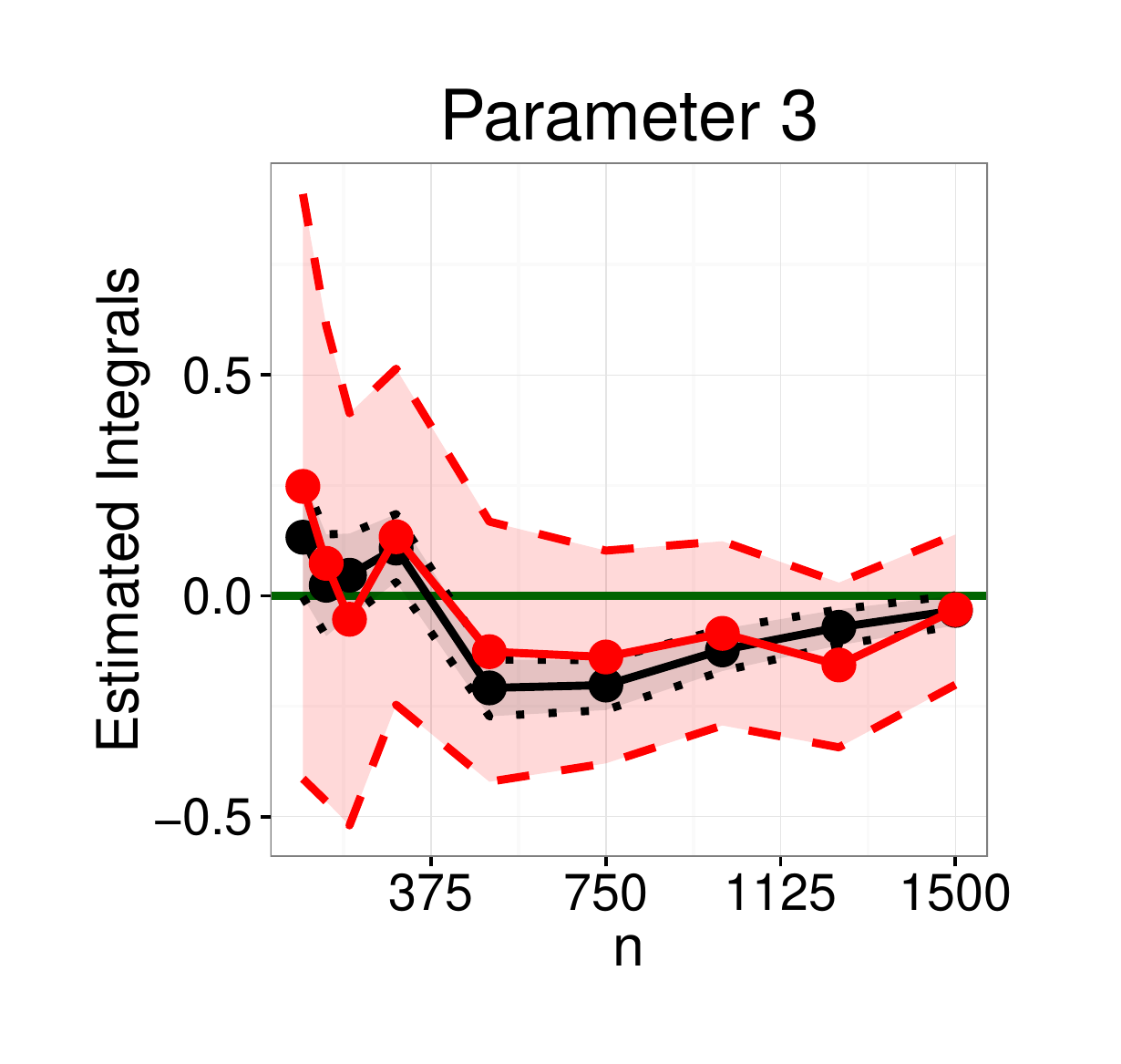}
\vspace{-3mm}
\end{minipage}
\begin{minipage}{\textwidth}
\includegraphics[width = 0.32\textwidth]{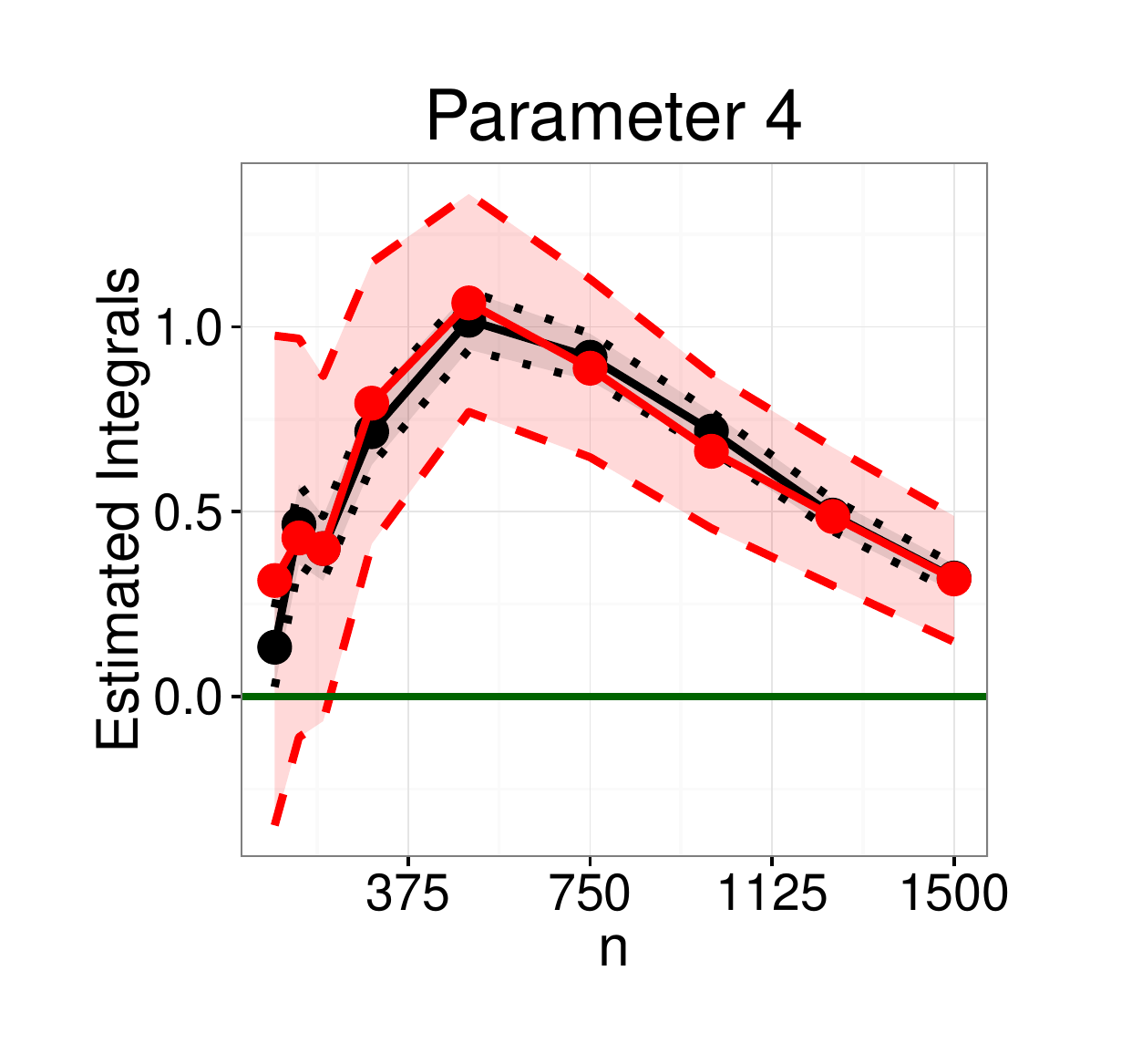}
\includegraphics[width = 0.32\textwidth]{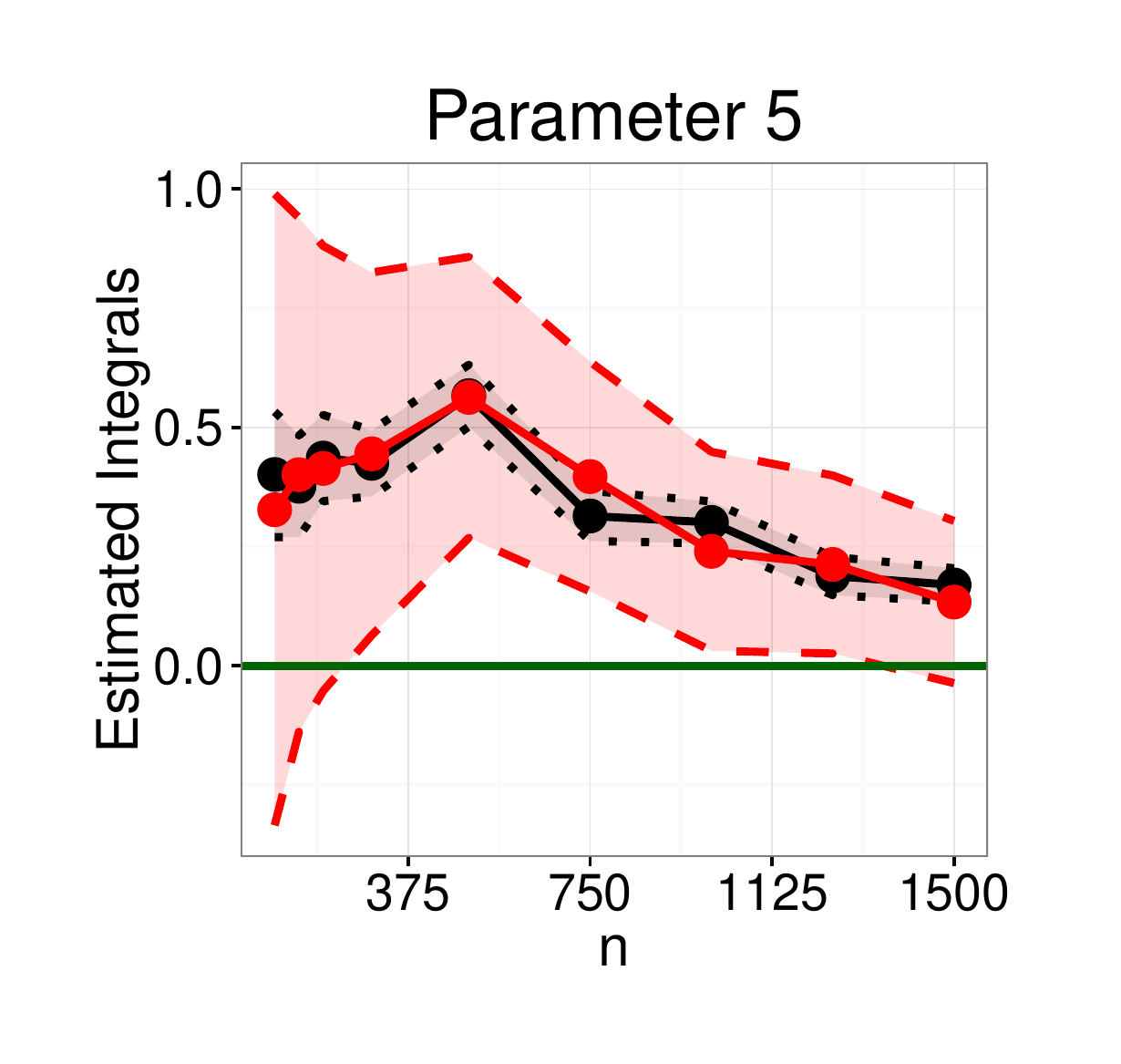}
\includegraphics[width = 0.32\textwidth]{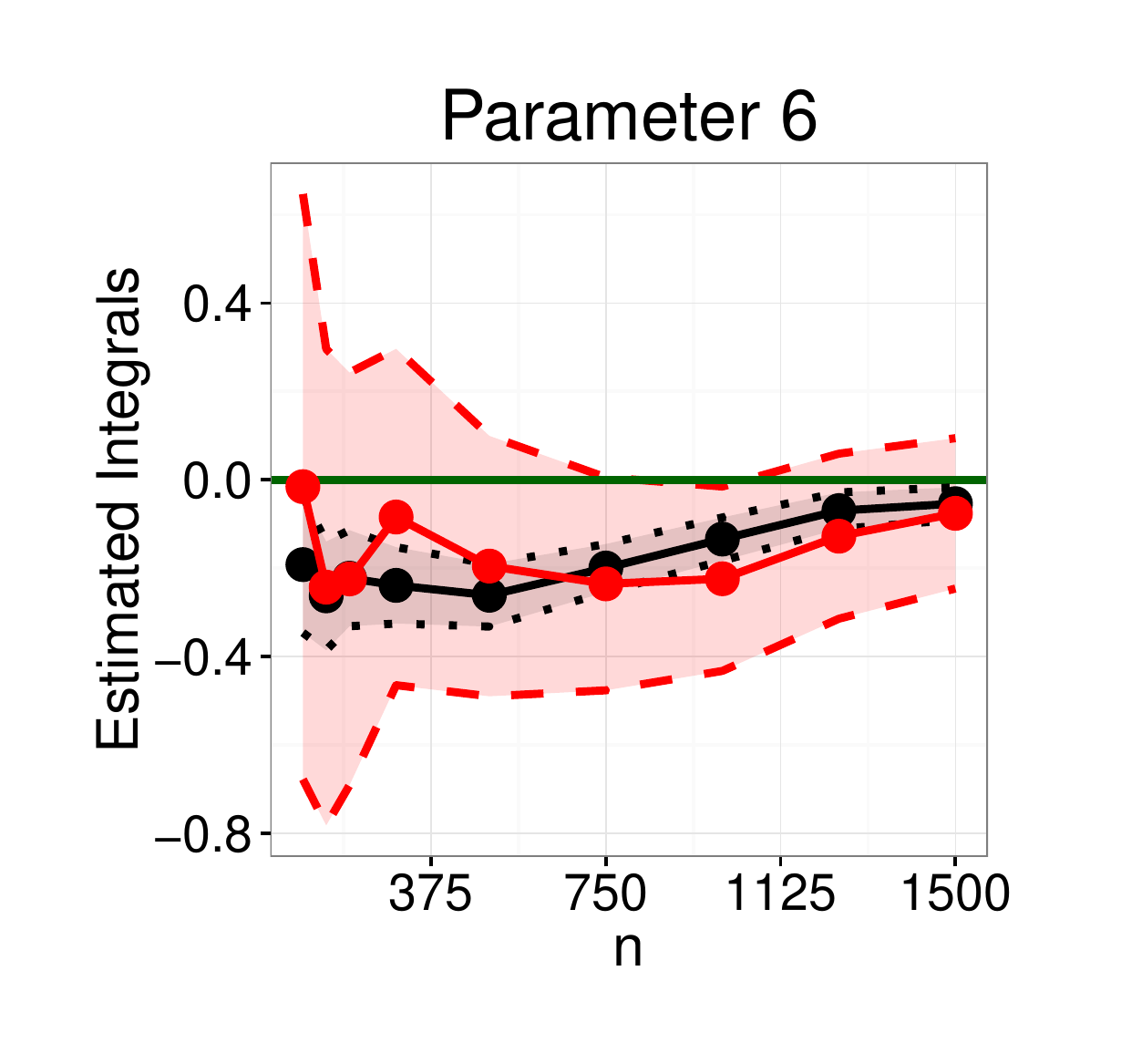}
\vspace{-3mm}
\end{minipage}
\begin{minipage}{\textwidth}
\includegraphics[width = 0.32\textwidth]{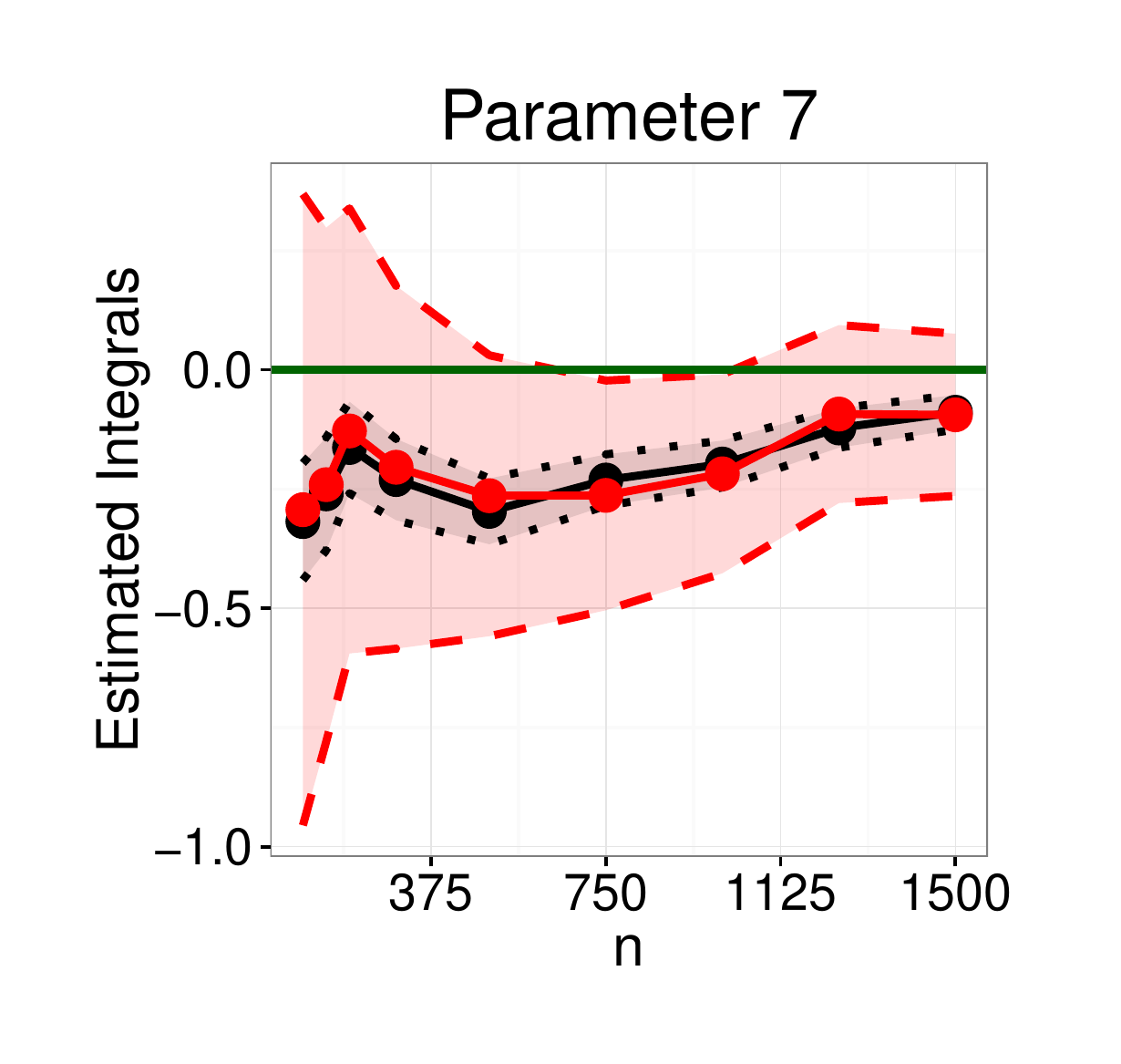}
\includegraphics[width = 0.32\textwidth]{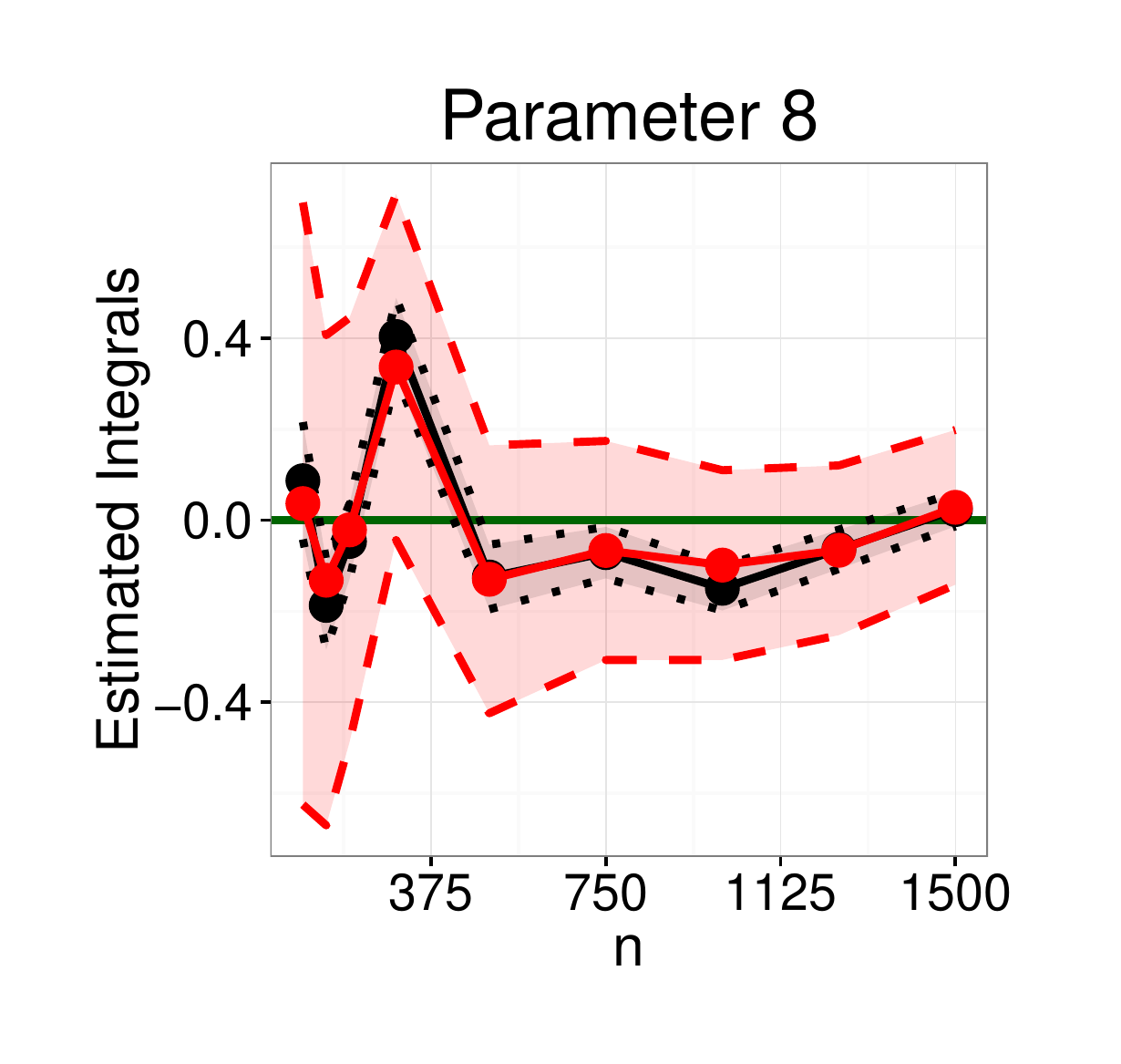}
\includegraphics[width = 0.32\textwidth]{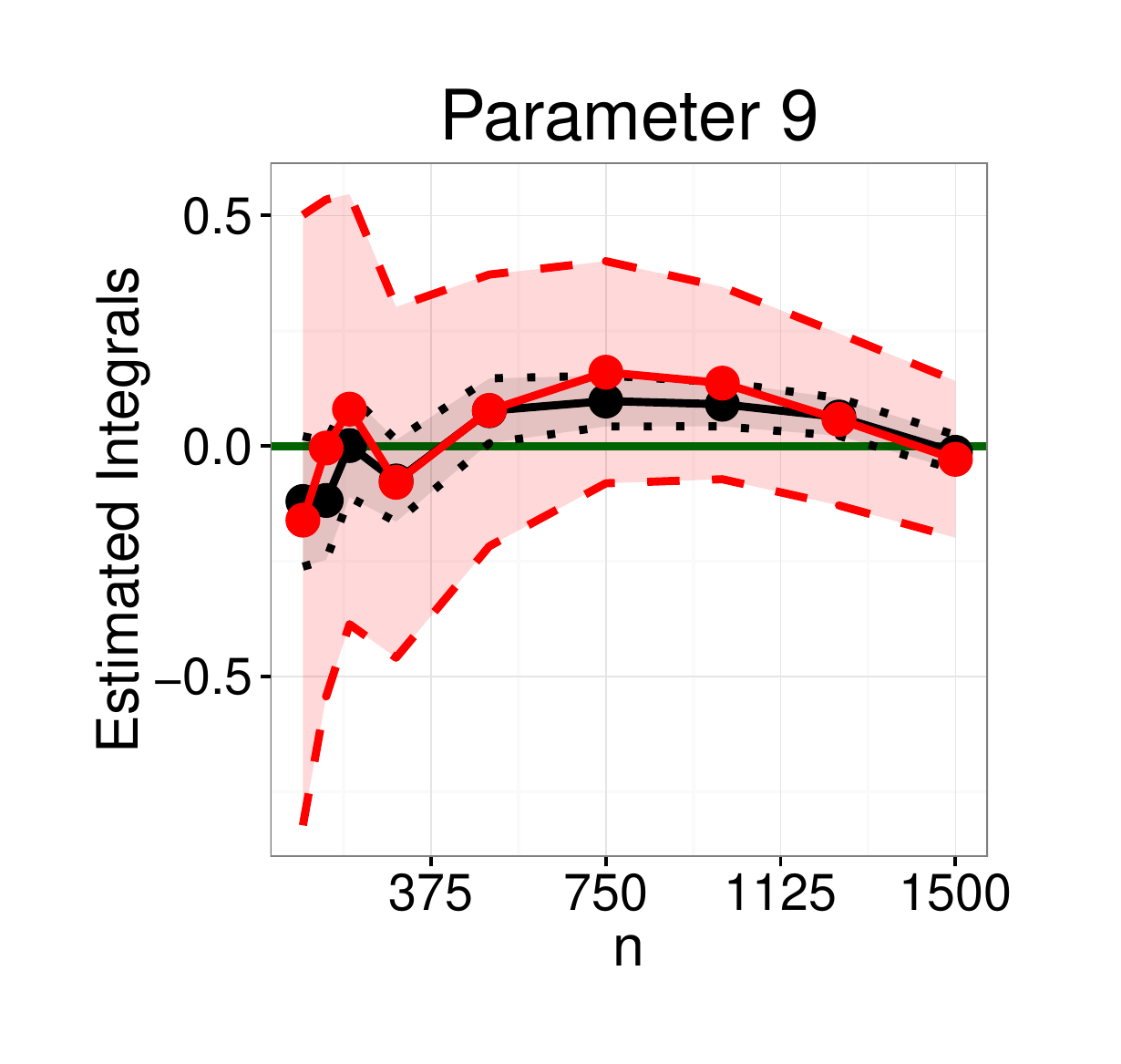}
\end{minipage}
\caption{Numerical estimation of parameter posterior means for the Teal South oil field model (centered around the true values). The green line gives the exact value of the integral. The MCMC (black line) and BMCMC point estimates (red line) provided similar performance.  The MCMC 95\% confidence intervals, based on estimated asymptotic variance (black dotted lines), are poorly calibrated whereas with the BMCMC 95\% credible intervals (red dotted lines) provide a more honest uncertainty assessment.}
\label{fig:Integrals_TealSouth}
\end{figure}

The particular integrals that we considered are posterior means for each model parameter, and we compared against an empirical benchmark obtained with brute force MCMC. 
BMCMC was employed with a Mat\'{e}rn $\alpha=3/2$ kernel whose lengthscale-parameter was selected using EB. 
Estimates for posterior means were obtained using both standard MCMC and BMCMC, shown in Fig. \ref{fig:Integrals_TealSouth}. 
For this example the posterior distribution provides sensible uncertainty quantification for integrals 1, 3, 6-9, but was over-confident for integrals 2, 4, 5. 
The point accuracy of the BMCMC estimator matched that of the standard MCMC estimator.
The lack of faster convergence for BMCMC appears to be due to inaccurate estimation of the kernel mean and we conjecture that alternative exact approaches, such as \cite{Oates2015a}, may provide improved performance in this context.
However, standard confidence intervals obtained from the CLT for MCMC with a plug-in estimate for the asymptotic variance were over-confident for parameters 2-9.

%\FloatBarrier

%%%%%%%%%%%%%%%%%%%%%%%%%%%%%%%%%%%%%%%%%%%%%%%%%%%%%%%%%%%%%%%%%%%%%%%%%%%%%%%%%%%%%%%%

\subsection{Case Study \#3: High-Dimensional Random Effects} \label{weighted Sob app}

Our aim here was to explore whether more flexible representations afforded by weighted combinations of Hilbert spaces enable probabilistic integration when $\mathcal{X}$ is high-dimensional. 
The focus was BQMC, but the methodology could be applied to other probabilistic integrators.

\subsubsection{Weighted Spaces}

The formulation of high (and infinite) -dimensional QMC can be achieved with a construction known as a \emph{weighted} Hilbert space.
These spaces, defined below, are motivated by the observation that many integrands encountered in applications seem to vary more in lower dimensional projections compared to higher dimensional projections.
Our presentation below follows Sec. 2.5.4 and 12.2 of \cite{Dick2010}, but the idea goes back at least to \citet[][Chap. 10]{Wahba1990}.

As usual with QMC, we work in $\mathcal{X} = [0,1]^d$ and $\pi$ uniform over $\mathcal{X}$.
Let $\mathcal{I} = \{1,2,\dots,d\}$.
For each subset $u \subseteq \mathcal{I}$, define a weight $\gamma_u \in (0,\infty)$ and denote the collection of all weights by $\bm{\gamma} = \{\gamma_u\}_{u \subseteq \mathcal{I}}$.
Consider the space $\mathcal{H}_{\bm{\gamma}}$ of functions of the form $f(\bm{x}) = \sum_{u \subseteq \mathcal{I}} f_u(\bm{x}_u)$, where $f_u$ belongs to an RKHS $\mathcal{H}_u$ with kernel $k_u$ and $\bm{x}_u$ denotes the components of $\bm{x}$ that are indexed by $u \subseteq \mathcal{I}$.
This is not restrictive, since any function can be written in this form by considering only $u=\mathcal{I}$.
We turn $\mathcal{H}_{\bm{\gamma}}$ into a Hilbert space by defining an inner product $\langle f , g \rangle_{\bm{\gamma}} := \sum_{u \subseteq \mathcal{I}} \gamma_u^{-1} \langle f_u , g_u \rangle_u
$
where $\bm{\gamma} = \{\gamma_u : u \subseteq \mathcal{I}\}$.
Constructed in this way, $\mathcal{H}_{\bm{\gamma}}$ is an RKHS with kernel $
k_{\bm{\gamma}}(\bm{x},\bm{x}') = \sum_{u \subseteq \mathcal{I}} \gamma_u k_u(\bm{x},\bm{x}')$.
Intuitively, the weights $\gamma_u$ can be taken to be small whenever the function $f$ does not depend heavily on the $|u|$-way interaction of the states $\bm{x}_u$.
Thus, most of the $\gamma_u$ will be small for a function $f$ that is effectively low-dimensional.
A measure of the effective dimension of the function is given by $\sum_{u \subseteq \mathcal{I}} \gamma_u$; in an extreme case $d$ could even be infinite provided that this sum remains bounded \citep{Dick2013}.

The (canonical) \emph{weighted} Sobolev space of dominating mixed smoothness $\mathcal{S}_{\alpha,\bm{\gamma}}$ is defined by taking each of the component spaces to be $\mathcal{S}_\alpha$. 
In finite dimensions $d < \infty$, BQMC rules based on a higher-order digital net attain optimal WCE rates $O(n^{-\alpha+\epsilon})$ for this RKHS; see Supplement \ref{logistic appendix} for full details.

\subsubsection{Semi-Parametric Random Effects Regression}

For illustration we considered generalised linear models, and focus on a Poisson semi-parametric random effects regression model studied by \citet[][Example 2]{Kuo2008}.
The context is inference for the parameters $\bm{\beta}$ of the following model
\begin{eqnarray*}
Y_j | \lambda_j & \sim & \text{Po}(\lambda_j) \nonumber \\
\log(\lambda_j) & = & \beta_0 + \beta_1 z_{1,j} + \beta_2 z_{2,j} + u_1 \phi_1(z_{2,j}) + \dots + u_d \phi_d(z_{2,j}) \\
u_j & \sim & N(0,\tau^{-1}) \text{ independent} \nonumber.
\end{eqnarray*}
Here $z_{1,j} \in \{0,1\}$, $z_{2,j} \in (0,1)$ and $\phi_j(z) = [z - \kappa_j]_+$ where $\kappa_j \in (0,1)$ are pre-determined knots.
We took $d = 50$ equally spaced knots in $[\min \bm{z}_2,\max \bm{z}_2]$.
Inference for $\bm{\beta}$ requires multiple evaluations of the observed data likelihood 
$
p(\bm{y}|\bm{\beta}) = \int_{\mathbb{R}^d} p(\bm{y} | \bm{\beta} , \bm{u}) p(\bm{u}) \mathrm{d}\bm{u}
$
and therefore is a candidate for probabilistic integration methods, in order to model the cumulative uncertainty of estimating multiple numerical integrals.

\begin{figure}[t]
\centering
\makebox{\includegraphics[width = 0.7\textwidth]{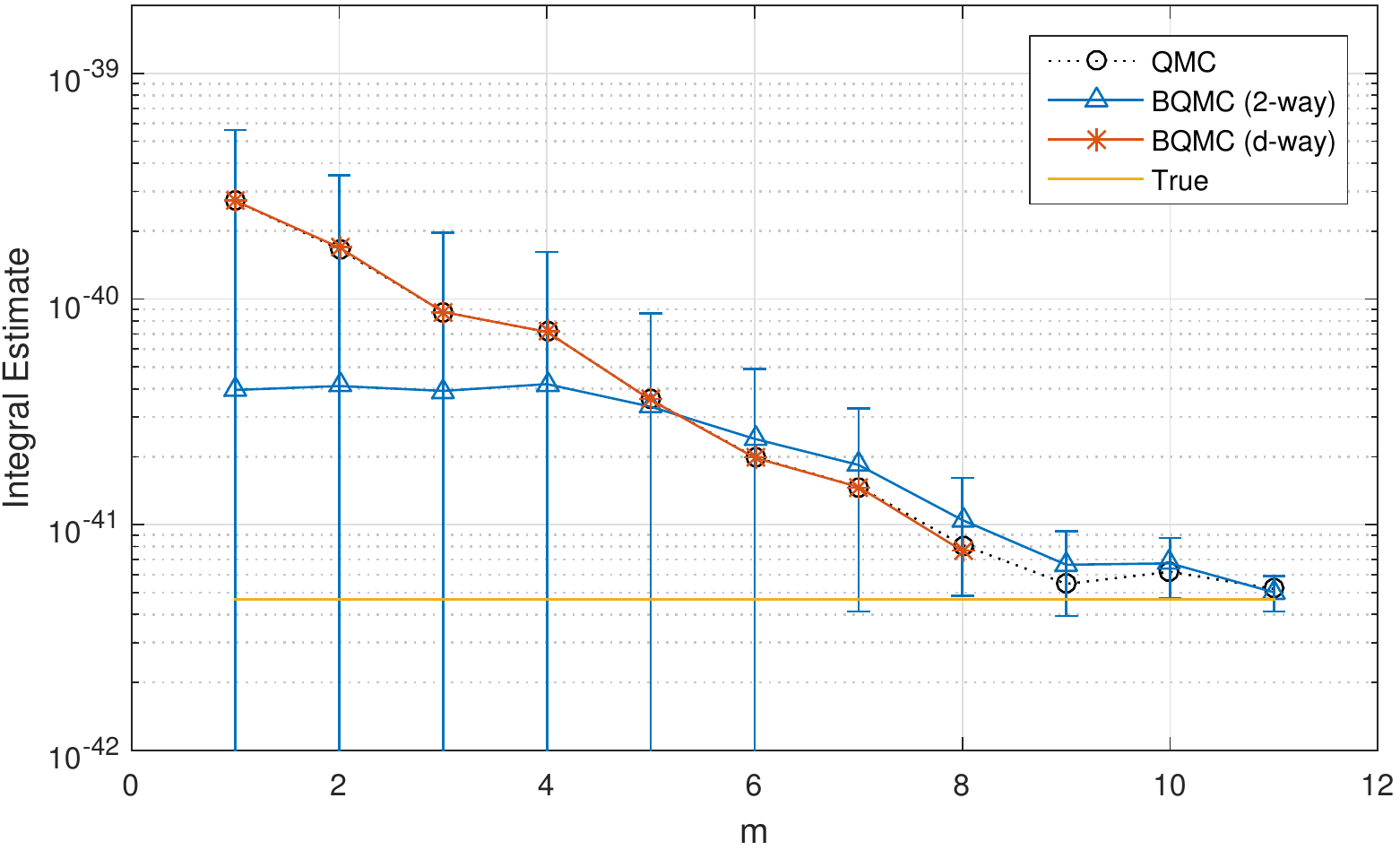}}
\caption{Application to semi-parametric random effects regression in $d = 50$ dimensions, based on $n = 2^m$ samples from a higher-order digital net.
[Error bars show $95\%$ credible regions.
To improve visibility results are shown on the log-scale; error bars are symmetric on the linear scale.
A brute-force QMC estimate was used to approximate the true value of the integral $p(\bm{y}|\bm{\beta})$ where $\bm{\beta} = (0,1,1)$ was the data-generating value of the parameter.]
}
\label{high-dim fig}
\end{figure}

In order to transform this integration problem to the unit cube we perform the change of variables $x_j = \Phi^{-1}(u_j)$ so that we wish to evaluate
$p(\bm{y}|\bm{\beta}) = \int_{[0,1]^d} p(\bm{y} | \bm{\beta} , \Phi^{-1}(\bm{x})) \mathrm{d}\bm{x}$.
Here $\Phi^{-1}(\bm{x})$ denotes the standard Gaussian inverse CDF applied to each component of $\bm{x}$. Probabilistic integration here proceeds under the hypothesis that the integrand  $f(\bm{x}) = p(\bm{y} | \bm{\beta} , \Phi^{-1}(\bm{x}))$ belongs to (or at least can be well approximated by functions in) $\mathcal{S}_{\alpha,\bm{\gamma}}$ for some smoothness parameter $\alpha$ and some weights $\bm{\gamma}$.
Intuitively, the integrand $f(\bm{x})$ is such that an increase in the value of $x_j$ at the knot $\kappa_j$ can be compensated for by a decrease in the value of $x_{j+1}$ at a neighbouring knot $\kappa_{j+1}$, but not by changing values of $\bm{x}$ at more remote knots.
Therefore we expect $f(\bm{x})$ to exhibit strong individual and pairwise dependence on the $x_j$, but expect higher-order dependency to be weaker.
This motivates the weighted space assumption.
\cite{Sinescu2012} provides theoretical analysis for the choice of weights $\bm{\gamma}$. 
Here, weights $\bm{\gamma}$ of \emph{order two} were used; $\gamma_u = 1$ for $|u| \leq d_{\max}$, $d_{\max} = 2$, $\gamma_u = 0$ otherwise, which corresponds to an assumption of low-order interaction terms (though $f$ can still depend on all $d$ of its arguments). 
Full details are provided in Supplement \ref{logistic appendix}.

Results in Fig. \ref{high-dim fig} showed that the $95\%$ posterior credible regions more-or-less cover the truth for this problem, suggesting that the uncertainty estimates are appropriate.
On the negative side, the BQMC method does not encode non-negativity of the integrand and, consequently, some posterior mass is placed on negative values for the integral, which is not meaningful.
To understand the effect of the weighted space construction here, we compared against the BQMC point estimate with $d$-way interactions ($u \in \{\emptyset,\mathcal{I}\}$).
An interesting observation was that these point estimates closely followed those produced by QMC.

%%%%%%%%%%%%%%%%%%%%%%%%%%%%%%%%%%%%%%%%%%%%%%%%%%%%%%%%%%%%%%%%%%%%%%%%%%%%%%%%%%%%%%%%%%%%%%%%%%%%%%%%%%%%%%%%%%%%%%%%%%%%%%%%%%%%%%%%%%%%%%%%%%%%%%%%%%%%%%%%%%%%%%%%%%%%%%%%%%%%%%%%%%%%%%%%%%%%

\subsection{Case Study \#4: Spherical Integration for Computer Graphics} \label{sec:sphere_application}

Probabilistic integration methods can be defined on arbitrary manifolds, with formulations on non-Euclidean spaces suggested as far back as \cite{Diaconis1988} and recently exploited in the context of computer graphics \citep{Brouillat2009,Marques2015}.
This forms the setting for our final case study.

\subsubsection{Global Illumination Integrals}

Below we analyse BQMC on the $d$-sphere $\mathbb{S}^d = \{\bm{x} = (x_1,\dots,x_{d+1}) \in \mathbb{R}^{d+1} : \|\bm{x}\|_2 = 1\} $
in order to estimate integrals of the form $\Pi[f] = \int_{\mathbb{S}^d} f \mathrm{d}\pi$, where $\pi$ is the spherical measure (i.e. uniform over $\mathbb{S}^d$ with $\int_{\mathbb{S}^d} \mathrm{d}\pi = 1$). 

%\begin{figure}[t]
%\centering
%\includegraphics[width = 0.9\textwidth,trim={0 0 0 6cm},clip]{figures/illumination_cartoon_v2.pdf}
%\caption{Application to illumination integrals in computer graphics.
%The California lake environment map, shown, was used in our experiment.}
%\label{illumination cartoon}
%\end{figure}

Probabilistic integration is applied to compute global illumination integrals used in the rendering of surfaces \citep{Pharr2004}, and we therefore focus on the case where $d=2$. 
Uncertainty quantification is motivated by inverse global illumination \citep[e.g.][]{Yu1999}, where the task is to make inferences from noisy observation of an object via computer-based image synthesis; a measure of numerical uncertainty could naturally be propagated in this context.
Below, to limit scope, we restrict attention to uncertainty quantification in the forward problem.

The models involved in global illumination are based on three main factors: a geometric model for the objects present in the scene, a model for the reflectivity of the surface of each object and a description of the light sources provided by an \emph{environment map}. The light emitted from the environment will interact with objects in the scene through reflection. 
This can be formulated as an illumination integral:
\begin{equation*}
L_o(\bm{\omega}_o) = L_e(\bm{\omega}_o) + \int_{\mathbb{S}^2} L_i(\bm{\omega}_i) \rho(\bm{\omega}_i,\bm{\omega}_o) [\bm{\omega}_i \cdot \bm{n}]_+ \mathrm{d}\pi(\bm{\omega}_i). \label{illumination integral}
\end{equation*}
Here $L_o(\bm{\omega}_o)$ is the \textit{outgoing radiance}, i.e. the outgoing light in the direction $\bm{\omega}_o$.
$L_e(\bm{\omega}_o)$ represents the amount of light emitted by the object itself (which we will assume to be known) and $L_i(\bm{\omega}_i)$ is the light hitting the object from direction $\bm{\omega}_i$. 
The term $\rho(\bm{\omega}_i,\bm{\omega}_o)$ is the \emph{bidirectional reflectance distribution} function (BRDF), which models the fraction of light arriving at the surface point from direction $\bm{\omega}_i$ and being reflected towards direction $\bm{\omega}_o$. 
Here $\bm{n}$ is a unit vector normal to the surface of the object.
Our investigation is motivated by strong empirical results for BQMC in this context obtained by \cite{Marques2015}.

To assess the performance of BQMC we consider a typical illumination integration problem based on a California lake environment.
The goal here is to compute intensities for each of the three RGB colour channels corresponding to observing a virtual object from a fixed direction $\bm{\omega}_o$.
We consider the case of an object directly facing the camera ($\bm{w}_o = \bm{n}$).
For the BRDF we took $\rho(\bm{\omega}_i,\bm{\omega}_o)= (2\pi)^{-1} \exp(\bm{\omega}_i \cdot \bm{\omega}_o -1)$. 
The integrand $f(\bm{\omega}_i) = L_i(\bm{\omega}_i) \rho(\bm{\omega}_i,\bm{\omega}_o) [\bm{\omega}_i \cdot \bm{\omega}_o]_+$ was modelled in a Sobolev space of low smoothness.
%In contrast, \cite{Marques2015} viewed Eqn. \ref{illumination integral} as an integral with respect to $\pi(\bm{\omega}_i) \propto \rho(\bm{\omega}_i,\bm{\omega}_o)$ and posited a space of smooth integrands restricted to the hemisphere.
%The approach that we propose has two possible advantages; (i) it provides a closed-form expression for the kernel mean, (ii) a rougher kernel may be more appropriate in the context of illumination integrals, as pointed out by \cite{Brouillat2009}.
The specific function space that we consider is the Sobolev space $\mathcal{H}_{\alpha}(\mathbb{S}^d)$ for $\alpha=3/2$, formally defined in Supplement \ref{appendix illumination}.

\begin{figure}[t]
\centering
\includegraphics[width = \textwidth]{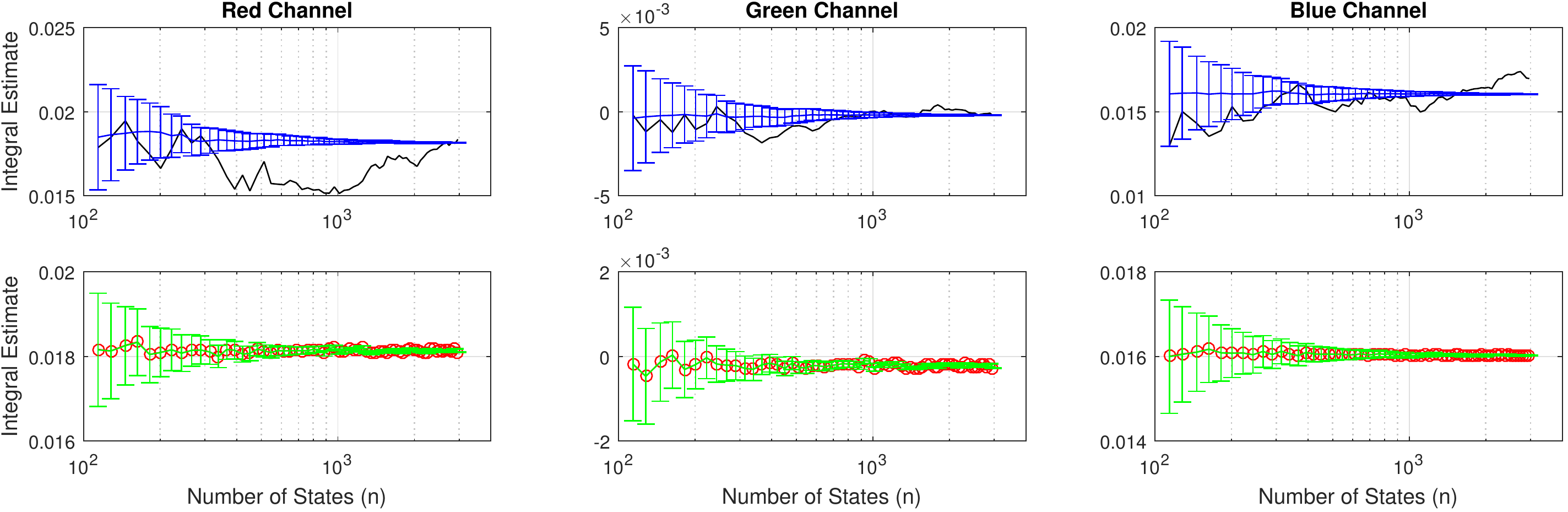}
\caption{Probabilistic integration over the sphere was employed to estimate the RGB colour intensities for the California lake environment.
[Error bars for BMC (blue) and BQMC (green) represent 95\% credible intervals.
MC estimates (black) and QMC estimates (red) are shown for reference.]
}
\label{realpictures_sphere}
\end{figure}

\subsubsection{Results}

Both BMC and BQMC were tested on this example.
To ensure fair comparison, identical kernels were taken as the basis for both methods.
BQMC was employed using a \emph{spherical $t$-design} \citep{Bondarenko2013}. 
It can be shown that for BQMC $\|\hat{\Pi}_{\text{BQMC}} - \Pi \|_{\mathcal{H}^*} = O(n^{-3/4})$ when this point set is used (see Supplement \ref{appendix illumination}). 

Fig. \ref{realpictures_sphere} shows performance in RGB-space.
For this particular test function, the BQMC point estimate was almost identical to the QMC estimate at all values of $n$.
Overall, both BMC and BQMC provided sensible quantification of uncertainty for the value of the integral at all values of $n$ that were considered.

\FloatBarrier

%%%%%%%%%%%%%%%%%%%%%%%%%%%%%%%%%%%%%%%%%%%%%%%%%%%%%%%%%%%%%%%%%%%%%%%%%%%%%%%%%%%%%%%%%%%%%%%%%%%%%%%%%%%%%%%%%%%%%%%%%%%%%%%%%%%%%%%%%%%%%%%%%%%%%%%%%%%%%%%%%

\section{Conclusion} \label{sec:conclusion}

The increasing sophistication of computational models, of which numerical integration is one component, demands an improved understanding of how numerical error accumulates and propagates through computation. In (now common) settings where integrands are computationally intensive, or very many numerical integrals are required, effective methods are required that make full use of information available about the problem at hand. 
This is evidenced by the recent success of QMC, which leverages the smoothness properties of integrands. 
Probabilistic numerics puts the statistician in centre stage and aims to {\it model} the integrand. 
This approach was eloquently summarised by \cite{Kadane1985b}, who proposed the following vision for the future of computation:

\begin{displayquote}
``Statistics can be thought of as a set of tools used in making decisions and
inferences in the face of uncertainty. Algorithms typically operate in such an
environment. Perhaps then, statisticians might join the teams of scholars
addressing algorithmic issues."
\end{displayquote}

This paper explored probabilistic integration from the perspective of the statistician.
Our results highlight both the advantages and disadvantages of such an approach. 
On the positive side, the general methodology described a unified framework in which existing MCMC and QMC methods can be associated with a probability distribution that models discretisation error. 
Posterior contraction rates were, for the first time, established. 
On the negative side, there remain many substantial open questions, in terms of philosophical foundations, theoretical analysis and practical application.
These are discussed below:

\paragraph{Philosophy:} There are several issues concerning interpretation.
First, whose epistemic uncertainty is being modelled? 
In \cite{Hennig2015} it was argued that the uncertainty being modelled is that of a hypothetical agent ``that we get to design''. 
That is, the statistician selects priors and loss functions for the agent so that it best achieves the statistician's own goals.
These goals typically involve a combination of relatively black-box behaviour, to perform well on a diverse range of problems, and a low computational overhead. 
Interpretation of the posterior is then more subtle than for subjective inference and many of the points of contention for objective inference also appear in this framework.

\paragraph{Methodology:}
There are options as to which part of the numerical method should be modelled. 
In this paper, the integrand $f$ was considered to be uncertain while the distribution $\pi$ was considered to be known. 
However, one could alternatively suppose that both $f$ and $\pi$ are unknown, pursued in \cite{Oates2016d}.
\textcolor{black}{Regardless, the endogenous nature of the uncertainty quantification means that in practice one is reliant on effective methods for data-driven estimation of kernel parameters.
The interaction of standard methods, such as empirical Bayes, with the task of numerical uncertainty quantification demands further theoretical research \citep[e.g.][]{Xu2017}. }

\paragraph{Theory:}
For probabilistic integration, further theoretical work is required. 
Our results did not address coverage at finite sample size, nor the interaction of coverage with methods for kernel parameter estimation.
A particularly important question, recently addressed in \cite{Kanagawa2016,Kanagawa2017}, is the behaviour of BC when the integrand does not belong to the posited RKHS.

\paragraph{Prior Specification:} 
A broad discussion is required on what prior information should be included, and what information should be ignored. 
Indeed, practical considerations essentially always demand that some aspects of prior information are ignored.
Competing computational, statistical and philosophical considerations are all in play and must be balanced. 

For example, the RKHS framework that we studied in this paper has the advantage of providing a flexible way to encode prior knowledge about the integrand, allowing to specify properties such as smoothness, periodicity, non-stationarity and effective low-dimension.
On the other hand, several important properties, including boundedness, are less easily encoded. 
For BC, the possibility for importance sampling (Eqn. \ref{importance eq}) has an element of arbitrariness that appears to preclude the pursuit of a default prior.

Even within the RKHS framework, there is the issue that integrands $f$ will usually belong to an infinitude of RKHS. 
Selecting an appropriate kernel is arguably the central open challenge for QMC research at present.
From a practical perspective, elicitation of priors over infinite-dimensional spaces in a hard problem. An adequate choice of prior can be very informative for the numerical scheme and can significantly improve the convergence rates of the method. 
Methods for choosing the kernel automatically could be useful here \citep[e.g.][]{Duvenaud2014}, but would need to be considered against their suitability for providing uncertainty quantification for the integral.

The list above is not meant to be exhaustive, but highlights the many areas of research that are yet to be explored.

%%%%%%%%%%%%%%%%%%%%%%%%%%%%%%%%%%%%%%%%%%%%%%%%%%%%%%%%%%%%%%%%%%%%%%%%%%%%%%%%%%%%%%%%%%%%%%%%%%%%%%%%%%%%%%%%%%%%%%%%%%%%%%%%%%%%%%%%%%%%%%%%%%%%%%%%%%%%%%%%%

\section*{Acknowledgements}

\textcolor{black}{The authors are grateful for the expert feedback received from the Associate Editor and Reviewers}, as well as from A. Barp, J. Cockayne, J. Dick, D. Duvenaud, A. Gelman, P. Hennig, M. Kanagawa, J. Kronander, X-L. Meng, A. Owen, C. Robert, S. S\"{a}rkk\"{a}, C. Schwab, D. Simpson, J. Skilling, T. Sullivan, Z. Tan, A. Teckentrup and H. Zhu. 
The authors thank S. Lan and R. Marques for providing code used in case studies 2 and 4. 
FXB was supported by the EPSRC grant [EP/L016710/1]. 
CJO was supported by the ARC Centre of Excellence for Mathematical and Statistical Frontiers (ACEMS). 
MG was supported by the EPSRC grant [EP/J016934/1, EP/K034154/1], an EPSRC Established Career Fellowship, the EU grant [EU/259348] and a Royal Society Wolfson Research Merit Award. This work was also supported by The Alan Turing Institute under the EPSRC grant [EP/N510129/1] and the Lloyds-Turing Programme on Data-Centric Engineering. Finally, this material was also based upon work partially supported by the National Science Foundation (NSF) under Grant DMS-1127914 to the Statistical and Applied Mathematical Sciences Institute. Any opinions, findings, and conclusions or recommendations expressed in this material are those of the author(s) and do not necessarily reflect the views of the NSF.

%%%%%%%%%%%%%%%%%%%%%%%%%%%%%%%%%%%%%%%%%%%%%%%%%%%%%%%%%%%%%%%%%%%%%%%%%%%%%%%%%%%%%%%%%%%%%%%%%%%%%%%%%%%%%%%%%%%%%%%%%%%%%%%%%%%%%%%%%%%%%%%%%%

%%%%%%%%%%%%%%%%%%%%%%%%%%%%%%%%%%%%%%%%%%%%%%%%%%%%%%%%%%%%%%%%%%%%%%%%%%%%%%%%%%%%%%%%%%%%%%%%%%%%%%%%%%%%%%%%%%%%%%%%%%%%%%%%%%%%%%%%%%%%%%%%%%%%%%%%%%%%%%%%%%%%%%%%%%%%%%%%%%%%%%%%%%%%%%%%%%%%%%%%%%%%%%%%%%%%%%%%%%%%%%%%%%%%%%%%%%%%%%%%%%%%%%%%%%%%%%%%%%%%%%%%%%%%%%%%%%%%%%%%%%%%%%%%%%%%%%%%%%%%%%%%%%%%%%%%%%%%%%%

\newpage
\normalsize
\section*{Supplement}
\appendix
\setcounter{page}{1}

This supplement provides complete proofs for theoretical results, extended numerics and full details to reproduce the experiments presented in the paper.

\vspace{3mm}

%%%%%%%%%%%%%%%%%%%%%%%%%%%%%%%%%%%%%%%%%%%%%%%%%%%%%%%%%%%%%%%%%%%%%%%%

\section{Proof of Theoretical Results} \label{appendix:general_RKHS}

\begin{proof}[Proof of Fact \ref{prop: derive mean and var}]
For a prior $\mathcal{N}(m,c)$ and data $\{(\bm{x}_i,f_i)\}_{i=1}^n$, standard conjugacy results for GPs lead to the posterior $\mathbb{P}_n  = \mathcal{N}(m_n,c_n)$ over $\mathcal{L}$, with mean $m_n(\bm{x}) = m(\bm{x}) + \bm{c}(\bm{x},X) \bm{C}^{-1} (\bm{f} - \bm{m})$ and covariance $c_n(\bm{x},\bm{x}') = c(\bm{x},\bm{x}') - \bm{c}(\bm{x},X) \bm{C}^{-1} \bm{c}(X,\bm{x}')$, see Chap. 2 of \cite{Rasmussen2006}.
Then repeated application of Fubini's theorem produces 
\begin{eqnarray*}
\mathbb{E}_n [\Pi[g]] & = & \mathbb{E}_n \left[ \int g \; \mathrm{d}\pi \right] = \int m_n \; \mathrm{d}\pi  \\
\mathbb{V}_n [\Pi[g]] & = & \int \left[ \int g \; \mathrm{d}\pi - \int m_n \; \mathrm{d}\pi \right]^2 \mathrm{d}\mathbb{P}_n(g) \\
& = & \iiint [g(\bm{x}) - m_n(\bm{x})] [g(\bm{x}') - m_n(\bm{x}')] \; \mathrm{d}\mathbb{P}_n(g) \mathrm{d}\pi(\bm{x}) \mathrm{d}\pi(\bm{x}') \\
& = & \iint c_n(\bm{x},\bm{x}') \; \mathrm{d}\pi(\bm{x}) \mathrm{d}\pi(\bm{x}') . 
\end{eqnarray*}
The proof is completed by substituting the expressions for $m_n$ and $c_n$ into these two equations.
(The result in the main text additionally sets $m \equiv 0$.)
\end{proof}

\begin{proof}[Proof of Fact \ref{mmd and mean element}]
From Eqn. \ref{CSMT} in the main text $\|\hat{\Pi} - \Pi\|_{\mathcal{H}^*}  \leq \|  \mu(\hat{\pi}) - \mu(\pi) \|_{\mathcal{H}}$.
For the converse inequality, consider the specific integrand $f = \mu(\hat{\pi}) - \mu(\pi)$.
Then, from the supremum definition of the dual norm, $\|\hat{\Pi} - \Pi\|_{\mathcal{H}^*}  \geq |\hat{\Pi}[f] - \Pi[f]| / \|  f \|_{\mathcal{H}}$.
Now we use the reproducing property:
\begin{eqnarray*} 
\frac{|\hat{\Pi}[f] - \Pi[f]|}{\|  f \|_{\mathcal{H}}} & = & \frac{|\langle f , \mu(\hat{\pi}) - \mu(\pi) \rangle_{\mathcal{H}}|}{\|f\|_{\mathcal{H}}} \\
& = & \frac{\| \mu(\hat{\pi}) - \mu(\pi) \|_{\mathcal{H}}^2}{\| \mu(\hat{\pi}) - \mu(\pi) \|_{\mathcal{H}}}  \; = \; \|  \mu(\hat{\pi}) - \mu(\pi) \|_{\mathcal{H}}.
\end{eqnarray*}
This completes the proof.
\end{proof}

\begin{proof}[Proof of Fact \ref{prop:WCE_kernelmean_equivalence}]
Combining Fact \ref{mmd and mean element} with direct calculation gives that
\begin{eqnarray*}
& & \hspace{-30pt} \|\hat{\Pi} - \Pi \|_{\mathcal{H}^*}^2 \; = \; \| \mu(\hat{\pi}) - \mu(\pi) \|_{\mathcal{H}}^2 \\
& = & \sum_{i,j=1}^n w_iw_j k(\bm{x}_i,\bm{x}_j) - 2 \sum_{i=1}^n w_i \int k(\bm{x},\bm{x}_i) \; \mathrm{d}\pi(\bm{x})  + \iint k(\bm{x},\bm{x}') \; \mathrm{d}\pi(\bm{x}) \mathrm{d}\pi(\bm{x}') \\
& = &  \bm{w}^\top  \bm{K} \bm{w} - 2 \bm{w}^\top  \Pi[\bm{k}(X,\cdot)] +  \Pi\Pi[k(\cdot,\cdot)]
\end{eqnarray*}
as required.
\end{proof}

The following lemma shows that probabilistic integrators provide a point estimate that is \emph{at least as good} as their non-probabilistic counterparts:

\begin{lemma}[Bayesian re-weighting] \label{theo:_existing} Let $f \in \mathcal{H}$.
Consider the \textcolor{black}{cubature} rule $\hat{\Pi}[f]=\sum_{i=1}^n w_i f(\bm{x}_i)$ and the corresponding BC rule $\hat{\Pi}_{\text{BC}}[f] = \sum_{i=1}^n w_i^{\textnormal{BC}} f(\bm{x}_i)$. 
Then $\|\hat{\Pi}_{\textnormal{BC}} - \Pi \|_{\mathcal{H}^*} \leq \|\hat{\Pi} - \Pi\|_{\mathcal{H}^*}$.
\end{lemma}
\begin{proof}
This is immediate from Fact \ref{prop:WCE_kernelmean_equivalence}, which shows that the BC weights $w_i^{\text{BC}}$ are an optimal choice for the space $\mathcal{H}$.
\end{proof}

The convergence of $\hat{\Pi}_{\text{BC}}$ is controlled by quality of the approximation $m_n$:
\begin{lemma}[Regression bound] \label{theo:_GPs}
Let $f \in \mathcal{H}$ and fix states $\{\bm{x}_i\}_{i=1}^n \in \mathcal{X}$. Then we have $|\Pi[f] - \hat{\Pi}_{\textnormal{BC}}[f] | \leq \|f - m_n \|_2$.
\end{lemma}
\begin{proof}
This is an application of Jensen's inequality: $|\Pi[f] - \hat{\Pi}_{\text{BC}}[f]|^2 = ( \int f - m_n \mathrm{d}\pi )^2  \leq \int (f - m_n)^2 \; \mathrm{d}\pi  = \|f - m_n\|_2^2$, as required.
\end{proof}
\noindent Note that this regression bound is not sharp in general \citep[][Prop. II.4]{Ritter2000} and, as a consequence, Thm. \ref{prop:consistency_BMCMC} below is not quite optimal. 

Lemmas \ref{theo:_existing} and \ref{theo:_GPs} refer to the point estimators provided by BC. 
However, we aim to quantify the change in probability mass as the number of samples increases:
\begin{lemma}[BC contraction] \label{theo:contraction_rates}
Assume $f \in \mathcal{H}$.
Suppose that $\|\hat{\Pi}_{\textnormal{BC}} - \Pi \|_{\mathcal{H}^*} \leq \gamma_n$ where $\gamma_n \rightarrow 0$ as $n \rightarrow \infty$.
Define $I_\delta = [\Pi[f] - \delta , \Pi[f] + \delta]$ to be an interval of radius $\delta > 0$ centred on the true value of the integral. 
Then $\mathbb{P}_n\{\Pi[g] \notin I_\delta\}$ vanishes at the rate $O (\exp (-(\delta^2/2) \gamma_n^{-2}))$.
\end{lemma}
\begin{proof}
Assume without loss of generality that $\delta < \infty$.
The posterior distribution over $\Pi[g]$ is Gaussian with mean $m_n$ and variance $v_n$.
Since $v_n = \|\hat{\Pi}_{\text{BC}} - \Pi \|_{\mathcal{H}^*}^2$ we have $v_n \leq \gamma_n^2$.
Now the posterior probability mass on $I_\delta^c$ is given by $\int_{I_\delta^c} \phi(r|m_n,v_n) \mathrm{d}r$, where $\phi(r|m_n,v_n)$ is the p.d.f. of the $\mathcal{N}(m_n,v_n)$ distribution. 
From the definition of $\delta$ we get the upper bound
\begin{eqnarray*}
\mathbb{P}_n\{\Pi[g] \notin I_\delta\} \quad & \leq & \int_{-\infty}^{\Pi[f] - \delta} \phi(r|m_n,v_n) \mathrm{d}r + \int_{\Pi[f] + \delta}^\infty \phi(r|m_n,v_n) \mathrm{d}r \\
& = & \quad 1 + \Phi\Big(\underbrace{\frac{\Pi[f] - m_n}{\sqrt{v_n}}}_{(*)} - \frac{\delta}{\sqrt{v_n}}\Big) - \Phi\Big(\underbrace{\frac{\Pi[f] - m_n}{\sqrt{v_n}}}_{(*)} + \frac{\delta}{\sqrt{v_n}}\Big).
\end{eqnarray*}
From the definition of the WCE we have that the terms $(*)$ are bounded by $\|f\|_{\mathcal{H}} <\infty$, so that asymptotically as $\gamma_n \rightarrow 0$ we have
\begin{equation*}
\begin{split}
\mathbb{P}_n\{\Pi[g] \notin I_\delta\} \quad & \lesssim \quad 1 + \Phi\big(- \delta / \sqrt{v_n} \big) - \Phi\big(\delta / \sqrt{v_n} \big) \\
& \lesssim \quad 1 + \Phi\big(- \delta / \gamma_n\big) - \Phi\big(\delta / \gamma_n \big) \\
& \lesssim  \quad \text{erfc}\big(\delta/\sqrt{2}\gamma_n\big).
\end{split}
\end{equation*}
The result follows from the fact that $\text{erfc}(x) \lesssim \exp(-x^2/2)$ for $x$ sufficiently small.
\end{proof}
\noindent This result demonstrates that the posterior distribution is well-behaved; probability mass concentrates in a neighbourhood $I_\delta$ of $\Pi[f]$. 
Hence, if our prior is well calibrated (see Sec. \ref{sec:calibration}), the posterior provides uncertainty quantification over the solution of the integral as a result of performing a finite number $n$ of integrand evaluations.

Define the \emph{fill distance} of the set $X = \{\bm{x}_i\}_{i=1}^n$ as 
$$
h_X = \sup_{\bm{x} \in \mathcal{X}} \; \min_{i = 1,\ldots,n} \|\bm{x} - \bm{x}_i\|_2.
$$
As $n \rightarrow \infty$ the scaling of the fill distance is described by the following special case of Lemma 2, \cite{Oates2016CF2}:
\begin{lemma} \label{lemma:lemma1CF2}
Let $v:[0,\infty)\rightarrow[0,\infty)$ be continuous, monotone increasing, and satisfy $v(0)=0$ and $\lim_{x\downarrow 0} v(x)\exp(x^{-3d})=\infty$. Suppose further $\mathcal{X}=[0,1]^d$, $\pi$ is bounded away from zero on $\mathcal{X}$, and $X=\{\bm{x}_i\}_{i=1}^n$ are samples from an uniformly ergodic Markov chain targeting $\pi$. Then we have $\mathbb{E}_{X}[v(h_{X})] = O\big(v(n^{-1/d+\epsilon})\big)$ where $\epsilon>0$ can be arbitrarily small.
\end{lemma}

\begin{proof}[Proof of Thm. \ref{prop:consistency_BMCMC}]
Initially consider fixed states $X = \{\bm{x}_i\}_{i=1}^n$ (i.e. fixing the random seed) and $\mathcal{H} = \mathcal{H}_\alpha$.
From a standard result in functional approximation due to \cite{Wu1993}, see also \citet[][Thm. 11.13]{Wendland2005}, there exists $C>0$ and $h_0 > 0$ such that, for all $\bm{x} \in \mathcal{X}$ and $h_X < h_0$, $|f(\bm{x}) - m_n(\bm{x})| \leq C h_X^\alpha \|f\|_{\mathcal{H}}$.
\citep[For other kernels, alternative bounds are well-known;][Table 11.1]{Wendland2005}.
We augment $X$ with a finite number of states $Y = \{\bm{y}_i\}_{i=1}^m$ to ensure that $h_{X \cup Y} < h_0$ always holds.
Then from the regression bound (Lemma \ref{theo:_GPs}),
\begin{eqnarray*}
\big|\hat{\Pi}_{\text{BMCMC}}[f] - \Pi[f]\big| & \leq & \|f - m_n\|_2 \; = \; \left(\int (f(\bm{x}) - m_n(\bm{x}))^2 \; \mathrm{d} \pi(\bm{x}) \right)^{1/2} \\
& \leq & \left(\int (C h_{X \cup Y}^\alpha \|f\|_{\mathcal{H}})^2 \; \mathrm{d}\pi(\bm{x}) \right)^{1/2} \; = \; C h_{X \cup Y}^\alpha \|f\|_{\mathcal{H}}.
\end{eqnarray*}
It follows that $\|\hat{\Pi}_{\text{BMCMC}} - \Pi \|_{\mathcal{H}_\alpha^*}  \leq  C h_{X \cup Y}^\alpha$. Now, taking an expectation $\mathbb{E}_{X}$ over the sample path $X = \{\bm{x}_i\}_{i=1}^n$ of the Markov chain, we have that 
\begin{eqnarray}
\mathbb{E}_{X} \|\hat{\Pi}_{\text{BMCMC}} - \Pi \|_{\mathcal{H}_\alpha^*} & \leq & C \mathbb{E}_{X} h_{X \cup Y}^\alpha \; \leq \; C \mathbb{E}_{X} h_X^\alpha. \label{vdvp1}
\end{eqnarray} 
From Lemma \ref{lemma:lemma1CF2} above, we have a scaling relationship such that, for $h_{X \cup Y} < h_0$, we have $\mathbb{E}_{X} h_X^\alpha =  O(n^{- \alpha/d  + \epsilon} )$ for $\epsilon > 0$ arbitrarily small.
From Markov's inequality, convergence in mean implies convergence in probability and thus, using Eqn. \ref{vdvp1}, we have $\|\hat{\Pi}_{\text{BMCMC}} - \Pi \|_{\mathcal{H}_\alpha^*} = O_P ( n^{- \alpha / d + \epsilon} )$. This completes the proof for $\mathcal{H} = \mathcal{H}_\alpha$.
More generally, if $\mathcal{H}$ is norm-equivalent to $\mathcal{H}_\alpha$ then the result follows from the fact that $\|\hat{\Pi}_{\text{BMCMC}} - \Pi \|_{\mathcal{H}^*} \leq \lambda \|\hat{\Pi}_{\text{BMCMC}} - \Pi \|_{\mathcal{H}_\alpha^*}$ for some $\lambda > 0$.
\end{proof}

\begin{proof}[Proof of Thm. \ref{theo:BQMC_sobolev}]
From Theorem 15.21 of \cite{Dick2010}, which assumes $\alpha \geq 2$, $\alpha \in \mathbb{N}$, the QMC rule $\hat{\Pi}_{\text{QMC}}$ based on a higher-order digital $(t,\alpha,1,\alpha m\times m,d)$ net over $\mathbb{Z}_b$ for some prime $b$ satisfies $\|\hat{\Pi}_{\text{BQMC}} - \Pi \|_{\mathcal{H}^*} \leq C_{d,\alpha} (\log n)^{d\alpha} n^{-\alpha} = O(n^{-\alpha+\epsilon})$ for $\mathcal{S}_\alpha$ the Sobolev space of dominating mixed smoothness order $\alpha$, where $C_{d,\alpha} > 0$ is a constant that depends only on $d$ and $\alpha$ (but not on $n$). 
The result follows immediately from norm equivalence and Lemma \ref{theo:_existing}.
The contraction rate follows from Lemma \ref{theo:contraction_rates}.
\end{proof}

\textcolor{black}{
\begin{proof}[Proof of Prop. \ref{Stdt prop}]
Denote by $\mathbb{P}_{n,\lambda}$ the posterior distribution on the integral conditional on a value of $\lambda$. 
Following Prop. \ref{prop: derive mean and var}, this is a Gaussian distribution with mean and variance given by:
\begin{eqnarray*}
\mathbb{E}_{n,\lambda}[\Pi[g]] 
& = & \Pi[c_0(\cdot,X)]\bm{C}_0^{-1}\bm{f} \\
\mathbb{V}_{n,\lambda}[\Pi[g]] 
& = & \lambda \{ \Pi\Pi[c_0(\cdot,\cdot)] - \Pi[c_0(\cdot,X)]\bm{C}_0^{-1}\Pi[c_0(X,\cdot)] \}
\end{eqnarray*}
Furthermore, the posterior on the amplitude parameter satisfies
\begin{eqnarray*}
p(\lambda|\bm{f}) 
& \propto & p(\bm{f}|\lambda) p(\lambda)\\
& = & \frac{1}{(2\pi)^{n/2} \lambda^{\frac{n}{2}+1}|\bm{C}_0|^{\frac{1}{2}}} \exp\left(-\frac{1}{2 \lambda} \bm{f}^\top \bm{C}_0^{-1} \bm{f}\right)
\end{eqnarray*}
which corresponds to an inverse-gamma distribution with parameters $\alpha = \frac{n}{2}$ and $\beta = \frac{1}{2}\bm{f}^\top \bm{C}_0^{-1} \bm{f}$. We therefore have that $(\Pi[g],\lambda)$ is distributed as normal-inverse-gamma and the marginal distribution for $\Pi[g]$ is a Student-t distribution, as claimed.
\end{proof}
}

\section{Kernel Means} \label{scalability appendix}

In this section we propose \emph{approximate Bayesian cubature}, $_a\hat{\Pi}_{\text{BC}}$, where the weights $_a\bm{w}_{\text{BC}} = \bm{K}^{-1} {_a}\Pi[\bm{k}(X,\cdot)]$ are an approximation to the optimal BC weights based on an approximation $_a\Pi[\bm{k}(X,\cdot)]$ of the kernel mean \citep[see also Prop. 1 in][]{Sommariva2006}.
The following lemma demonstrates that we can bound the contribution of this error and inflate our posterior $\mathbb{P}_n \mapsto {_a} \mathbb{P}_n$ to reflect the additional uncertainty due to the approximation, so that uncertainty quantification is still provided. 

\begin{lemma}[Approximate kernel mean]\label{lemma:approx_kernel_mean} 
Consider an approximation $_a\pi$ to $\pi$ of the form $_a\pi = \sum_{j=1}^m {_a}w_j \delta_{{_a}\bm{x}_j}$.
Then BC can be performed analytically with respect to $_a\pi$; denote this estimator by $_a\hat{\Pi}_{\text{BC}}$.
Moreover, $\|{_a}\hat{\Pi}_{\text{BC}} - \Pi\|_{\mathcal{H}^*}^2 \leq \|\hat{\Pi}_{\text{BC}} - \Pi\|_{\mathcal{H}^*}^2 + \sqrt{n} \|{_a}\Pi - \Pi \|_{\mathcal{H}^*}^2$.
\end{lemma}
\begin{proof}
Define $\bm{z} = \Pi[\bm{k}(X,\cdot)]$ and ${_a}\bm{z} = {_a}\Pi[\bm{k}(X,\cdot)]$.
Let $\bm{\epsilon} = {_a}\bm{z} - \bm{z}$, write $_a\hat{\Pi}_{\text{BC}} = \sum_{i=1}^n {_a}w_i^{\text{BC}} \delta_{\bm{x}_i}$ and consider 
\begin{eqnarray*}
\|{_a}\hat{\Pi}_{\text{BC}} - \Pi \|_{\mathcal{H}^*}^2 & = & \| \mu({_a}\hat{\pi}_{\text{BC}}) - \mu(\pi) \|_{\mathcal{H}}^2 \\
& & \hspace{-90pt} = \; \left\langle \sum_{i=1}^n {_a}w_i^{\text{BC}} k(\cdot,\bm{x}_i) - \int k(\cdot,\bm{x}) \mathrm{d}\pi(\bm{x}) , \sum_{i=1}^n {_a}w_i^{\text{BC}} k(\cdot,\bm{x}_i) - \int k(\cdot,\bm{x}) \mathrm{d}\pi(\bm{x}) \right\rangle_{\mathcal{H}} \\
& = & {_a}\bm{w}_{\text{BC}}^\top  \bm{K} {_a}\bm{w}_{\text{BC}} - 2 {_a}\bm{w}_{\text{BC}}^\top  \bm{z} + \Pi[\mu(\pi)] \\
& = & (\bm{K}^{-1} {_a}\bm{z})^\top  \bm{K} (\bm{K}^{-1} {_a}\bm{z}) - 2 (\bm{K}^{-1} {_a}\bm{z})^\top  \bm{z} + \Pi[\mu(\pi)] \\
& = & (\bm{z} + \bm{\epsilon})^\top  \bm{K}^{-1} (\bm{z} + \bm{\epsilon}) - 2 (\bm{z} + \bm{\epsilon})^\top  \bm{K}^{-1} \bm{z} + \Pi[\mu(\pi)] \\
& = & \|\hat{\Pi}_{\text{BC}} - \Pi \|_{\mathcal{H}^*}^2 + \bm{\epsilon}^\top  \bm{K}^{-1} \bm{\epsilon}.
\end{eqnarray*}
Use $\otimes$ to denote the tensor product of RKHS.
Now, since $\epsilon_i = {_a}z_i - z_i = \mu({_a}\hat{\pi})(\bm{x}_i) - \mu(\pi)(\bm{x}_i) = \langle \mu({_a}\hat{\pi}) - \mu(\pi) , k(\cdot, \bm{x}_i) \rangle_{\mathcal{H}}$, we have:
\begin{eqnarray*}
\bm{\epsilon}^\top  \bm{K}^{-1} \bm{\epsilon} & = & \sum_{i,i'} [\bm{K}^{-1}]_{i,i'} \big\langle \mu({_a}\hat{\pi}) - \mu(\pi) , k(\cdot, \bm{x}_i) \big\rangle_{\mathcal{H}} \big\langle \mu({_a}\hat{\pi}) - \mu(\pi) , k(\cdot, \bm{x}_{i'}) \big\rangle_{\mathcal{H}}  \\
& = & \Big\langle  \big(\mu({_a}\hat{\pi}) - \mu(\pi) \big) \otimes \big(\mu({_a}\hat{\pi}) - \mu(\pi) \big) , \sum_{i,i'} [\bm{K}^{-1}]_{i,i'} k(\cdot, \bm{x}_i) \otimes k(\cdot, \bm{x}_{i'}) \Big\rangle_{\mathcal{H} \otimes \mathcal{H}}  \\
& \leq & \| \mu({_a}\hat{\pi}) - \mu(\pi) \|_{\mathcal{H}}^2 \Big\| \sum_{i,i'} [\bm{K}^{-1}]_{i,i'} k(\cdot, \bm{x}_i) \otimes k(\cdot, \bm{x}_{i'}) \Big\|_{\mathcal{H} \otimes \mathcal{H}}.
\end{eqnarray*}
From Fact \ref{mmd and mean element} we have $\|\mu({_a}\hat{\pi}) - \mu(\pi) \|_\mathcal{H} = \|{_a}\hat{\Pi} - \Pi\|_{\mathcal{H}}$ so it remains to show that the second term is equal to $\sqrt{n}$. Indeed,
\begin{eqnarray*}
& & \hspace{-40pt} \Big\| \sum_{i,i'} [\bm{K}^{-1}]_{i,i'} k(\cdot,\bm{x}_i) \otimes k(\cdot,\bm{x}_{i'})\Big\|_{\mathcal{H}}^2 \\
& = & \sum_{i,i',l,l'} [\bm{K}^{-1}]_{i,i'} [\bm{K}^{-1}]_{l,l'} \big\langle k(\cdot,\bm{x}_i) \otimes k(\cdot,\bm{x}_{i'}), k(\cdot,\bm{x}_l) \otimes k(\cdot,\bm{x}_{l'}) \big\rangle_{\mathcal{H}} \\
& = & \sum_{i,i',l,l'} [\bm{K}^{-1}]_{i,i'} [\bm{K}^{-1}]_{l,l'} [\bm{K}]_{il}[\bm{K}]_{i',l'} \; = \;
\text{tr} [\bm{K}\bm{K}^{-1}\bm{K}\bm{K}^{-1}] \; = \; n.
\end{eqnarray*}
This completes the proof.
\end{proof}

Under this method, the posterior variance ${_a}\mathbb{V}_n[\Pi[g]] := \|{_a}\hat{\Pi}_{\text{BC}} - \Pi \|_{\mathcal{H}^*}^2$ cannot be computed in closed-form, but computable upper-bounds can be obtained and these can then be used to propagate numerical uncertainty through the remainder of our statistical task.
The idea here is to make use of the triangle inequality:
\begin{eqnarray}\label{eq:upper_bound}
\|{_a}\hat{\Pi}_{\text{BC}} - \Pi \|_{\mathcal{H}^*} & \leq & \|{_a}\hat{\Pi}_{\text{BC}} - {_a}\Pi \|_{\mathcal{H}^*} \; + \; \|{_a}\Pi - \Pi \|_{\mathcal{H}^*}.
\end{eqnarray}
The first term on the RHS is now available analytically; from Fact \ref{prop: derive mean and var} its square is ${_a}\Pi {_a}\Pi[k(\cdot,\cdot)] - {_a}\Pi[\bm{k}(\cdot,X)] \bm{K}^{-1} {_a}\Pi[\bm{k}(X,\cdot)]$.
For the second term, explicit upper bounds exist in the case where states $_a\bm{x}_i$ are independent random samples from $\pi$.
For instance, from \cite[Thm. 27]{Song2008} we have, for a radial kernel $k$, uniform $_aw_j = m^{-1}$ and independent $_a\bm{x}_i \sim \pi$,
\begin{eqnarray}\label{eq:upper_bound2}
\|{_a}\Pi - \Pi \|_{\mathcal{H}^*} & \leq & \frac{2}{\sqrt{m}} \sup_{\bm{x} \in \mathcal{X}} \sqrt{k(\bm{x},\bm{x})} + \sqrt{\frac{\log(2 / \delta)}{2m}}
\end{eqnarray}
with probability at least $1 - \delta$.
(For dependent $_a\bm{x}_j$, the $m$ in Eqn. \ref{eq:upper_bound2} can be replaced with an estimate for the effective sample size.)
Write $C_{n,\gamma,\delta}$ for a $100(1-\gamma)\%$ credible interval for $\Pi[f]$ defined by the conservative upper bound described in Eqns. \ref{eq:upper_bound} and \ref{eq:upper_bound2}.
Then we conclude that $C_{n,\gamma,\delta}$ is $100(1-\gamma)\%$ credible interval with probability at least $1 - \delta$.

Note that, even though the credible region has been inflated, it still contracts to the truth, since the first term on the RHS in Lemma \ref{lemma:approx_kernel_mean} can be bounded by the sum of $\|{_a}\hat{\Pi}_{\text{BC}} - \Pi \|_{\mathcal{H}^*}$ and $\|{_a}\Pi - \Pi \|_{\mathcal{H}^*}$, both of which vanish as $n,m \rightarrow \infty$.
The resulting (conservative) posterior ${_a}\mathbb{P}_n$ can be viewed as a updating of beliefs based on an approximation to the likelihood function; the statistical foundations of such an approach are made clear in the recent work of \cite{Bissiri2016}.

%We pause to briefly discuss the utility and significance of such an approach. Obviously, the new approximation problem (that of approximating $\pi$ with ${_a}\pi$) could also be computed with a BC method, and we may hence end up in an ``infinite regress" scenario \citep{OHagan1991}, where the new kernel mean is itself unknown and so on. However, one level of approximation may be enough in many scenarios. 
%Indeed, by using MC to select $\{{_a}\bm{x}_j\}_{j=1}^m$ and increasing $m$ sufficiently faster than $n$, the error term $\sqrt{n}\|{_a}\Pi - \Pi \|_{\mathcal{H}^*}^2$ can be made to vanish faster than $\|\hat{\Pi}_{\text{BC}} - \Pi \|_{\mathcal{H}^*}^2$ and hence the WCE for $_a\hat{\Pi}_{\text{BC}}$ will be asymptotically identical to the WCE for the (intractable) exact BC estimator $\hat{\Pi}_{\text{BC}}$. Therefore, it will be reasonable to expend computational effort on raising $m$ in settings where evaluation of the integrand constitutes the principal computational. This is because approximating the kernel mean only requires sampling $m$ times, but does not require us to evaluate the integrand. 

%%%%%%%%%%%%%%%%%%%%%%%%%%%%%%%%%%%%%%%%%%%%%%%%%%%%%%%%%%%%%%%%%%%%%%%%%%%%%%%%%%%%%%%%%%%%%%%%%%%%%%%%%%%%%%%%%%%%%%%%%%%%%%%%%%%%%%

\begin{figure}[t]
\centering
\includegraphics[width = 0.45\textwidth,trim={1cm 0.2cm 1cm 0.5cm},clip]{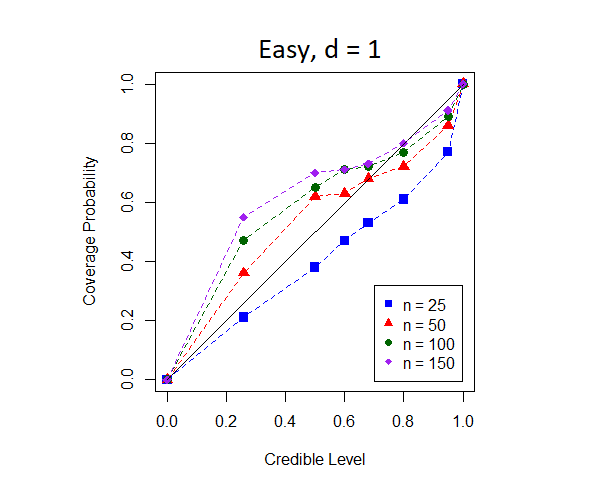}
\includegraphics[width = 0.45\textwidth,trim={1cm 0.2cm 1cm 0.5cm},clip]{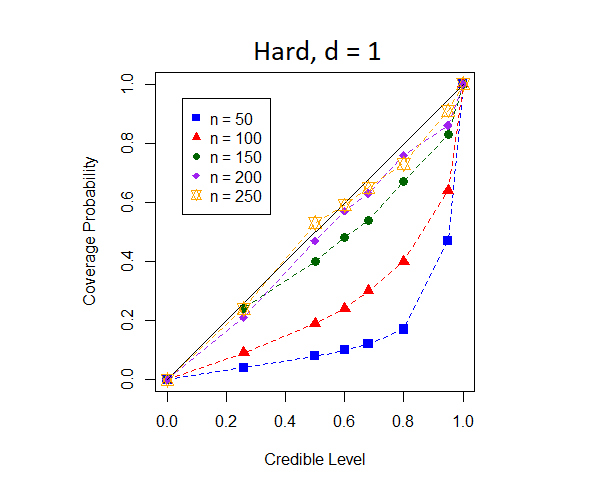}
\includegraphics[width = 0.45\textwidth,trim={1cm 0.2cm 1cm 0.5cm},clip]{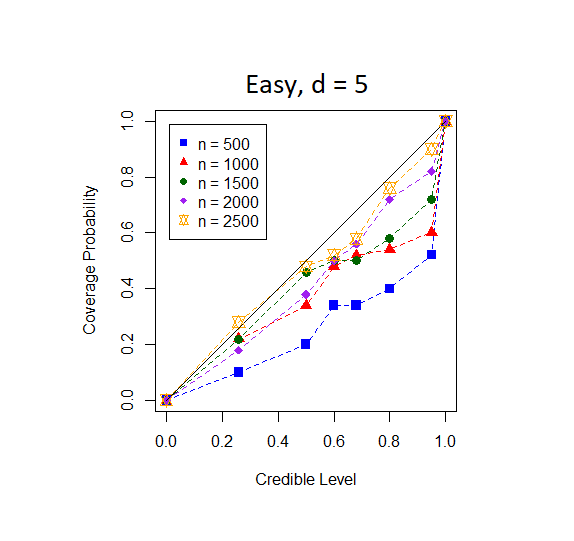}
\includegraphics[width = 0.45\textwidth,trim={1cm 0.2cm 1cm 0.5cm},clip]{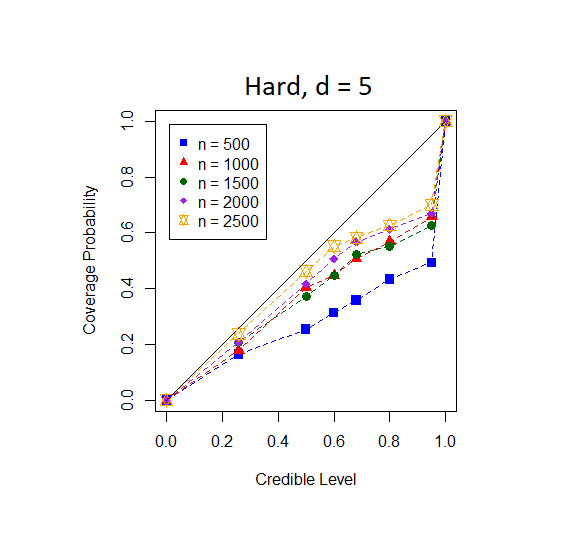}
\caption{Evaluation of uncertainty quantification provided by EB for both $\bm{\sigma}$ and $\lambda$.
Results are shown for $d=1$ (top) and $d=5$ (bottom). 
Coverage frequencies $C_{n,\gamma}$ (computed from $100$ (top) or $50$ (bottom) realisations) were compared against notional $100(1-\gamma)\%$ Bayesian credible regions for varying level $\gamma$. 
\textit{Left:} ``Easy'' test function $f_1$. 
\textit{Right:} ``Hard'' test function $f_2$.
}
\label{fig:calibration_additional}
\end{figure}

\begin{figure}[t!]
\centering
\includegraphics[width = 0.37\textwidth]{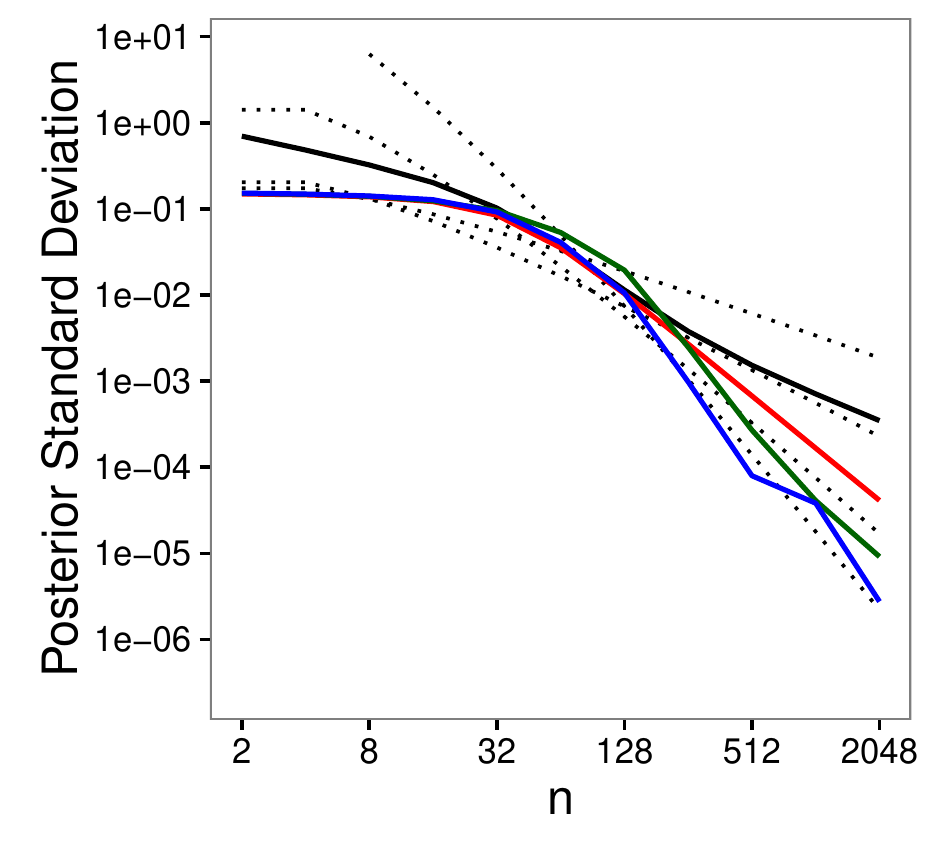}
\includegraphics[width = 0.37\textwidth]{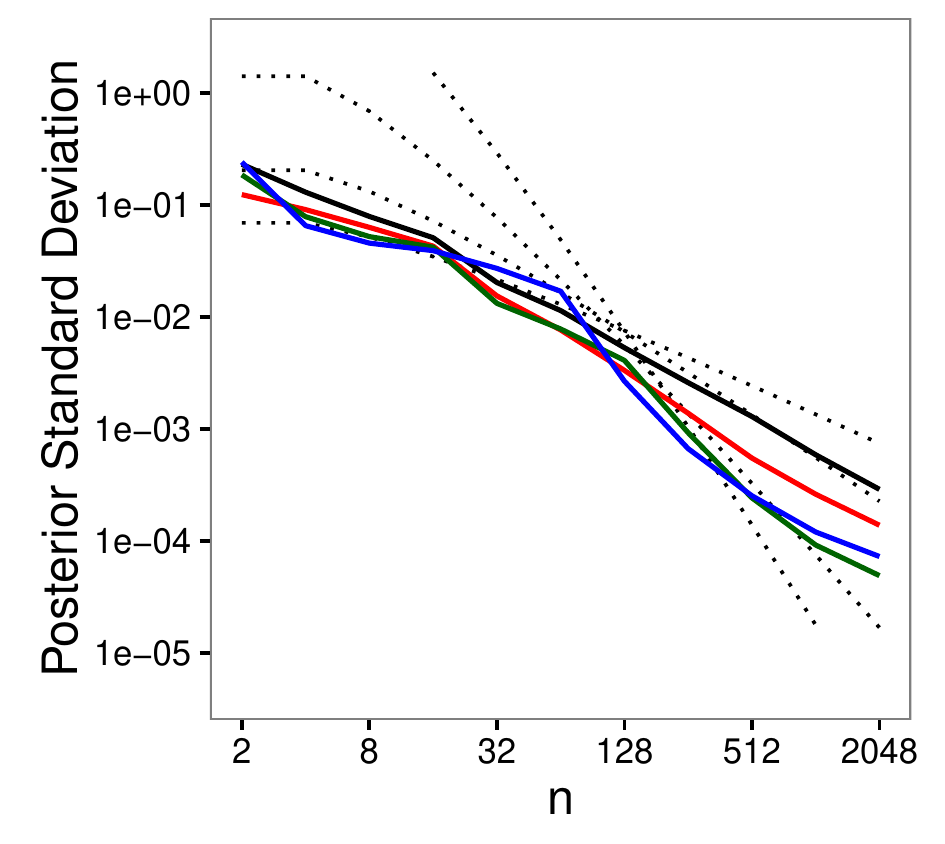}
\caption{Empirical investigation of BQMC in $d=1$ (left) and $d=5$ (right) dimensions and a Sobolev space of mixed dominating smoothness $\mathcal{S}_\alpha$.  
The results are obtained using tensor product Mat\'{e}rn kernels of smoothness $\alpha = 3/2$ (red), $\alpha = 5/2$ (green) and $\alpha=7/2$ (blue). 
Dotted lines represent the theoretical convergence rates established for each kernel. 
The black line represents standard QMC. 
Kernel parameters were fixed to $(\sigma_i,\lambda) = (0.005,1)$ (left) and $(\sigma_i,\lambda) = (1,0.5)$ (right).}
\label{fig:BMC_rates}
\end{figure}

\section{Additional Numerics} \label{extra numerics appendix}

This section presents additional numerical results concerning the calibration of uncertainty for multiple parameters and in higher dimensions.

\paragraph{Calibration in $d = 1$:} In Fig. \ref{fig:calibration_additional} (top row) we study the quantification of uncertainty provided by EB in the same setup as in the main text, but optimising over both length-scale parameter $\sigma_1$ and magnitude parameter $\lambda$. For both ``easy'' and ``hard'' test functions, we notice that EB led to over-confident inferences in the ``low $n$" regime, but attains approximately correct frequentist coverage for larger $n$. 

\paragraph{Calibration in $d=5$:} The experiments of Sec. \ref{sec:test_functions}, based on BMC, were repeated in dimension $d=5$. Results are shown in Fig. \ref{fig:calibration_additional} (bottom row). Clearly more integrand evaluations are required for EB to attain a good frequentist coverage of the credible intervals, due to the curse of dimension. 
However, the frequentist coverage was reasonable for large $n$ in this task.

\paragraph{Empirical convergence assessment:} 
The convergence of BQMC was studied based on higher-order digital nets. 
The theoretical rates provided in Sec. \ref{sec:theory_BQMC} for this method are $O(n^{-\alpha + \epsilon})$ for any $\alpha>1/2$. 
Figure \ref{fig:BMC_rates} gives the results obtained for $d=1$ (left) and $d=5$ (right). In the one dimensional case, the $O(n^{-\alpha + \epsilon})$ theoretical convergence rate is attained by the method in all cases $p = \alpha + 1/2 \in \{3/2,5/2,7/2\}$ considered. However, in the $d=5$ case, the rates are not observed for the number $n$ of evaluations considered. This helps us demonstrate the important point that (in addition to numerical conditioning) the rates we provide are asymptotic, and may require large values of $n$ before being observed.

%%%%%%%%%%%%%%%%%%%%%%%%%%%%%%%%%%%%%%%%%%%%%%%%%%%%%%%%%%%%%%%%%%%%%%%%%%%%%%%%%%%%%%%%%%%%%%%%%%%%%%%%%%%%%%%%%%%%%%%%%%%%%%%%%%%%%%%%%%%%%%%%%%%%%%%%%%%%%%%%%%%%%%%%%%%%%%%%%%%%%%%%%%%%%%%%%%%%%%%%%%%

\FloatBarrier

\section{Supplemental Information for Case Studies} \label{appendix:case_studies}

\subsection{Case Study \#1}
\label{TI importance}

\paragraph{MCMC:}

In this paper we used the manifold Metropolis-adjusted Langevin algorithm \citep{Girolami2011} in combination with population MCMC.
Population MCMC shares information across temperatures during sampling, yet previous work has not leveraged evaluation of the log-likelihood $f$ from one sub-chain $t_i$ to inform estimates derived from other sub-chains $t_{i'}$, $i' \neq i$.
In contrast, this occurs naturally in the probabilistic integration framework, as described in the main text.

Here MCMC was used to generate a small number, $n = 200$, of samples on a per-model basis, in order to simulate a scenario where numerical error in computation of marginal likelihood will be non-negligible.
A temperature ladder with $m = 10$ rungs was employed, for the same reason, according to the recommendation of \cite{Calderhead2009}.
No convergence issues were experienced; the same MCMC set-up has previously been successfully used in \cite{Oates2016}.

\paragraph{Prior elicitation:} 

Here we motivate a prior for the unknown function $g$ based on the work of \cite{Calderhead2009}, who advocated the use of a power-law schedule $t_i = (\frac{i-1}{m-1})^5$, $i = 1,\dots,m$, based on an extensive empirical comparison of possible schedules.
A ``good'' temperature schedule approximately satisfies the criterion $|g(t_i) (t_{i+1} - t_i)| \approx m^{-1}$, on the basis that this allocates equal area to the portions of the curve $g$ that lie between $t_i$ and $t_{i+1}$, controlling bias for the trapezium rule.
Substituting $t_i = (\frac{i-1}{m-1})^5$ into this optimality criterion produces $|g(t_i)| ((i+1)^5 - i^5) \approx m^4$.
Now, letting $i = \theta m$, we obtain $|g(\theta^5)| (5 \theta^4 m^4 + o(m^4)) \approx m^4$.
Formally treating $\theta$ as continuous and taking the $m \rightarrow \infty$ limit produces $|g(\theta^5)| \approx 0.2 \theta^{-4}$ and so $|g(t)| \approx 0.2 t^{-4/5}$.
From this we conclude that the transformed function $h(t) = 5 t^{4/5} g(t)$ is approximately stationary and can reasonably be assigned a stationary GP prior.
However, in an importance sampling transformation we require that $\pi(t)$ has support over $[0,1]$.
For this reason we took $\pi(t) = 1.306/(0.01 + 5t^{4/5})$ in our experiment.

\paragraph{Variance computation:}

The covariance matrix $\bm{\Sigma}$ cannot be obtained in closed-form due to intractability of the kernel mean $\Pi_{t_i}[k_f(\cdot,\bm{\theta})]$.
We therefore explored an approximation ${_a}\bm{\Sigma}$ such that plugging in ${_a}\bm{\Sigma}$ in place of $\bm{\Sigma}$ provides an approximation to the posterior variance $\mathbb{V}_n[\log p(\bm{y})]$ for the log-marginal likelihood.
This took the form
\begin{eqnarray*}
{_a}\Sigma_{i,j} & := & {_a}\Pi_{t_i} {_a}\Pi_{t_j}[k_f(\cdot,\cdot)] - {_a}\Pi_{t_i}[\bm{k}_f(\cdot,X)] \bm{K}_f^{-1} {_a}\Pi_{t_j}[\bm{k}_f(X,\cdot)]
\end{eqnarray*}
where an empirical distribution ${_a}\pi = \frac{1}{100} \sum_{i=1}^{100} \delta_{\bm{x}_i}$ was employed based on the first $m = 100$ samples, while the remaining samples $X = \{\bm{x}_i\}_{i=101}^{200}$ were reserved for the kernel computation.
This heuristic approach becomes exact as $m \rightarrow \infty$, in the sense that ${_a}\Sigma_{i,j} \rightarrow \Sigma_{i,j}$, but under-estimates covariance at finite $m$.

\paragraph{Kernel choice:}

In experiments below, both $k_f$ and $k_h$ were taken to be Gaussian covariance functions; for example: 
$
k_f(\bm{x},\bm{x}') = \lambda_f \exp\big(- \|\bm{x} - \bm{x}'\|_2^2/2\sigma_f^2 \big)
$
parametrised by $\lambda_f$ and $\sigma_f$.
This choice was made to capture smoothness of both integrands $f$ and $h$ involved.
For this application we found that, while the $\sigma$ parameters were possible to learn from data using EB, the $\lambda$ parameters required a large number of data to pin down.
Therefore, for these experiments we fixed $\lambda_f = 0.1 \times \text{mean}(f_{i,j})$ and $\lambda_h = 0.01 \times \text{mean}(h_i)$.
In both cases the remaining kernel parameters $\sigma$ were selected using EB.

\paragraph{Data generation:}

As a test-bed that captures the salient properties of model selection discussed in the main text, we considered variable selection for logistic regression:
\begin{eqnarray*}
p(\bm{y} | \bm{\beta}) & = & \prod_{i=1}^N p_i(\bm{\beta})^{y_i} [1 - p_i(\bm{\beta})]^{1 - y_i} \\
\text{logit}(p_i(\bm{\beta})) & = & \gamma_1 \beta_1 x_{i,1} + \dots \gamma_d \beta_d x_{i,d}, \quad \gamma_1,\dots,\gamma_d \in \{0,1\}
\end{eqnarray*}
where the model $\mathcal{M}_k$ specifies the active variables via the binary vector $\bm{\gamma} =  (\gamma_1,\dots,\gamma_d)$.
A model prior $p(\bm{\gamma}) \propto d^{-\|\bm{\gamma}\|_1}$ was employed.
Given a model $\mathcal{M}_k$, the active parameters $\beta_j$ were endowed with independent priors $\beta_j \sim \mathcal{N}(0,\tau^{-1})$, where here $\tau = 0.01$.

A single dataset of size $N = 200$ were generated from model $\mathcal{M}_1$ with parameter $\bm{\beta} = (1,0,\dots,0)$; as such the problem is under-determined (there are in principle $2^{10} = 1024$ different models) and the true model is not well-identified.
The selected model is thus sensitive to numerical error in the computation of marginal likelihood.
In practice we limited the model space to consider only models with $\sum \gamma_i \leq 2$; this speeds up the computation and, in this particular case, only rules out models that have much lower posterior probability than the actual MAP model.
There were thus 56 models being compared.

%%%%%%%%%%%%%%%%%%%%%%%%%%%%%%%%%%%%%%%%%%%%%%%%%%%%%%%%%%%%%%%%%%%%%%%%%%%%%%%%%%%%%%%%%%%%%%%%%%%%%%%%%%%%%%%%%%%%%%%%%%%%%%%%%%%%%%%%%%%%%%%%%%%%%%%%%%%%%%%%%%%%%%%

\subsection{Case Study \#2} \label{teal appendix}

\paragraph{Background on the model:}
The Teal South model is a PDE computer model for an oil reservoir. The model studied is on an $11 \times 11$ grid with 5 layers. It has 9 parameters representing physical quantities of interest. 
These include horizontal permeabilities for each of the 5 layers, the vertical to horizontal permeability ratio, aquifer strength, rock compressibility and porosity. For our experiments, we used an emulator of the likelihood model documented in \cite{Lan2016} in order to speed up MCMC; however this might be undesirable in general due to the additional uncertainty associated with the approximation in the results obtained.

\paragraph{Kernel choice:} The numerical results in Sec. \ref{computer section} were obtained using a Mat\'{e}rn $\alpha=3/2$ kernel given by $k(r)= \lambda^2 \big(1 + \sqrt{3}r/\sigma\big) \exp \big(-\sqrt{3}r/\sigma \big)$ where $r=\|\bm{x}-\bm{y}\|_2$, which corresponds to the Sobolev space $\mathcal{H}_{3/2}$. 
We note that $f\in \mathcal{H}_{3/2}$ is satisfied. 
We used EB over the length-scale parameter $\sigma$, but fixed the amplitude parameter to $\lambda=1$.

\paragraph{Variance computation:} Due to intractability of the posterior distribution, the kernel mean $\mu(\pi)$ is unavailable in closed form.
To overcome this, the methodology in Supplement \ref{scalability appendix} was employed to obtain an empirical estimate of the kernel mean (half of the MCMC samples were used with BC weights to approximate the integral and the other half with MC weights to approximate the kernel mean).
Eqn. \ref{eq:upper_bound} was used to upper bound the intractable BC posterior variance. For the upper bound to hold, states ${_a}\bm{x}_j$ must be independent samples from $\pi$, whereas here they were obtained using MCMC and were therefore not independent.
In order to ensure that MCMC samples were ``as independent as possible'' we employed sophisticated MCMC methodology developed by \cite{Lan2016}.
Nevertheless, we emphasise that there is a gap between theory and practice here that we hope to fill in future research. For the results in this paper we fixed $\delta = 0.05$ in Eqn. \ref{eq:upper_bound2}, so that $C_{n,\gamma} = C_{n,\gamma,0.05}$ is essentially a $95(1-\gamma)\%$ credible interval.
A formal investigation into the theoretical properties of the uncertainty quantification studied by these methods is not provided in this paper.

%%%%%%%%%%%%%%%%%%%%%%%%%%%%%%%%%%%%%%%%%%%%%%%%%%%%%%%%%%%%%%%%%%%%%%%%%%%%%%%%%%%%%%%%%%%%%%%%%%%%%%%%%%%%%%%%%%%%%%%%%%%%%%%%%%%%%%%%%%%%%%%%%%%%%%%%%%%%%%%%%%%%%%%

\subsection{Case Study \#3} \label{logistic appendix}

\paragraph{Kernel choice:}

The (canonical) \emph{weighted} Sobolev space $\mathcal{S}_{\alpha,\bm{\gamma}}$ is defined by taking each of the component spaces $\mathcal{H}_u$ to be Sobolev spaces of dominating mixed smoothness $\mathcal{S}_\alpha$. i.e. the space $\mathcal{H}_u$ is norm-equivalent to a tensor product of $|u|$ one-dimensional Sobolev spaces, each with smoothness parameter $\alpha$.
Constructed in this way, $\mathcal{S}_{\alpha,\bm{\gamma}}$ is an RKHS with kernel
\begin{equation*}
k_{\alpha,\bm{\gamma}}(\bm{x},\bm{x}') = \sum_{u \subseteq \mathcal{I}} \gamma_u \prod_{i \in u} \left( \sum_{k=1}^\alpha \frac{B_k(x_i)B_k(x_i')}{(k!)^2} - (-1)^\alpha \frac{B_{2\alpha}(|x_i-x_i'|)}{(2\alpha)!} \right),
\end{equation*}
where the $B_k$ are Bernoulli polynomials.

\paragraph{Theoretical results:}

In finite dimensions $d < \infty$, we can construct a higher-order digital net that attains optimal QMC rates for weighted Sobolev spaces:
\begin{theorem} \label{WSS}
Let $\mathcal{H}$ be an RKHS that is norm-equivalent to $\mathcal{S}_{\alpha,\bm{\gamma}}$. Then BQMC based on a digital $(t,\alpha,1,\alpha m \times m,d)$-net over $\mathbb{Z}_b$ attains the optimal rate $\|\hat{\Pi}_{\textnormal{BQMC}} - \Pi \|_{\mathcal{H}^*} = O(n^{-\alpha+\epsilon})$ for any $\epsilon > 0$, where $n = b^m$.
\end{theorem}
\begin{proof}
This follows by combining Thm. 15.21 of \cite{Dick2010} with Lemma \ref{theo:_existing}.
\end{proof}

The QMC rules in Theorem \ref{WSS} do not explicitly take into account the values of the weights $\bm{\gamma}$.
An algorithm that tailors QMC states to specific weights $\bm{\gamma}$ is known as the \emph{component by component} (CBC) algorithm; further details can be found in \citep{Kuo2003}.
In principle the CBC algorithm can lead to improved rate constants in high dimensions, because effort is not wasted in directions where $f$ varies little, but the computational overheads are also greater.
We did not consider CBC algorithms for BQMC in this paper.

Note that the weighted Hilbert space framework allows us to bound the WCE \emph{independently of dimension} providing that $\sum_{u \in \mathcal{I}} \gamma_u < \infty$ \citep{Sloan1998}.
This justifies the use of ``high-dimensional'' in this context.
%Analogous results for functional approximation were provided by \cite{Fasshauer2012} for the Gaussian kernel.
Further details are provided in Sec. 4.1 of \cite{Dick2013}.

%%%%%%%%%%%%%%%%%%%%%%%%%%%%%%%%%%%%%%%%%%%%%%%%%%%%%%%%%%%%%%%%%%%%%%%%%%%%%%%%%%%%%%%%%%%%%%%%%%%%%%%%%%%%%%%%%%%%%%%%%%%%%%%%%%%%%%%%%%%%%%%%%%%%%%%%%%%%%%%%%%%%%%%

\subsection{Case Study \#4} \label{appendix illumination}

\paragraph{Kernel choice:} 

The function spaces that we consider are Sobolev spaces $\mathcal{H}_{\alpha}(\mathbb{S}^d)$ for $\alpha > d/2$, obtained using the reproducing kernel $k(\bm{x},\bm{x}') = \sum_{l=0}^\infty \lambda_l P_l^{(d)} (\bm{x}^\top \bm{x}')$, $\bm{x},\bm{x}' \in \mathbb{S}^d$, where $\lambda_l \asymp (1+l)^{-2\alpha}$ and $P_l^{(d)}$ are normalised Gegenbauer polynomials \citep{Brauchart2012}.
A particularly simple expression for the kernel in $d=2$ and Sobolev space $\alpha=3/2$ can be obtained by taking $\lambda_0 = 4/3$ along with $\lambda_l = - \lambda_0 \times (-1/2)_l/(3/2)_l$ where $(a)_l = a(a+1)\dots(x+l-1) = \Gamma(a+l) / \Gamma(a)$ is the Pochhammer symbol.
Specifically, these choices produce $k(\bm{x},\bm{x}') = 8/3 - \|\bm{x} - \bm{x}'\|_2$, $\bm{x},\bm{x}' \in \mathbb{S}^2$.
This kernel is associated with a tractable kernel mean $\mu(\pi)(\bm{x}) =  \int_{\mathbb{S}^2} k(\bm{x},\bm{x}') \mathrm{d}\pi(\bm{x}') = 4/3$
and hence the initial error is also available $\Pi[\mu(\pi)]= \int_{\mathbb{S}^2} \mu(\pi)(\bm{x}) \mathrm{d}\pi(\bm{x}') = 4/3$.

\begin{figure}[t!]
\centering
\includegraphics[trim= 0 -3cm 0 0, clip=true,width = 0.25\textwidth]{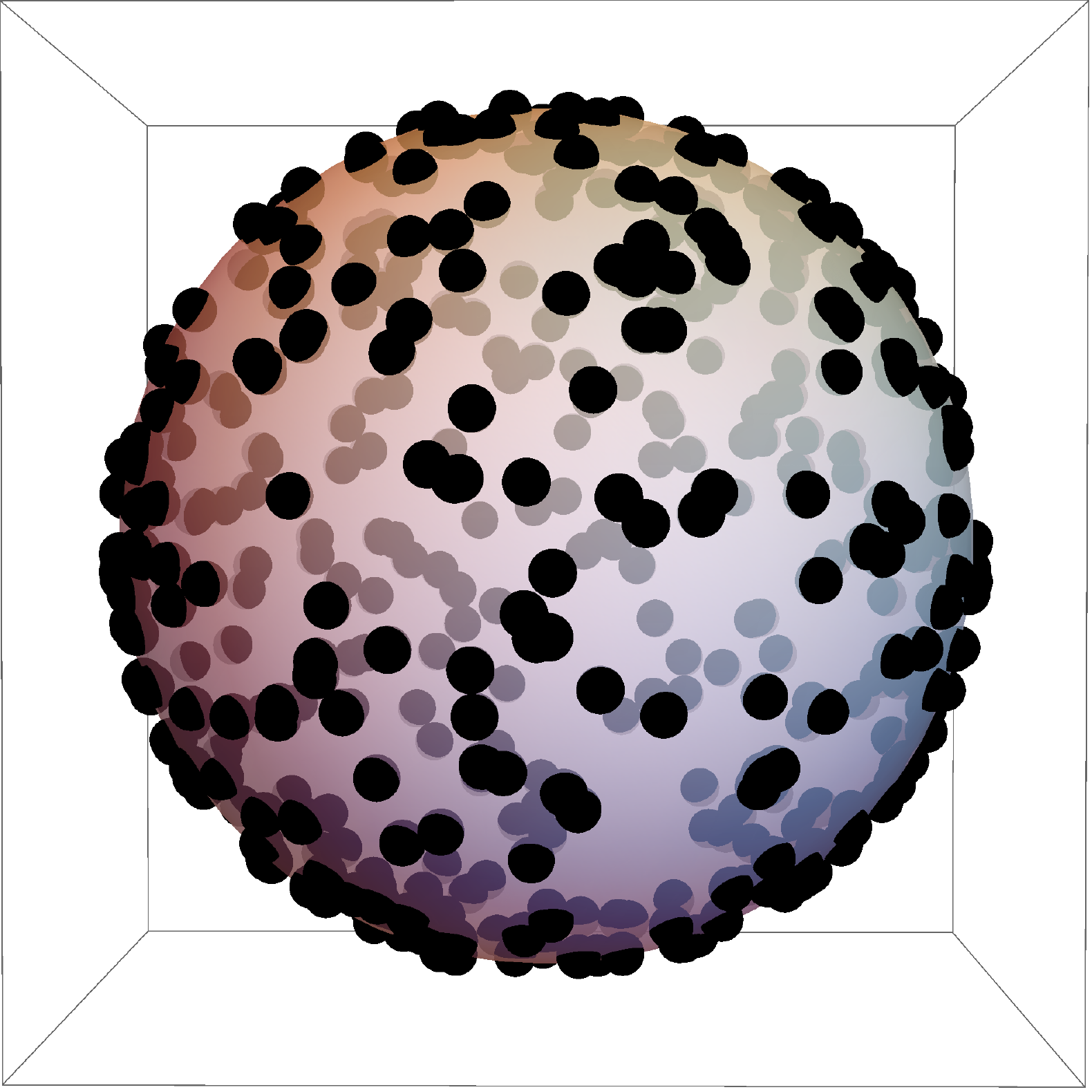}
\hspace{30pt}
\includegraphics[trim= 0 0 0 0.35cm, clip=true, width = 0.65\textwidth]{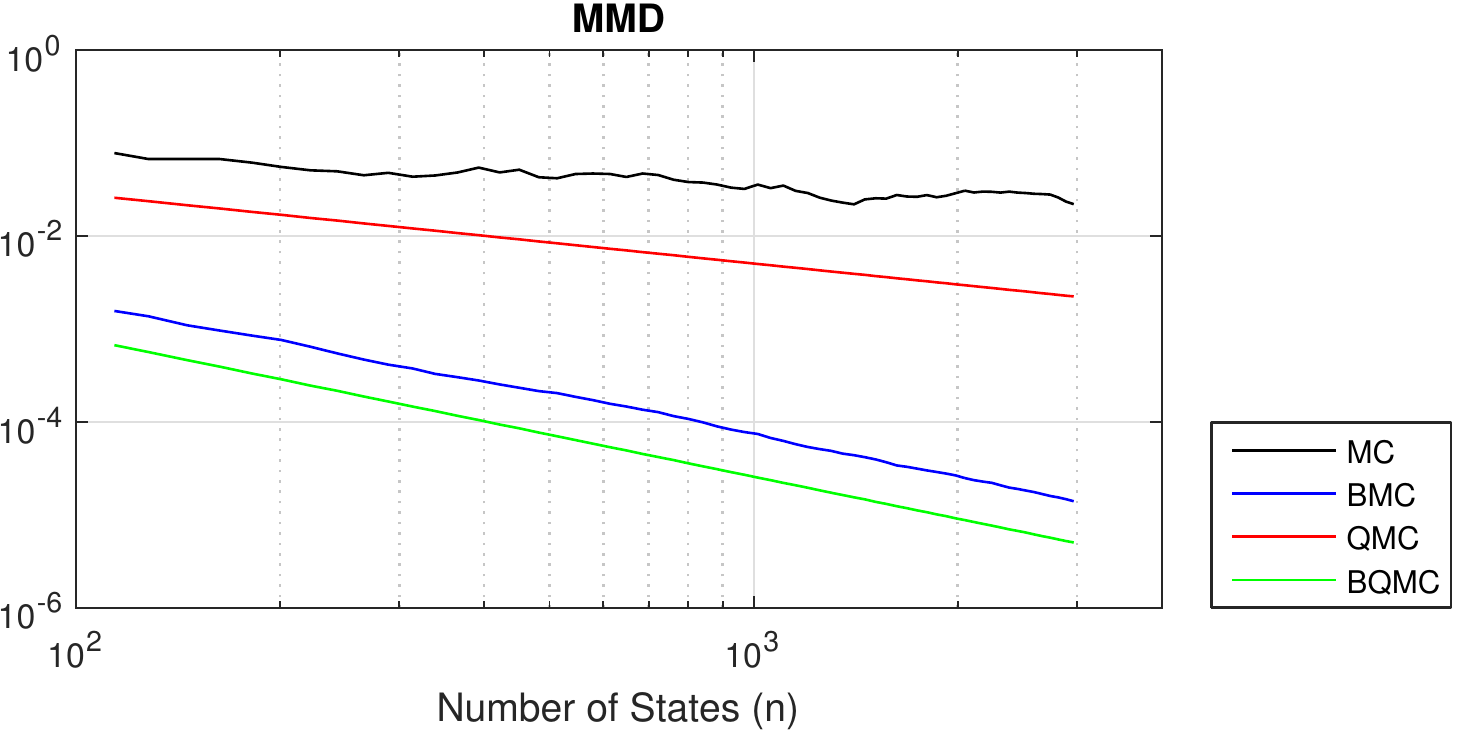}
\caption{Application to global illumination integrals in computer graphics. 
\textit{Left:} A spherical $t$-design over $\mathbb{S}^2$. \textit{Right} The WCE, or worst-case-error, for Monte Carlo (MC), Bayesian MC (BMC), Quasi MC (QMC) and Bayesian QMC (BQMC).
}
\label{realpictures_sphere SI}
\end{figure}

\paragraph{Theoretical results:} 

The states $\{\bm{x}_i\}_{i=1}^n$ could be generated with MC.
In that case, analogous results to those obtained in Sec. \ref{sec:theory_BMC} can be obtained.
Specifically, from Thm. 7 of \cite{Brauchart2012} and Bayesian re-weighting (Lemma \ref{theo:_existing}), classical MC leads to slow convergence $\|\hat{\Pi}_{\text{MC}} - \Pi \|_{\mathcal{H}^*} = O_P(n^{-1/2})$.
The regression bound argument (Lemma \ref{theo:_GPs}) together with a functional approximation result in \citet[][Thm. 3.2]{LeGia2012}, gives a faster rate for BMC of $\|\hat{\Pi}_{\text{BMC}} - \Pi \|_{\mathcal{H}^*} = O_P(n^{-3/4})$ in dimension $d = 2$.

Rather than focus on MC methods, we present results based on spherical QMC point sets.
We briefly introduce the concept of a \emph{spherical $t$-design} \citep{Bondarenko2013} which is define as a set $\{\bm{x}_i\}_{i=1}^n \subset \mathbb{S}^d$ satisfying $\int_{\mathbb{S}^d} f \mathrm{d}\pi = \frac{1}{n} \sum_{i=1}^n f(\bm{x}_i)$ for all polynomials $f : \mathbb{S}^d \rightarrow \mathbb{R}$ of degree at most $t$.
(i.e. $f$ is the restriction to $\mathbb{S}^d$ of a polynomial in the usual Euclidean sense $\mathbb{R}^{d+1} \rightarrow \mathbb{R}$). 
\begin{theorem} \label{sphere statement}
For all $d \geq 2$ there exists $C_d$ such that for all $n \geq C_d t^d$ there exists a spherical $t$-design on $\mathbb{S}^d$ with $n$ states.
Moreover, for $\alpha = 3/2$ and $d = 2$, the use of a spherical $t$-designs leads to a rate $\|\hat{\Pi}_{\textnormal{BQMC}} - \Pi \|_{\mathcal{H}^*} = O(n^{-3/4})$.
\end{theorem}
\begin{proof}
This property of spherical $t$-designs follows from combining \cite{Hesse2005,Bondarenko2013} and Lemma \ref{theo:_existing}.
\end{proof}
The rate in Thm. \ref{sphere statement} is best-possible for a deterministic method in $\mathcal{H}_{3/2}(\mathbb{S}^2)$ \citep{Brauchart2012}.
Although explicit spherical $t$-designs are not currently known in closed-form, approximately optimal point sets have been computed\footnote{our experiments were based on such point sets provided by R. Womersley on his website \url{http://web.maths.unsw.edu.au/~rsw/Sphere/EffSphDes/sf.html} [Accessed 24 Nov. 2015].} numerically to high accuracy.
Additional theoretical results on point estimates can be found in \cite{Fuselier2014}.
In particular they consider the conditioning of the associated linear systems that must be solved to obtain BC weights.

\paragraph{Numerical results:}
In Fig. \ref{realpictures_sphere SI}, the value of the WCE is plotted\footnote{the environment map used in this example is freely available at: \url{http://www.hdrlabs.com/sibl/archive.html} [Accessed 23 May 2017].} for each of the four methods considered (MC, QMC, BMC, BQMC) as the number of states increases. 
Both BMC and BQMC appear to attain the same rate for $\mathcal{H}_{3/2}(\mathbb{S}^2)$, although BQMC provides a constant factor improvement over BMC. Note that $O(n^{-3/4})$ was shown by \cite{Brauchart2012} to be best-possible for a deterministic method in the space $\mathcal{H}_{3/2}(\mathbb{S}^2)$.

\FloatBarrier

%\footnotesize

%%%%%%%%%%%%%%%%%%%%%%%%%%%%%%%%%%%%%%%%%%%%%

\end{document}